\documentclass[11pt]{article}

\usepackage[top=30truemm,bottom=30truemm,left=25truemm,right=25truemm]{geometry}
\usepackage[utf8]{inputenc} 
\usepackage[T1]{fontenc}    

\PassOptionsToPackage{capitalize,noabbrev,nameinlink}{cleveref}
\PassOptionsToPackage{authoryear,round}{natbib}
\PassOptionsToPackage{font=small,labelfont=bf}{caption}
\PassOptionsToPackage{noactivechars}{mathstyle}

\usepackage{microtype}
\usepackage{graphicx}
\usepackage{xcolor}
\usepackage{subcaption}
\usepackage{booktabs}

\usepackage{my_macro}
\usepackage{breqn}
\usepackage{tgtermes}
\usepackage[textsize=tiny]{todonotes}
\usepackage{float}
\usepackage{adjustbox}
\usepackage{titletoc}
\usepackage[noend]{algpseudocode}
\newcommand{\tildeell}{\tilde{\ell}}

\newcommand{\trunc}{\mathrm{trunc}}
\newcommand{\tildeL}{\tilde{L}}

\title{Prediction with Expert Advice under Heavy-Tailed Losses}

\author{
  Yuki Shibukawa\footnote{
    The University of Tokyo and RIKEN;
    \texttt{shibu-yu762@g.ecc.u-tokyo.ac.jp}.
  }
  \and
  Shinji Ito\footnote{
    The University of Tokyo and RIKEN; \texttt{shinji@mist.i.u-tokyo.ac.jp}.
  }
}

\begin{document}
\maketitle

\begin{abstract}
We study prediction with expert advice over $T$ rounds with $K$ experts under heavy-tailed losses whose $(1+\epsilon)$-moment is bounded by $\theta>0$ for some $\epsilon\in(0,1]$.
In this setting, existing algorithms designed for the case $\epsilon=1$ can be suboptimal when $\epsilon<1$, suffering an additional polynomial dependence on $K$.
We characterize the minimax pseudo-regret rate as $\Theta((\theta T)^{\frac{1}{1+\epsilon}}(\log K)^{\frac{\epsilon}{1+\epsilon}})$ and establish gap-dependent lower bounds in a self-bounded environment.
To attain this minimax rate without prior knowledge of either $\epsilon$ or $\theta$, we propose Adaptive Heavy-Tailed FTRL, which uses robust data-dependent scale adaptation together with truncation.
We further develop Adaptive Heavy-Tailed OMD, which likewise requires neither $\epsilon$ nor $\theta$ and achieves near-optimal regret guarantees in both adversarial and self-bounded environments.
Finally, when $\epsilon$ and $\theta$ are known, an FTRL algorithm with Shannon-entropy regularization achieves optimal regret guarantees in both environments.
\end{abstract}

\section{Introduction}
\label{sec:introduction}

We study prediction with expert advice \citep{VOVK1998153_prediction_with_expert_advice,littlestone1994weighted,cesa2006prediction} under heavy-tailed losses.
This is a full-information online learning problem (hereafter referred to as the experts problem) in which a learner repeatedly combines the predictions of multiple experts.
At each round $t$, the learner chooses a distribution $p_t\in\Delta_K=\set{p\in\R^K_{\geq0}:\sum_{i=1}^K p(i)=1}$ over $K$ experts, observes the loss vector $\ell_t\in\R^K$, and incurs the linear loss $\inpr{\ell_t,p_t}$.
The learner aims to minimize its cumulative loss.
We measure its performance using pseudo-regret (hereafter referred to as regret), defined as the difference between its cumulative loss and that of the best fixed expert in expectation:
\begin{equation}
    \label{eq:def_regret}
    \mathcal{R}_T=\max_{i\in[K]}\E\brk*{\sum_{t=1}^T\prn{\inpr{\ell_t,p_t}-\ell_t(i)}}.
\end{equation}
Unlike the classical setting with uniformly bounded losses, we assume only that the $(1+\epsilon)$-moment ($\epsilon\in(0,1]$) of each expert's loss is upper bounded by $\theta>0$.
We refer to $\epsilon$ as the moment parameter and $\theta$ as the moment bound.
This heavy-tailed condition allows losses to be unbounded and permits rare but extremely large observations, which makes standard concentration arguments and learning-rate choices based on bounded losses inapplicable.

One concrete example of the experts problem arises in sequential investment based on multiple trading strategies \citep{AlBaghdadi2020}.
Each trading strategy can be regarded as an expert, and an investor maintains weights over these experts and aggregates their recommendations over time.
Since financial asset returns exhibit heavy tails and extreme fluctuations \citep{Cont2001_financial_heavytail,BRADLEY200335_financial_risk_heavy_tails}, this provides a natural motivation for studying the experts problem under heavy-tailed losses.
Similar heavy-tailed behavior arises in electricity spot prices \citep{Weron2009_elec} and network traffic \citep{Crovella1998HeavyTail}.

The closest existing result is by \citet{NEURIPS2025_Moulin_Expert_HTEXP}, who consider the setting of $\epsilon=1$, which corresponds to assuming that the second moment of each expert's loss is bounded.
By contrast, the literature regarding heavy-tailed bandits commonly adopts the more general assumption that the $(1+\epsilon)$-moment is bounded for some $\epsilon\in(0,1]$ \citep{Bubeck2013BanditsWithHeavyTail,huang2022adaptive,pmlr-v247-genalti24a,chen2025uniinf_HTMAB}.
This naturally raises the question of whether the algorithms of \citet{NEURIPS2025_Moulin_Expert_HTEXP} continue to achieve desirable regret guarantees under this more general moment condition.

Our answer to the above question is negative for a direct application of these algorithms.
For every $\epsilon\in(0,1)$ and sufficiently large $T$, we establish algorithm-specific lower bounds for the Lower-Order-Term-Free (LoOT-Free) OMD and FTRL of \citet{NEURIPS2025_Moulin_Expert_HTEXP} that exhibit a polynomial dependence on $K$ absent from the minimax rate (see \cref{app:subsec:alg_specific_lower_bound} for details).
As shown in \cref{app:subsec:epsilon_known_moulin}, this failure is caused not by their framework itself, but by the mismatch between a quadratic scale and the general moment condition.

However, identifying the limitations of these existing algorithms does not determine the minimax regret rate of the heavy-tailed experts problem.
A direct rescaling of the $\Omega\prn{\sqrt{T\log K}}$ lower bound for the bounded experts problem \citep{cesa1997} yields only $\Omega\prn{\theta^{\frac{1}{1+\epsilon}}\sqrt{T\log K}}$.
Because the exponents of $T$ and $K$ in this bound do not depend on $\epsilon$, it does not reveal how the moment parameter changes the scaling of regret with $T$ and $K$.
For comparison, the regret lower bound for non-heavy-tailed bandits is $\Omega(\sqrt{KT})$ ({\it e.g.}, \citet{bubeck2012regret_mab}), whereas its heavy-tailed counterpart is $\Omega(\theta^{\frac{1}{1+\epsilon}}K^{\frac{\epsilon}{1+\epsilon}}T^{\frac{1}{1+\epsilon}})$ \citep{Bubeck2013BanditsWithHeavyTail}.
This comparison suggests that an $\epsilon$-dependent rate, involving for example $T^{\frac{1}{1+\epsilon}}$, should also arise in the heavy-tailed experts problem.
Prior to this work, however, the correct minimax dependence on $T$, $K$, and $\epsilon$ remained unknown.

Adaptivity to the heavy-tail parameters\footnote{Throughout this paper, unless otherwise specified, adaptivity refers to adaptation to both unknown heavy-tail parameters $\epsilon$ and $\theta$, rather than to the horizon $T$.} also raises an unresolved question in the heavy-tailed experts problem.
For heavy-tailed bandits, \citet{pmlr-v247-genalti24a} showed that, without additional assumptions, no algorithm adaptive to either $\epsilon$ or $\theta$ can achieve a regret upper bound matching the lower bound.
For the heavy-tailed experts problem, by contrast, \citet{NEURIPS2025_Moulin_Expert_HTEXP} constructed algorithms adaptive to $\theta$, but established regret guarantees only for $\epsilon=1$.
When applied to $\epsilon\in(0,1)$, these algorithms are not optimal (\cref{prop:moulin-balanced-sparse-spike}).
Thus, prior to our work, it remained open whether an algorithm adaptive to both $\epsilon$ and $\theta$ exists for the heavy-tailed experts problem.

Beyond minimax optimality in adversarial environments, it is desirable to obtain smaller gap-dependent regret under additional stochastic structure.
Rather than restricting attention to i.i.d.\ stochastic losses, we consider the more general self-bounded environment \citep{pmlr-v89-zimmert19a_bobw,NEURIPS2021_ito_robustness_adversarial_decision}, which includes the standard stochastic setting as a special case.
We refer to an algorithm as best-of-both-worlds (BoBW) if it achieves near-optimal regret guarantees in both adversarial and self-bounded environments, where near-optimal means matching the corresponding optimal rates up to logarithmic factors.
For $\epsilon=1$, \citet{NEURIPS2025_Moulin_Expert_HTEXP} developed a BoBW algorithm for the heavy-tailed experts problem.
However, it remained open whether BoBW guarantees could be extended to general $\epsilon\in(0,1]$. 
A related challenge is to eliminate the additional logarithmic dependence on $T$ while retaining adaptivity to both $\epsilon$ and $\theta$.

Motivated by these observations, we focus on the following three main research questions.
\begin{enumerate}[label=(Q\arabic*), topsep=0.1pt, partopsep=0.5pt, itemsep=0.01cm, parsep=0pt]
    \item \label{q:minimax} What is the minimax regret rate for the heavy-tailed experts problem under the $(1+\epsilon)$-moment condition?
    \item \label{q:adaptivity} Can this minimax rate be attained by an algorithm that requires prior knowledge of neither $\epsilon$ nor $\theta$?
    \item \label{q:bobw} Can we design a BoBW algorithm for general $\epsilon\in(0,1]$?
\end{enumerate}

\subsection{Contributions}
\label{subsec:intro_contributions}

\begin{table*}[t]
    \centering
    \caption{Summary of prior results and our upper and lower regret bounds.
    The self-bounded column considers the environment in \cref{asp:self_bounded} with $C=0$.
    The $\epsilon$-adaptive and $\theta$-adaptive columns indicate whether the algorithm is adaptive to the respective parameters.
    The lower bounds hold under the respective conditions of the cited statements, including conditions on $K$, $T$, and $\Delta_{\min}$.
    \looseness=-1
    }
    \label{tab:upper-lower-bounds}
    \footnotesize
    \setlength{\tabcolsep}{3pt}
    \renewcommand{\arraystretch}{0.8}
    \begin{tabular}{@{}p{0.23\textwidth}p{0.26\textwidth}p{0.26\textwidth}p{0.1\textwidth}p{0.1\textwidth}@{}}
        \toprule
        \raggedright Algorithm
        &
        \raggedright Adversarial
        &
        \raggedright Self-bounded ($C=0$)
        &
        \raggedright $\epsilon$-adaptive
        &
        \raggedright $\theta$-adaptive \tabularnewline
        \midrule
        \raggedright \citep[LoOT-Free FTRL]{NEURIPS2025_Moulin_Expert_HTEXP} ($\epsilon=1$)
        &
        $\displaystyle O\prn{\sqrt{\theta T\log K}}$
        &
        $\displaystyle O\prn{\sqrt{\theta T\log K}}$
        &
        N/A
        &
        Yes
        \\
        \raggedright \citep[LoOT-Free FTRL]{NEURIPS2025_Moulin_Expert_HTEXP}
        &
        $\displaystyle O\prn{(\theta T)^{\frac{1}{1+\epsilon}}K^{\frac{1-\epsilon}{2(1+\epsilon)}}\sqrt{\log K}}$\newline(\cref{cor:moulin-ftrl-componentwise})
        &
        $\displaystyle O\prn{(\theta T)^{\frac{1}{1+\epsilon}}K^{\frac{1-\epsilon}{2(1+\epsilon)}}\sqrt{\log K}}$\newline(\cref{cor:moulin-ftrl-componentwise})
        &
        Yes
        &
        Yes
        \\
        \addlinespace
        \raggedright \cref{alg:adaptive-multiscale-ftrl-componentwise_simple} (ours)
        &
        $\displaystyle O\prn{(\theta T)^{\frac{1}{1+\epsilon}}(\log K)^{\frac{\epsilon}{1+\epsilon}}}$\newline(\cref{thm:upper-bound-multiscale-ftrl-componentwise_simple})
        &
        $\displaystyle O\prn{(\theta T)^{\frac{1}{1+\epsilon}}(\log K)^{\frac{\epsilon}{1+\epsilon}}}$\newline(\cref{thm:upper-bound-multiscale-ftrl-componentwise_simple})
        &
        Yes
        &
        Yes
        \\
        \addlinespace
        \raggedright \cref{alg:parameter-free-heavy-tailed-omd} (ours)
        &
        $\displaystyle O\prn{(\theta T)^{\frac{1}{1+\epsilon}}(\log(KT))^{\frac{\epsilon}{1+\epsilon}}}$\newline(\cref{thm:omd-regret-adversarial})
        &
        $\displaystyle O\prn{\theta^{\frac{1}{\epsilon}}\Delta_{\min}^{-\frac{1}{\epsilon}}\log(KT)}$\newline(\cref{thm:omd-regret-self-bounded})
        &
        Yes
        &
        Yes
        \\
        \addlinespace
        \raggedright \cref{alg:known-parameter-shannon-ftrl} (ours)
        &
        $\displaystyle O\prn{(\theta T)^{\frac{1}{1+\epsilon}}(\log K)^{\frac{\epsilon}{1+\epsilon}}}$\newline(\cref{thm:known-parameter-adversarial})
        &
        $\displaystyle O\prn{\theta^{\frac{1}{\epsilon}}\Delta_{\min}^{-\frac{1}{\epsilon}}\log K}$\newline(\cref{thm:known-parameter-stochastic})
        &
        No
        &
        No
        \\
        \midrule
        \raggedright \citep[LoOT-Free FTRL]{NEURIPS2025_Moulin_Expert_HTEXP}
        &
        $\displaystyle \Omega\prn{(\theta T)^{\frac{1}{1+\epsilon}}K^{\frac{1-\epsilon}{2(1+\epsilon)}}\sqrt{\log K}}$\newline(\cref{prop:moulin-balanced-sparse-spike})
        &
        ---
        &
        ---
        &
        ---
        \\
        \raggedright Any algorithm
        &
        $\displaystyle \Omega\prn{(\theta T)^{\frac{1}{1+\epsilon}}(\log K)^{\frac{\epsilon}{1+\epsilon}}}$\newline(\cref{cor:lower_bound_gap_independent})
        &
        $\displaystyle \Omega\prn{\theta^{\frac{1}{\epsilon}}\Delta_{\min}^{-\frac{1}{\epsilon}}\log K}$\newline(\cref{thm:lower_bound_gap_dependent})
        &
        ---
        &
        ---
        \\
        \bottomrule
    \end{tabular}
\end{table*}

Our main contribution is a minimax characterization of the heavy-tailed experts problem, together with an adaptive algorithm that attains the minimax rate without prior knowledge of either $\epsilon$ or $\theta$.
This result is complemented by two further contributions: algorithm-specific lower bounds showing the limitations of existing LoOT-Free algorithms under the general heavy-tailed setting, and BoBW guarantees.
\cref{tab:upper-lower-bounds} summarizes these results and compares them with the closest existing guarantees.

As a motivating negative result, we establish algorithm-specific lower bounds showing that direct applications of the LoOT-Free algorithms can incur polynomial dependence on $K$ under the weaker moment condition.
When $T^{\frac{\epsilon}{1+\epsilon}}\geq cK^{\frac{1-\epsilon}{2(1+\epsilon)}}\sqrt{\log K}$ for a sufficiently large universal constant $c>0$, these lower bounds are $\Omega\prn{(\theta T)^{\frac{1}{1+\epsilon}}K^{\frac{1-\epsilon}{2(1+\epsilon)}}/\sqrt{\log(KT)}}$ for LoOT-Free OMD and $\Omega\prn{(\theta T)^{\frac{1}{1+\epsilon}}K^{\frac{1-\epsilon}{2(1+\epsilon)}}\sqrt{\log K}}$ for LoOT-Free FTRL (\cref{prop:moulin-balanced-sparse-spike}).
To establish these lower bounds, we construct a loss process that exploits a mismatch between truncation in the update and the quadratic scale used by the LoOT-Free algorithms.
A rare, large spike produces instantaneous regrets that are discarded from the update but still inflate the scale, making subsequent updates overly conservative.

We next establish minimax lower bounds for the heavy-tailed experts problem.
For $K\geq4$, $T\geq\log K$, and minimum gap $\Delta_{\min}>0$, our construction tailored to the heavy-tailed experts problem establishes, for any learner, the gap-dependent lower bound $\Omega(\theta^{\frac{1}{\epsilon}}\Delta_{\min}^{-\frac{1}{\epsilon}}\log K)$ under the additional conditions on $\Delta_{\min}$ and $T$ stated in \cref{thm:lower_bound_gap_dependent}.
Optimizing over the gap parameter yields the gap-independent lower bound $\Omega((\theta T)^{\frac{1}{1+\epsilon}}(\log K)^{\frac{\epsilon}{1+\epsilon}})$ (\cref{cor:lower_bound_gap_independent}).
Furthermore, by adapting the technique of \citet{NEURIPS2021_ito_robustness_adversarial_decision}, we establish a gap-dependent lower bound in a self-bounded environment (\cref{thm:lower_bound_self_bounded_simple}).

Our central algorithmic contribution is Adaptive Heavy-Tailed follow-the-regularized-leader (FTRL), which requires neither $\epsilon$ nor $\theta$ and matches the minimax lower bound.
A natural extension of LoOT-Free FTRL using a $(1+\epsilon)$-moment scale attains the optimal regret bound (see \cref{app:subsec:epsilon_known_moulin} for details), but this construction requires prior knowledge of $\epsilon$.
Our algorithm instead uses the robust increment function $\phi_s(x)=\min\set{x^2/s,\abs{x}}$ for $s>0$, with $\phi_0(x)=\abs{x}$.
Since $\phi_s$ satisfies $\phi_s(x)\leq s^{-\epsilon}\abs{x}^{1+\epsilon}$, the resulting scales and truncation cost can be controlled by the $(1+\epsilon)$-moment bound in the analysis, even though the algorithm does not know $\epsilon$.
Using $\phi_s$, Adaptive Heavy-Tailed FTRL achieves the minimax-optimal regret bound $O((\theta T)^{\frac{1}{1+\epsilon}}(\log K)^{\frac{\epsilon}{1+\epsilon}})$ (\cref{thm:upper-bound-multiscale-ftrl-componentwise_simple}).

Beyond this minimax characterization, we study BoBW guarantees in the self-bounded environment.
Adaptive Heavy-Tailed online mirror descent (OMD) achieves near-optimal regret guarantees: a gap-independent bound of $O((\theta T)^{\frac{1}{1+\epsilon}}(\log(KT))^{\frac{\epsilon}{1+\epsilon}})$ and, under \cref{asp:self_bounded} with $C=0$, a gap-dependent bound of $O(\theta^{\frac{1}{\epsilon}}\Delta_{\min}^{-\frac{1}{\epsilon}}\log(KT))$ (\cref{thm:omd-regret-adversarial,thm:omd-regret-self-bounded}).
To obtain these bounds, we combine the robust increment $\phi_s$ with an exponential doubling trick tailored to the experts setting.
Unlike LoOT-Free OMD, the resulting algorithm requires no prior knowledge of $T$.
When $\epsilon$ and $\theta$ are known, we additionally propose Known-Parameter Shannon-FTRL, which achieves optimal BoBW guarantees (\cref{thm:known-parameter-adversarial,thm:known-parameter-stochastic}).

As an additional result, we develop an algorithm that achieves a near-optimal gap-dependent regret bound under i.i.d.\ loss vectors without prior knowledge of $\epsilon$ or $\theta$, with only an additional doubly logarithmic term (see \cref{app:iid-adaptation} for details).
We also provide an experimental comparison of the proposed algorithms with existing methods in \cref{app:experiments}.
\section{Preliminaries}
\label{sec:preliminaries}
In this section, we provide the notation and problem setting used in this paper.
\paragraph{Notation}
For each positive integer $n$, let $[n]=\set{1,\dots,n}$.
For each $i\in[n]$, let $e_i\in\R^n$ denote the $i$th standard basis vector.
For $x\in\R^n$, we write $x(i)$ for its $i$th coordinate.
For $x,y\in\R^n$, let $\inpr{x,y}=\sum_{i=1}^n x(i)y(i)$ denote their inner product.
For $q\in[1,\infty)$, define $\nrm{x}_q=\prn*{\sum_{i=1}^n\abs{x(i)}^q}^{\frac{1}{q}}$, and define $\nrm{x}_\infty=\max_{i\in[n]}\abs{x(i)}$.
Let $\Delta_K=\set{p\in\R_{\geq0}^K:\sum_{i=1}^Kp(i)=1}$ denote the probability simplex over $K$ experts.
For $p,q\in\Delta_K$, let $\KL(p;q)=\sum_{i=1}^Kp(i)\log\frac{p(i)}{q(i)}$ denote the Kullback--Leibler (KL) divergence, with the standard conventions $0\log(0/q(i))=0$ and $p(i)\log(p(i)/0)=+\infty$ for $p(i)>0$.
We write $\E$ and $\P$ for expectation and probability, respectively, and $\ind\brk{E}$ for the indicator of an event $E$.
Unless stated otherwise, the notation $O(\cdot)$, $\Omega(\cdot)$, and $\Theta(\cdot)$ hides positive constants that may depend on $\epsilon$ but are independent of $T$, $K$, $\theta$, $\Delta_{\min}$, and $C$.
\looseness=-1

\subsection{Setting}
We consider prediction with expert advice over $T$ rounds with $K$ experts.
At the beginning of each round $t\in[T]$, the learner selects a distribution $p_t\in\Delta_K$.
The loss vector $\ell_t\in\R^K$ is then generated arbitrarily and revealed to the learner, who incurs the loss $\inpr{p_t,\ell_t}$.
Let $(\mathcal{F}_t)_{t=0}^T$ denote the filtration, where $\mathcal{F}_{t-1}$ contains all observations and the learner's internal randomization available before $\ell_t$ is revealed.
We require $p_t$ to be $\mathcal{F}_{t-1}$-measurable.
For any random variable $X$, we write $\E_t[X]=\E[X\mid\mathcal{F}_{t-1}]$.

We assume that $\ell_t$ satisfies the following moment condition:
\begin{assumption}
    \label{asp:unknown-epsilon_wise}
    For some $\epsilon\in(0,1]$ and $\theta>0$, the loss sequence $(\ell_t)_{t=1}^T$ satisfies $\E_t\brk{\abs{\ell_t(i)}^{1+\epsilon}}\leq\theta$ for all $i\in[K]$ and $t\in[T]$.
\end{assumption}
This condition is commonly used in the heavy-tailed bandit and experts literature \citep{Bubeck2013BanditsWithHeavyTail,chen2025uniinf_HTMAB,NEURIPS2025_Moulin_Expert_HTEXP}.
In particular, losses satisfying $\abs{\ell_t(i)}\leq B$ almost surely satisfy \cref{asp:unknown-epsilon_wise} with $\theta=B^{1+\epsilon}$.

For each expert $i\in[K]$, define its instantaneous regret by $r_t(i)=\inpr{p_t,\ell_t}-\ell_t(i)$ and its expected cumulative regret by $\mathcal{R}_T(i)=\E\brk{\sum_{t=1}^T r_t(i)}$.
We recall that the regret of the learner is $\mathcal{R}_T=\max_{i\in[K]}\mathcal{R}_T(i)$.

\paragraph{Self-bounded environment}
For the gap-dependent analysis, we additionally consider the following assumption \citep{pmlr-v89-zimmert19a_bobw,NEURIPS2021_ito_robustness_adversarial_decision}.
\begin{assumption}
    \label{asp:self_bounded}
    There exist an expert $i^\ast\in[K]$, a gap parameter $\Delta_{\min}>0$, and a corruption budget $C\geq0$ such that, for any algorithm, $i^\ast$ is the unique minimizer of $\E\brk{\sum_{t=1}^T\ell_t(i)}$ over $i\in[K]$ and $\mathcal{R}_T(i^\ast)\geq\Delta_{\min}\sum_{t=1}^T\E\brk{1-p_t(i^\ast)}-C$.
\end{assumption}
\cref{asp:self_bounded} captures environments in which assigning probability mass to suboptimal experts leads to cumulative regret, while allowing deviations from this behavior through the corruption budget $C$.
If the loss vectors $(\ell_t)_{t=1}^T$ are drawn i.i.d.\ independently of the learner and have a unique best expert $i^\ast$, with $\Delta_{\min}=\min_{i\neq i^\ast}\E\brk{\ell_t(i)-\ell_t(i^\ast)}>0$, then \cref{asp:self_bounded} holds with $C=0$.
The gap-independent regret bounds require only \cref{asp:unknown-epsilon_wise}, whereas the gap-dependent regret bounds also require \cref{asp:self_bounded}.

\section{Related Work}
\label{sec:related_work}

Prediction with expert advice is a classical online learning problem \citep{VOVK1998153_prediction_with_expert_advice,littlestone1994weighted,cesa2006prediction}.
For bounded losses, subsequent work has developed algorithms with second-order regret guarantees \citep{pmlr-v35-gaillard14_second_order_bound,pmlr_v40_Koolen15a_squint}.
Other work has developed loss-range-adaptive and scale-free guarantees \citep{JMLR_2014_Rooij_AdaHedge,ORABONA2018_scale_free_online_learning,pmlr-v99-mhammedi19a_lipschitz_adaptivity}.
Related work has further established algorithms that retain worst-case guarantees while achieving gap-dependent regret that can be independent of $T$ \citep{mourtada2019optimality_expert,NEURIPS2020_amir_prediction_expert,NEURIPS2021_ito_robustness_adversarial_decision}.
We refer to \citet[Section 3]{NEURIPS2025_Moulin_Expert_HTEXP} for a more detailed review of algorithms for the experts problem and the limitations of existing adaptive guarantees under unbounded losses.

To the best of our knowledge, the only previous work that directly studies prediction with expert advice under heavy-tailed losses is \citet{NEURIPS2025_Moulin_Expert_HTEXP}.
They assume a second-moment bound $\E_t[\ell_t(i)^2]\leq\theta$.
They first show that the maximum-loss or maximum-range terms appearing in existing regret bounds adaptive to the loss range can be as large as $\Omega(\sqrt{KT})$ even when $\theta=O(1)$, and may therefore dominate the leading second-order term.
To avoid this dependence, they propose two algorithms based on truncation.
LoOT-Free FTRL achieves a worst-case regret bound of $O(\sqrt{\theta T\log K})$, whereas LoOT-Free OMD achieves a regret bound of $O(\sqrt{\theta T\log(KT)})$ and simultaneously satisfies a gap-dependent bound in a self-bounded environment.
In particular, when the corruption parameter is zero, the latter bound is $O(\theta\Delta_{\min}^{-1}\log(KT))$.
However, under the $(1+\epsilon)$-moment condition with $\epsilon\in(0,1)$, direct applications of these algorithms incur an additional polynomial dependence on $K$.
As detailed in \cref{app:subsec:alg_specific_lower_bound}, this dependence is unavoidable for the original LoOT-Free OMD and FTRL algorithms.
\cref{app:subsec:epsilon_known_moulin} further shows that this limitation stems from the mismatch between the quadratic scale and the weaker moment condition.
When $\epsilon$ is known, replacing the quadratic scale with a $(1+\epsilon)$-moment scale recovers the minimax-optimal regret rate.
\looseness=-1

\paragraph{Lower bounds}
The lower bounds closely related to our work arise from the experts problem and stochastic bandits with heavy-tailed losses.
For the classical experts problem, \citet{cesa1997} established the gap-independent lower bound $\Omega(\sqrt{T\log K})$.
In the stochastic experts problem, \citet{mourtada2019optimality_expert} proved the gap-dependent lower bound $\Omega(\Delta^{-1}\log K)$ by using Fano's inequality relative to a symmetric reference environment.
In addition, \citet{NEURIPS2021_ito_robustness_adversarial_decision} proved a gap-dependent lower bound in a self-bounded environment, which generalizes the gap-dependent lower bound of \citet{mourtada2019optimality_expert}.
For stochastic bandits with heavy-tailed losses, \citet{Bubeck2013BanditsWithHeavyTail} established the gap-independent lower bound $\Omega\prn{\theta^{\frac{1}{1+\epsilon}}K^{\frac{\epsilon}{1+\epsilon}}T^{\frac{1}{1+\epsilon}}}$.
We discuss how our lower-bound constructions differ from these prior arguments in \cref{sec:lower_bound}.

\paragraph{Quadratic-to-linear interpolation in adaptive analyses}
Quadratic-to-linear bounds of the form $\min\set{x^2/s,\abs{x}}$ appear in the analyses of AdaHedge \citep{JMLR_2014_Rooij_AdaHedge} and scale-free FTRL \citep{ORABONA2018_scale_free_online_learning}.
In these works, however, this structure is used only to bound one-step regret or stability terms.
By contrast, our algorithm incorporates this quadratic-to-linear form directly into its update, rather than using it merely as an analytical tool.

We provide additional discussion of related work on heavy-tailed bandits, self-bounded environments, and online convex optimization in \cref{app:relare_work}.

\section{Lower Bounds}
\label{sec:lower_bound}
This section establishes lower bounds that apply to any learner.
Rescaling the lower bound of \citet{cesa1997} yields the lower bound of $\Omega\prn{\theta^{\frac{1}{1+\epsilon}}\sqrt{T\log K}}$.
However, unlike the heavy-tailed bandit lower bound of $\Omega\prn{\theta^{\frac{1}{1+\epsilon}}K^{\frac{\epsilon}{1+\epsilon}}T^{\frac{1}{1+\epsilon}}}$ \citep{Bubeck2013BanditsWithHeavyTail}, this bound does not reflect the effect of $\epsilon$ on its dependence on $T$ and $K$, suggesting that it can be improved.
We derive a tighter lower bound by combining a proof technique of \citet{mourtada2019optimality_expert} developed for the non-heavy-tailed experts problem with an instance tailored to the heavy-tailed setting.
\begin{theorem}
    \label{thm:lower_bound_gap_dependent}
    Fix $\epsilon\in(0,1]$ and $\theta>0$.
    Let $\Delta_{\min}\in(0,8^{-\frac{\epsilon}{1+\epsilon}}\theta^{\frac{1}{1+\epsilon}}]$, $K\geq 4$, and $T\geq\log K$ satisfy $T\geq\frac{1}{8}\theta^{\frac{1}{\epsilon}}\Delta_{\min}^{-\frac{1+\epsilon}{\epsilon}}\log K$.
    Then, for any algorithm, there exist a unique optimal expert $i^\ast\in[K]$ and a distribution such that the losses satisfy \cref{asp:unknown-epsilon_wise}, $\E\brk*{\ell_t(i)-\ell_t(i^\ast)}\geq\Delta_{\min}$ for all $t\in[T]$ and $i\neq i^\ast$, and the regret satisfies
    \[ \mathcal{R}_T\geq\frac{\theta^{\frac{1}{\epsilon}}\log K}{128\Delta_{\min}^{\frac{1}{\epsilon}}}. \]
\end{theorem}

The proof framework of \cref{thm:lower_bound_gap_dependent} largely follows the technique of \citet{mourtada2019optimality_expert}, whereas the lower-bound instance is specific to the heavy-tailed setting.
For a hidden optimal expert $i^\ast$, set $\ell_t(i^\ast)=0$, whereas each other expert independently incurs loss $B=(\theta/\Delta_{\min})^{\frac{1}{\epsilon}}$ with probability $q=\theta^{-\frac{1}{\epsilon}}\Delta_{\min}^{1+\frac{1}{\epsilon}}$.
With this choice, $qB=\Delta_{\min}$ and $qB^{1+\epsilon}=\theta$. 
Consequently, every suboptimal expert has gap $\Delta_{\min}$, and the instance satisfies the $(1+\epsilon)$-moment condition.
Moreover, since all losses are nonnegative and expert $i^\ast$ always incurs zero loss, $i^\ast$ is also a best expert for every realized loss sequence.

The key observation is that a positive loss rules out an expert as optimal, whereas observing zero loss does not distinguish the optimal expert from the others.
For every $n\leq T$, the regret under the environment indexed by $i^\ast$ is $\mathcal{R}_n(i^\ast)=\Delta_{\min}\sum_{t=1}^n\E_{i^\ast}\brk*{1-p_t(i^\ast)}$, where $\E_{i^\ast}$ denotes expectation under that environment.
To quantify the difficulty of identifying $i^\ast$, compare these environments with a symmetric baseline environment in which every expert independently incurs loss $B$ with probability $q$.
Over the first $T_0$ rounds, the KL divergence from each indexed environment to the baseline is at most $2qT_0$.
Taking $T_0$ of order $(\log K)/q$ with a sufficiently small constant factor, Fano's inequality implies that, in at least one environment, the learner assigns a constant fraction of its cumulative probability mass to suboptimal experts.
Consequently, $\mathcal{R}_T(i^\ast)\geq\mathcal{R}_{T_0}(i^\ast)=\Omega(\Delta_{\min}T_0)$.
See \cref{app:subsec:lower_bound_gap_dependent} for the formal proof of \cref{thm:lower_bound_gap_dependent}.

Optimizing the choice of $\Delta_{\min}$ in \cref{thm:lower_bound_gap_dependent} yields the following gap-independent lower bound.
\begin{corollary}
    \label{cor:lower_bound_gap_independent}
    Consider the same setting as in \cref{thm:lower_bound_gap_dependent}.
    For any algorithm, there exists an instance for which the regret satisfies $\mathcal{R}_T= \Omega\prn{\prn{\theta T}^{\frac{1}{1+\epsilon}}\prn{\log K}^{\frac{\epsilon}{1+\epsilon}}}$.
\end{corollary}
\begin{proof}
    Set $\Delta_{\min}=\theta^{\frac{1}{1+\epsilon}}\prn*{\frac{\log K}{8T}}^{\frac{\epsilon}{1+\epsilon}}$ and apply \cref{thm:lower_bound_gap_dependent}.
\end{proof}
When $\epsilon=1$, this lower bound matches the upper bound of LoOT-Free FTRL proposed by \citet{NEURIPS2025_Moulin_Expert_HTEXP}.
When $\epsilon\in(0,1)$ and $T$ satisfies the condition in \cref{prop:moulin-balanced-sparse-spike}, the algorithm-specific lower bound for LoOT-Free FTRL is larger than the minimax lower bound above by a polynomial factor in $K$.

\begin{remark}[Difference from the lower-bound instance for heavy-tailed bandits]
    Both our construction and the heavy-tailed bandit lower bound of \citet[Theorem~2]{Bubeck2013BanditsWithHeavyTail} use rare and large observations.
    In the bandit construction, the probability of a rare reward is perturbed between two arm distributions, and an argument based on the KL divergence between the corresponding Bernoulli distributions captures the need to sample a suboptimal arm sufficiently often.
    In our construction, by contrast, the optimal expert always incurs zero loss, while each other expert independently incurs a rare positive loss.
    Such a bandit argument does not directly apply to the experts problem because the learner observes the losses of all experts at every round.
    We instead compare the $K$ environments corresponding to different hidden optimal experts with a symmetric baseline environment and use Fano's inequality to show that identifying the optimal expert requires accumulating enough evidence.
\end{remark}

In addition, we obtain a lower bound under \cref{asp:self_bounded} by combining \cref{thm:lower_bound_gap_dependent} with the technique of \citet[Theorem~5]{NEURIPS2021_ito_robustness_adversarial_decision}.
For conciseness, we state a simplified version below. 
The more general statement is given in \cref{thm:lower_bound_self_bounded}.
\begin{theorem}
    \label{thm:lower_bound_self_bounded_simple}
    Fix $\epsilon\in(0,1]$, $\theta>0$, $\Delta_{\min}\in(0,8^{-\frac{\epsilon}{1+\epsilon}}\theta^{\frac{1}{1+\epsilon}}]$, $K\geq4$, and $C\geq0$.
    Then, for any integer $T$ satisfying $T\geq \max\set{\frac{1}{8}\theta^{\frac{1}{\epsilon}}\Delta_{\min}^{-\frac{1+\epsilon}{\epsilon}}\log K, C\Delta_{\min}^{-1}}$ and any algorithm, there exists a nonnegative loss instance satisfying \cref{asp:unknown-epsilon_wise,asp:self_bounded} for which
    \[
        \mathcal{R}_T=\Omega\prn[\bigg]{\prn*{\frac{\theta}{\Delta_{\min}}}^{\frac{1}{\epsilon}}\log K+\prn*{\frac{\theta C}{\Delta_{\min}}}^{\frac{1}{1+\epsilon}}\prn*{\log K}^{\frac{\epsilon}{1+\epsilon}}}.
    \]
\end{theorem}
The corruption budget allows the per-round expected gap to be smaller than $\Delta_{\min}$ during an initial prefix while the aggregate self-bounded condition remains satisfied.
This postpones the learner's identification of the optimal expert, giving rise to the $C$-dependent term.
The proof is given in \cref{app:lower_bound_self_bounded}.

\section{Algorithms}
\label{sec:algorithms}
This section presents our algorithms and analyses for the heavy-tailed experts problem.
\subsection{Robust Increment Function}
\label{subsec:robust-scale-increment}

For $s\geq0$ and $x\in\R$, define
\begin{equation*}
    \phi_s(x)=
    \begin{cases}
        \displaystyle\min\set*{\frac{x^2}{s},\abs{x}}, & s>0,\\[1ex]
        \abs{x}, & s=0.
    \end{cases}
\end{equation*}
This function satisfies the following properties, which are the only properties of $\phi_s$ used in our analyses:
\begin{enumerate}[label=(P\arabic*), topsep=0.5pt, partopsep=1pt, itemsep=0.01cm, parsep=0pt]
    \item \label{p:scale-monotonicity} For every $x\in\R$ and $0\leq a\leq b$, we have $0\leq\phi_b(x)\leq\phi_a(x)\leq\abs{x}$.
    \item \label{p:moment-bound} For every $\epsilon\in(0,1]$, $s>0$, and $x\in\R$, we have $\phi_s(x)\leq s^{-\epsilon}\abs{x}^{1+\epsilon}$.
    \item \label{p:regime-lower-bounds} For every $s\geq0$ and $x\in\R$, if $\abs{x}>s$, then $\abs{x}\leq\phi_s(x)$, and if $\abs{x}\leq s$, then $x^2\leq s\phi_s(x)$.
\end{enumerate}
\hyperref[p:scale-monotonicity]{(P1)} ensures that increasing the reference scale does not increase the scale increment.
\hyperref[p:moment-bound]{(P2)} allows the same data-dependent scale to be analyzed under any $(1+\epsilon)$-moment condition without supplying $\epsilon$ to the algorithm.
The two regimes in \hyperref[p:regime-lower-bounds]{(P3)} serve complementary purposes: the linear lower bound charges the contribution removed by truncation, whereas the quadratic lower bound controls the local stability term of the update.
Note that our proofs rely only on \hyperref[p:scale-monotonicity]{(P1)}--\hyperref[p:regime-lower-bounds]{(P3)} and not on the specific functional form of $\phi_s$.
We discuss the tightness of this choice of $\phi_s$ in \cref{rem:robust-scale-increment-tightness}.

\begin{remark}[Why a robust scale is needed]
    If $\epsilon$ were known, the polynomial dependence on $K$ in the regret bound for LoOT-Free FTRL could be removed simply by replacing its cumulative squared-regret scale with a cumulative $(1+\epsilon)$-moment scale based on $\abs{r_t(i)}^{1+\epsilon}$ and taking the power $1/(1+\epsilon)$.
    With $\beta=(\log K)^{\frac{1}{1+\epsilon}}$, this yields the minimax rate $O\prn{(\theta T)^{\frac{1}{1+\epsilon}}(\log K)^{\frac{\epsilon}{1+\epsilon}}}$ (see \cref{app:subsec:epsilon_known_moulin} for details).
    Thus, the difficulty lies not in the FTRL framework itself, but in constructing an appropriate scale without knowing $\epsilon$.
    The robust increment above achieves precisely this through \hyperref[p:moment-bound]{(P2)}.
\end{remark}

\subsection{Adaptive FTRL Algorithm}
\label{subsec:ftrl}
\begin{algorithm}[t]
    \caption{Adaptive Heavy-Tailed FTRL}
    \label{alg:adaptive-multiscale-ftrl-componentwise_simple}
    \begin{algorithmic}[1]
        \Require {Number of experts $K\geq2$.}
        \State {Initialize $G_0\gets0$ and $\Lambda_{\mathrm{F}}\gets1+\log K$.}
        \State {Initialize $\pi(i)\gets1/K$, $\tilde{L}_0(i)\gets0$, $B_0(i)\gets0$, and $M_0(i)\gets0$ for all $i\in[K]$.}
        \For{$t=1,2,\ldots,T$}
            \If{$\max_{i\in[K]}M_{t-1}(i)=0$}
                \State {$p_t\gets\pi$.}
            \Else
                \State {$p_t\gets\argmin_{p\in\Delta_K}\set{\inpr{\tilde{L}_{t-1},p}+D_{t-1}(p;\pi)}$.}
            \EndIf
            \State {Predict $p_t$, incur loss $\inpr{p_t,\ell_t}$, and observe $\ell_t$.}
            \State {$r_t(i)\gets\inpr{p_t,\ell_t}-\ell_t(i)$ for all $i\in[K]$.}
            \State {$g_t\gets\sum_{j=1}^K p_t(j)\phi_{M_{t-1}(j)}\prn{r_t(j)}$ and $d_t(i)\gets\phi_{M_{t-1}(i)}\prn{r_t(i)}$ for all $i\in[K]$.}
            \State {$G_t\gets G_{t-1}+\frac{g_t}{\Lambda_{\mathrm{F}}}$, $B_t(i)\gets B_{t-1}(i)+\frac{d_t(i)}{\Lambda_{\mathrm{F}}}$, and $M_t(i)\gets\max\set*{G_t,B_t(i)}$ for all $i\in[K]$.}
            \State {$\rho_t(i)\gets M_{t-1}(i)/M_t(i)$ if $M_t(i)>0$, and $\rho_t(i)\gets0$ otherwise, for all $i\in[K]$.}
            \State {$\tilde{\ell}_t(i)\gets-r_t(i)\rho_t(i)\ind\brk*{\abs{r_t(i)}\leq M_t(i)}$ and $\tilde{L}_t(i)\gets\tilde{L}_{t-1}(i)+\tilde{\ell}_t(i)$ for all $i\in[K]$.}
        \EndFor
    \end{algorithmic}
\end{algorithm}

\Cref{alg:adaptive-multiscale-ftrl-componentwise_simple} applies the robust increment from \cref{subsec:robust-scale-increment} to the instantaneous regrets and requires no prior knowledge of $T$, $\epsilon$, or $\theta$.
The global scale $G_t$ accumulates the learner-weighted robust increments, whereas $B_t(i)$ accumulates the robust increments associated with expert $i$.
Taking $M_t(i)=\max\set{G_t,B_t(i)}$ yields an effective coordinatewise scale that incorporates both the global and expert-specific scales and is used consistently in truncation, rescaling, and regularization.
The factor $\rho_t(i)=M_{t-1}(i)/M_t(i)$ accounts for the growth of the effective scale and, together with truncation, keeps the surrogate loss $\tilde{\ell}_t(i)$ normalized by $M_{t-1}(i)$ bounded.
Although the algorithm does not know $\epsilon$, \hyperref[p:moment-bound]{(P2)} yields a regret guarantee for every $\epsilon\in(0,1]$ whenever the losses satisfy \cref{asp:unknown-epsilon_wise}.

For the resulting scales $M_t(i)$, define the multi-scale entropic regularizer $\Psi_t(p)=\sum_{i=1}^K M_t(i)p(i)\log p(i)$.
Let $D_t$ denote the Bregman divergence induced by $\Psi_t$, given by $D_t(p;q)=\Psi_t(p)-\Psi_t(q)-\inpr{\nabla\Psi_t(q),p-q}=\sum_{i=1}^K M_t(i)\prn[\big]{p(i)\log\frac{p(i)}{q(i)}-p(i)+q(i)}$.
For every $t\in[T]$, let $\bar{\ell}_t=-r_t$ and write $p_{1:T}=(p_1,\ldots,p_T)$.
For each $i\in[K]$, define $\tilde{\mathcal{R}}_T(e_i)=\sum_{t=1}^T\inpr{p_t-e_i,\tilde{\ell}_t}$, $\trunc(p_{1:T})=\sum_{t=1}^T\inpr{p_t,\bar{\ell}_t-\tilde{\ell}_t}$, and $\trunc(-e_i)=\sum_{t=1}^T\inpr{-e_i,\bar{\ell}_t-\tilde{\ell}_t}$.

\paragraph{Regret bound}
We now establish a regret upper bound for \cref{alg:adaptive-multiscale-ftrl-componentwise_simple}.
\begin{theorem}
    \label{thm:upper-bound-multiscale-ftrl-componentwise_simple}
    Under \cref{asp:unknown-epsilon_wise}, the regret of \cref{alg:adaptive-multiscale-ftrl-componentwise_simple} satisfies
    \begin{equation*}
        \mathcal{R}_T=O\prn[\big]{(\theta T)^{\frac{1}{1+\epsilon}}\prn{\log K}^{\frac{\epsilon}{1+\epsilon}}}.
    \end{equation*}
\end{theorem}
\begin{proof}
    Choose $i^\ast\in\argmax_{i\in[K]}\mathcal{R}_T(i)$, so that $\mathcal{R}_T=\mathcal{R}_T(i^\ast)$.
    Since $\inpr{p_t,r_t}=0$, the pathwise cumulative regret against $i^\ast$ admits the decomposition $\sum_{t=1}^T r_t(i^\ast)=\tilde{\mathcal{R}}_T(e_{i^\ast})+\trunc(p_{1:T})+\trunc(-e_{i^\ast})$.
    Using \cref{lem:multiscale-truncation-error,lem:ftrl-term-componentwise_simple} gives
    \begin{equation}
        \label{eq:multiscale-ftrl-pathwise-regret_simple}
        \sum_{t=1}^T r_t(i^\ast)\leq C_{\mathrm{G}}G_T+C_{\mathrm{B}}B_T(i^\ast)+\frac{1}{K}\sum_{j=1}^K B_T(j),
    \end{equation}
    where $C_{\mathrm{G}}=3+(2+e/2)\Lambda_{\mathrm{F}}$ and $C_{\mathrm{B}}=2\Lambda_{\mathrm{F}}$.
    Applying Jensen's inequality to the bounds in \cref{lem:known-parameter-centered-moments,lem:multiscale-scale-recursion} with $\mathrm{A}=\mathrm{F}$ yields $\E\brk{G_T},\ \E\brk{B_T(j)}=O\prn{\prn{\theta T}^{\frac{1}{1+\epsilon}}(\log K)^{-\frac{1}{1+\epsilon}}}$ for every $j\in[K]$.
    Taking expectations in \eqref{eq:multiscale-ftrl-pathwise-regret_simple} proves the desired bound.
\end{proof}
Lemmas that we use in the proof of \cref{thm:upper-bound-multiscale-ftrl-componentwise_simple} are stated and proved in \cref{app:subsec:common_analysis,app:subsec:ftrl}.

\subsection{Adaptive OMD Algorithm}
\label{subsec:omd}
\begin{algorithm}[t]
    \caption{Adaptive Heavy-Tailed OMD}
    \label{alg:parameter-free-heavy-tailed-omd}
    \begin{algorithmic}[1]
        \Require {Number of experts $K\geq2$.}
        \State {Initialize $m\gets0$, $H_0\gets K$, and $I_0\gets[K]$.}
        \For{$t=1,\ldots,T$}
            \If{$t\notin I_m$}
                \State {$H_{m+1}\gets KH_m^2$.}
                \State {$I_{m+1}\gets\set{t,\ldots,t+H_{m+1}-1}$.}
                \State {$m\gets m+1$.}
            \EndIf
            \If{$t=\min I_m$}
                \State {Initialize $a_m\gets1/H_m$, $\Lambda_m\gets1+\log(KH_m)$, and $\mathcal{P}_{a_m}\gets\set*{p\in\Delta_K:\min_{i\in[K]}p(i)\geq a_m/K}$.}
                \State {Initialize $G_{t-1}^{(m)}\gets0$, $B_{t-1}^{(m)}(i)\gets0$, $M_{t-1}^{(m)}(i)\gets0$, and $p_t(i)\gets1/K$ for all $i\in[K]$.}
            \ElsIf{$\max_{i\in[K]}M_{t-1}^{(m)}(i)=0$}
                \State {$p_t\gets p_{t-1}$.}
            \Else
                \State {$p_t\!\gets\!\argmin_{p\in\mathcal{P}_{a_m}}\set{\!\inpr{\tilde{\ell}_{t-1},p}+D_{t-1}^{(m)}(p;p_{t-1})\!}$.}
            \EndIf
            \State {Predict $p_t$, incur loss $\inpr{p_t,\ell_t}$, and observe $\ell_t$.}
            \State {$r_t(i)\gets\inpr{p_t,\ell_t}-\ell_t(i)$ for all $i\in[K]$.}
            \State {$d_t(i)\gets\phi_{M_{t-1}^{(m)}(i)}\prn{r_t(i)}$ for all $i\in[K]$, and $g_t\gets\sum_{j=1}^Kp_t(j)d_t(j)$.}
            \State {$G_t^{(m)}\gets G_{t-1}^{(m)}+\frac{g_t}{\Lambda_m}$, $B_t^{(m)}(i)\gets B_{t-1}^{(m)}(i)+\frac{d_t(i)}{\Lambda_m}$, and $M_t^{(m)}(i)\!\gets\!\max\set{G_t^{(m)},B_t^{(m)}(i)}$ for all $i\in[K]$.}
            \State {$\tilde{\ell}_t(i)\!\gets\!-r_t(i)\ind\brk{\abs{r_t(i)}\!\leq\! M_t^{(m)}(i)}$ for all $i\in[K]$.}
        \EndFor
    \end{algorithmic}
\end{algorithm}

Adaptive Heavy-Tailed FTRL is optimal in terms of the worst-case regret, but we do not establish a gap-dependent guarantee for this algorithm.\footnote{It incurs $\Omega(T^{1-\epsilon})$ regret for every $\epsilon\in(0,1)$ and $\Omega(\log T)$ regret for $\epsilon=1$ on self-bounded instances with $C=0$ (see \cref{prop:adaptive-ftrl-self-bounded-counterexample} for details).}
To address this limitation, we propose an OMD-based algorithm that achieves a near-optimal gap-dependent regret bound in the self-bounded environment specified by \cref{asp:self_bounded}.

Within each epoch $m$, the scales, regularizers, and divergences are defined as for FTRL, with an additional epoch index: $\Psi_t^{(m)}(p)=\sum_{i=1}^K M_t^{(m)}(i)p(i)\log p(i)$, and $D_t^{(m)}$ is its Bregman divergence.
Because the one-step mirror update uses $\tilde{\ell}_t$ together with the current divergence $D_t^{(m)}$, truncation at the current scale $M_t^{(m)}(i)$ directly bounds $\abs{\tilde{\ell}_t(i)}/M_t^{(m)}(i)$, so the rescaling factor $\rho_t$ is unnecessary.
Except on the first round of an epoch, the prediction $p_t$ is obtained from the one-step mirror update over $\mathcal{P}_{a_m}$ rather than a cumulative FTRL update over $\Delta_K$.
The constraint $p_t(i)\geq a_m/K$ controls the drift of the time-varying regularizer, and the choice $a_m=1/H_m$ controls comparator smoothing over every prefix of epoch $m$.

\paragraph{Regret bounds}
The following theorems give regret bounds for \cref{alg:parameter-free-heavy-tailed-omd} in the adversarial and self-bounded environments, respectively:
\begin{theorem}
    \label{thm:omd-regret-adversarial}
    Under \cref{asp:unknown-epsilon_wise}, for every integer $T\geq2$, the regret of \cref{alg:parameter-free-heavy-tailed-omd} satisfies
    \[
        \mathcal{R}_T=O\prn[\big]{(\theta T)^{\frac{1}{1+\epsilon}}\prn*{\log(KT)}^{\frac{\epsilon}{1+\epsilon}}}.
    \]
\end{theorem}

\begin{theorem}
    \label{thm:omd-regret-self-bounded}
    For any integer $T\geq2$, suppose that \cref{asp:unknown-epsilon_wise,asp:self_bounded} hold.
    Then \cref{alg:parameter-free-heavy-tailed-omd} satisfies
    \[
        \mathcal{R}_T\!=\!O\!\prn[\bigg]{\!\prn[\bigg]{\!\frac{\theta}{\Delta_{\min}}\!}^{\!\frac{1}{\epsilon}}\!\log(KT)+\!\prn[\bigg]{\!\frac{\theta C}{\Delta_{\min}}\!}^{\!\frac{1}{1+\epsilon}}\!\prn*{\log(KT)}^{\frac{\epsilon}{1+\epsilon}}\!}.
    \]
    In particular, when $C=0$, this bound is $O\prn{\theta^{\frac{1}{\epsilon}}\Delta_{\min}^{-\frac{1}{\epsilon}}\log(KT)}$.
\end{theorem}

Within each epoch, the OMD regret bound depends on the global scale and the scale of expert $i^\ast$.
Combining the resulting moment bounds across epochs yields a bound in terms of $\sum_{t=1}^T\E[1-p_t(i^\ast)]$, to which \cref{asp:self_bounded} applies over the full horizon.
In contrast, the FTRL penalty contains the term $K^{-1}\sum_{j=1}^K B_T(j)$, which prevents the same argument.
The complete proofs are given in \cref{app:subsec:omd}.

\paragraph{Doubling trick.}
We use a doubling trick to remove the need for prior knowledge of $T$.
Our choice of the exponential doubling trick is motivated by \citet{besson2018doublingtrickscantmultiarmed}, who showed that the geometric doubling trick does not generally preserve logarithmic bounds in multi-armed bandits. 
The analysis across epochs is also closely related to the horizon-removal argument of \citet[Section 5.3]{pmlr-v134-ito21a_parameter_free_bandit}. 
Both analyses aggregate data-dependent quantities accumulated over the executed prefix of each epoch.
\citet{pmlr-v134-ito21a_parameter_free_bandit} uses the Cauchy--Schwarz inequality for second-order quantities, whereas we use H\"older's inequality for $(1+\epsilon)$-moment quantities.

\subsection{Known-Parameter BoBW Algorithm}
\label{subsec:non-adaptive-bobw}

The algorithms proposed in \cref{subsec:ftrl,subsec:omd} are adaptive in that they require no prior knowledge of $\epsilon$ or $\theta$.
However, neither adaptive algorithm achieves optimal BoBW guarantees, as Adaptive Heavy-Tailed FTRL lacks an optimal gap-dependent bound and Adaptive Heavy-Tailed OMD incurs an additional $\log T$ factor.
This raises the natural question of whether prior knowledge of $\epsilon$ and $\theta$ can be exploited to achieve optimal regret bounds in both adversarial and self-bounded environments.

\begin{algorithm}[t]
    \caption{Known-Parameter Shannon-FTRL}
    \label{alg:known-parameter-shannon-ftrl}
    \begin{algorithmic}[1]
        \Require {Number of experts $K\geq2$ and known parameters $\epsilon\in(0,1]$ and $\theta>0$.}
        \State {Initialize $\pi(i)\gets1/K$ and $\tildeL_0(i)\gets0$ for every $i\in[K]$.}
        \For{$t=1,2,\ldots,T$}
            \State {$\beta_t\gets\prn{\theta t}^{\frac{1}{1+\epsilon}}(1+\log K)^{-\frac{1}{1+\epsilon}}$.}
            \State {$p_t\gets\argmin_{p\in\Delta_K}\set{\inpr{\tildeL_{t-1},p}+\beta_t\KL(p;\pi)}$.}
            \State {Predict $p_t$, incur loss $\inpr{p_t,\ell_t}$, and observe $\ell_t$.}
            \State {$r_t(i)\gets\inpr{p_t,\ell_t}-\ell_t(i)$ for every $i\in[K]$.}
            \State {$\tildeell_t(i)\gets-r_t(i)\ind\brk*{\abs{r_t(i)}\leq\beta_t}$ and $\tildeL_t(i)\gets\tildeL_{t-1}(i)+\tildeell_t(i)$ for every $i\in[K]$.}
        \EndFor
    \end{algorithmic}
\end{algorithm}

Motivated by the HTINF algorithm of \citet{huang2022adaptive}, which uses prior knowledge of $\epsilon$ and $\theta$ to achieve optimal regret bounds for heavy-tailed bandits, we propose Known-Parameter Shannon-FTRL (\cref{alg:known-parameter-shannon-ftrl}).
The proposed algorithm uses Shannon-entropy FTRL with a time-varying regularization scale determined by $\epsilon$ and $\theta$, and truncates the instantaneous regrets at the same scale to control the effect of heavy-tailed losses.

\paragraph{Regret bounds}
The above algorithm achieves the following regret bounds:
\begin{theorem}
    \label{thm:known-parameter-adversarial}
    Under \cref{asp:unknown-epsilon_wise}, \cref{alg:known-parameter-shannon-ftrl} achieves $\mathcal{R}_T=O\prn{(\theta T)^{\frac{1}{1+\epsilon}}(\log K)^{\frac{\epsilon}{1+\epsilon}}}$.
\end{theorem}
\begin{theorem}
    \label{thm:known-parameter-stochastic}
    Suppose that \cref{asp:unknown-epsilon_wise,asp:self_bounded} hold.
    Then \cref{alg:known-parameter-shannon-ftrl} satisfies
    \[
        \mathcal{R}_T=O\prn[\bigg]{\prn[\bigg]{\frac{\theta}{\Delta_{\min}}}^{\frac{1}{\epsilon}}\log K+\prn[\bigg]{\frac{\theta C}{\Delta_{\min}}}^{\frac{1}{1+\epsilon}}(\log K)^{\frac{\epsilon}{1+\epsilon}}}.
    \]
    In particular, when $C=0$, this bound is $O\prn{\theta^{\frac{1}{\epsilon}}\Delta_{\min}^{-\frac{1}{\epsilon}}\log K}$.
\end{theorem}
These regret bounds avoid the extra $\log T$ factor that appears in \cref{thm:omd-regret-adversarial,thm:omd-regret-self-bounded} and are optimal in both adversarial and self-bounded environments.
The key is that prior knowledge of $\epsilon$ and $\theta$ allows us to choose $\beta_t$ to balance the regularization term $\beta_T\log K$ against the stability and clipping terms of order $\theta\sum_{t=1}^T\beta_t^{-\epsilon}$ at the optimal rate.
The same choice also enables the self-bounded analysis to achieve the optimal rate without an additional $\log T$ factor.
The formal statements and proofs are given in \cref{app:subsec:non-adaptive}.
\looseness=-1

\paragraph{Computational aspects.} 
The updates of \cref{alg:adaptive-multiscale-ftrl-componentwise_simple,alg:parameter-free-heavy-tailed-omd} can be computed efficiently via a line search, such as bisection, following the corresponding FTRL and OMD implementations of \citet{NEURIPS2025_Moulin_Expert_HTEXP}, respectively.
\cref{alg:known-parameter-shannon-ftrl} admits the standard closed-form exponential-weights update.

\section{Conclusion}
\label{sec:conclusion}
Our results answer the three research questions.
For \hyperref[q:minimax]{Q1}, we established the minimax regret rate $(\theta T)^{\frac{1}{1+\epsilon}}(\log K)^{\frac{\epsilon}{1+\epsilon}}$ under the $(1+\epsilon)$-moment condition.
For \hyperref[q:adaptivity]{Q2}, we showed that Adaptive Heavy-Tailed FTRL (\cref{alg:adaptive-multiscale-ftrl-componentwise_simple}) attains this rate without knowing $\epsilon$ or $\theta$.
For \hyperref[q:bobw]{Q3}, we developed BoBW algorithms for every $\epsilon\in(0,1]$: Adaptive Heavy-Tailed OMD (\cref{alg:parameter-free-heavy-tailed-omd}) is near-optimal without knowing the parameters, whereas Known-Parameter Shannon-FTRL (\cref{alg:known-parameter-shannon-ftrl}) is optimal when they are known.

An important direction for future work is to close the logarithmic gap in adaptive BoBW regret bounds.
In particular, it remains open whether an algorithm can attain the optimal adversarial and gap-dependent rates simultaneously without knowing $\epsilon$ or $\theta$, or whether the additional $\log T$ factor is an unavoidable cost of such adaptivity.

\bibliography{reference_HTEXP}
\bibliographystyle{plainnat}

\clearpage
\appendix
\startcontents[appendices]
\section*{Contents of the Appendix}
\printcontents[appendices]{}{1}{}
\clearpage
\section{Additional Related Work}
\label{app:relare_work}
In this section, we provide additional discussion of related work that complements \cref{sec:related_work}.

\paragraph{Heavy-tailed Bandits}

An important related problem is heavy-tailed multi-armed bandits, which differ from our setting in that the learner observes only the loss of the selected arm.
Unlike elsewhere in the paper, in this paragraph, we refer to an algorithm as BoBW if it achieves near-optimal regret in both adversarial and stochastic environments.
The stochastic heavy-tailed bandit problem under a finite $(1+\epsilon)$-moment condition was introduced by \citet{Bubeck2013BanditsWithHeavyTail}, who established an optimal gap-dependent regret bound and a gap-independent regret bound that is optimal up to a logarithmic factor.
\citet{Lai2021MinimaxPolicyHTEXP} subsequently obtained a minimax-optimal regret bound in the stochastic setting.
Moreover, \citet{huang2022adaptive} developed an algorithm that attains optimal regret rates in both stochastic and adversarial environments when the heavy-tail parameters are known, under the truncated nonnegativity assumption on the optimal arm.
Here, a random variable $X$ is truncated nonnegative if $\E\brk*{X\ind\brk*{\abs{X}>M}}\geq0$ for every $M\geq0$.
\citet{NEURIPS2024_Cheng_HeavyTailed_bandits} further established near-optimal high-probability BoBW regret guarantees without the truncated nonnegativity assumption when the heavy-tail parameters are known.
Adaptation to unknown heavy-tail parameters has also been studied extensively.
\citet{huang2022adaptive}, \citet{pmlr-v247-genalti24a}, and \citet{chen2025uniinf_HTMAB} developed algorithms adaptive to the heavy-tail parameters, and \citet{chen2025uniinf_HTMAB} in particular proposed a parameter-free BoBW algorithm attaining near-optimal regret rates in both stochastic and adversarial environments up to logarithmic factors.
These guarantees, however, rely on the truncated nonnegativity assumption, and \citet{pmlr-v247-genalti24a} showed that without additional assumptions one cannot in general obtain the same regret rates as the algorithms that use the heavy-tail parameters.
More recently, \citet{genalti2026parameterfreeheavytailedbandits} characterized the optimal regret trade-offs for parameter-free heavy-tailed bandits and developed algorithms attaining the resulting adaptation frontiers up to logarithmic factors.

\paragraph{Self-bounded environment}
\citet{Lykouris2018StochasticBanditsRobustToAdversarialCorruptions} and \citet{pmlr-v99-gupta19a_stochastic_bandits_adversarial_corruptions} studied adversarial corruptions in stochastic multi-armed bandits. 
\citet{NEURIPS2020_amir_prediction_expert} subsequently considered an analogous corrupted stochastic setting for the experts problem. 
Rather than explicitly specifying a corruption process, \citet{pmlr-v89-zimmert19a_bobw} introduced, for multi-armed bandits, the more general adversarial regime with a self-bounded constraint, which subsumes stochastic bandits with adversarial corruptions.
\citet{NEURIPS2021_ito_robustness_adversarial_decision} studied such self-bounded environments for both multi-armed bandits and the experts problem.
For the experts problem, the same self-bounded constraint is used in \cref{asp:self_bounded} and by \citet{NEURIPS2025_Moulin_Expert_HTEXP}.

\paragraph{Heavy-tailed OCO}
In addition to heavy-tailed bandits, online convex optimization (OCO), which includes prediction with expert advice as a special case, has also been studied under heavy-tailed feedback \citep{NEURIPS2022_Zhang_ParameterFreeRegret_HTOCO,pmlr-v313-liu26a_oco_heavy_tails}. 
In particular, \citet{pmlr-v313-liu26a_oco_heavy_tails} shows that some existing OCO algorithms continue to enjoy regret guarantees in a heavy-tailed setting.
However, their setting assumes a $(1+\epsilon)$-moment bound on the Euclidean norm of the centered stochastic gradient, whereas we assume only componentwise moment bounds on the loss vector.
Deriving their norm-based assumption from our condition generally incurs a dependence on the number of experts.
Consequently, directly applying their OCO results to our setting does not yield the logarithmic dependence on the number of experts that we aim for.

\section{Additional Discussion of LoOT-Free Algorithms}
\label{app:limitations_existing_algorithms}
In this section, we present three results concerning the algorithms of \citet{NEURIPS2025_Moulin_Expert_HTEXP}: we establish their suboptimality under the general $(1+\epsilon)$-moment condition, derive the upper bound achieved by their LoOT-Free FTRL, and show that a natural extension of LoOT-Free FTRL achieves the optimal regret bound when $\epsilon$ is known.

\subsection{Algorithm-Specific Lower Bounds for LoOT-Free OMD and FTRL}
\label{app:subsec:alg_specific_lower_bound}
\begin{algorithm}[t]
    \caption{LoOT-Free OMD and FTRL \citep[Algorithms~1 and~2]{NEURIPS2025_Moulin_Expert_HTEXP}}
    \label{alg:loot-free-variants-lower-bound}
    \begin{algorithmic}[1]
        \Require {Number of experts $K\geq 2$, variant $\mathrm{A}\in\{\mathrm{O},\mathrm{F}\}$, where $\mathrm{O}$ denotes OMD and $\mathrm{F}$ denotes FTRL.}
        \State {Initialize $p_1(i)\gets1/K$, $V_0(i)\gets0$, and $b_0(i)\gets0$ for every $i\in[K]$, and $\bar V_0\gets0$.}
        \State {For $\mathrm{A}=\mathrm{O}$, initialize $\beta_{\mathrm{A}}\gets\sqrt{\log(KT)}$, $a\leftarrow 1/T$ and $\mathcal P_{\mathrm{A}}\gets\set*{p\in\Delta_K:p(i)\geq a/K\text{ for every }i\in[K]}$.}
        \State {For $\mathrm{A}=\mathrm{F}$, initialize $\beta_{\mathrm{A}}\gets\sqrt{\log K}$ and $\mathcal P_{\mathrm{A}}\gets\Delta_K$.}
        \For{$t=1,2,\ldots,T$}
            \State {Predict $p_t$, incur loss $\inpr{p_t,\ell_t}$, and observe $\ell_t$.}
            \State {$r_t(i)\gets\inpr{p_t,\ell_t}-\ell_t(i)$ and $v_t(i)\gets r_t(i)^2$ for every $i\in[K]$.}
            \State {$\bar v_t\gets\sum_{i=1}^Kp_t(i)v_t(i)$, $V_t(i)\gets V_{t-1}(i)+v_t(i)$ for every $i\in[K]$, and $\bar V_t\gets\bar V_{t-1}+\bar v_t$.}
            \State {$b_t(i)\gets\max\set*{\bar V_t,V_t(i)}^{\frac{1}{2}}$ for every $i\in[K]$.}
            \If{$\bar V_t=0$}
                \State {$\tilde\ell_t\gets0$ and $p_{t+1}\gets p_t$.}
            \Else
                \State {$\eta_{t,i}\gets\beta_{\mathrm{A}}/b_t(i)$ for every $i\in[K]$.}
                \If{$\mathrm{A}=\mathrm{O}$}
                    \State {$\tilde\ell_t(i)\gets-r_t(i)\ind\brk*{\abs{r_t(i)}\leq1/\eta_{t,i}}$ for every $i\in[K]$.}
                    \State {$p_{t+1}\gets\argmin_{p\in\mathcal P_{\mathrm{A}}}\set{\inpr{p,\tilde\ell_t}+D_t(p;p_t)}$.}
                \Else
                    \State {$\tilde\ell_t(i)\gets-r_t(i)\frac{b_{t-1}(i)}{b_t(i)}\ind\brk*{\abs{r_t(i)}\leq1/\eta_{t,i}}$ for every $i\in[K]$.}
                    \State {$p_{t+1}\gets\argmin_{p\in\mathcal P_{\mathrm{A}}}\set{\sum_{s=1}^t\inpr{p,\tilde\ell_s}+D_t(p;p_1)}$.}
                \EndIf
            \EndIf
        \EndFor
    \end{algorithmic}
\end{algorithm}
For completeness, we provide the pseudocode of LoOT-Free OMD and FTRL (\cref{alg:loot-free-variants-lower-bound}).
Throughout this subsection, we set the learning-rate parameters to $\beta_{\mathrm{O}}=\sqrt{\log(KT)}$ for LoOT-Free OMD and $\beta_{\mathrm{F}}=\sqrt{\log K}$ for LoOT-Free FTRL.
They share the same instantaneous regrets and second-order scales, whereas their parameter choices, domains, surrogate losses, and online-learning updates are specified by the two branches below.
For $t\geq1$, write $v_t(i)=r_t(i)^2$, $\bar v_t=\sum_{i=1}^Kp_t(i)v_t(i)$, $V_t(i)=\sum_{s=1}^tv_s(i)$, $\bar V_t=\sum_{s=1}^t\bar v_s$, and $b_t(i)=\max\set*{\bar V_t,V_t(i)}^{\frac{1}{2}}$, and define
\[
    D_t(p;q)=\sum_{i=1}^K\frac{1}{\eta_{t,i}}\brk*{p(i)\log\frac{p(i)}{q(i)}-p(i)+q(i)}.
\]

Thus, a zero-scale round is skipped: the prediction and all cumulative scales remain unchanged.
The OMD specialization uses the simplex-truncation parameter $1/T$, whereas the FTRL specialization uses the rescaling factor $b_{t-1}(i)/b_t(i)$.

\begin{proposition}
    \label{prop:moulin-balanced-sparse-spike}
    Fix $\epsilon\in(0,1]$, $\theta>0$, $K\geq16$, and $T\geq4$, and let $S=(\theta T)^{\frac{1}{1+\epsilon}}K^{\frac{1-\epsilon}{2(1+\epsilon)}}$.
    There exists a loss process satisfying \cref{asp:unknown-epsilon_wise} with a unique optimal expert such that LoOT-Free OMD with simplex-truncation parameter $1/T$ and learning-rate parameter $\beta_{\mathrm{O}}=\sqrt{\log(KT)}$ incurs
    \[ \mathcal{R}_T^{\mathrm{OMD}}=\Omega\prn*{\min\set*{\theta^{\frac{1}{1+\epsilon}}T,\frac{S}{\sqrt{\log(KT)}}}}, \]
    whereas LoOT-Free FTRL with learning-rate parameter $\beta_{\mathrm{F}}=\sqrt{\log K}$ incurs
    \[ \mathcal{R}_T^{\mathrm{FTRL}}=\Omega\prn*{\min\set*{\theta^{\frac{1}{1+\epsilon}}T,\frac{ST}{\sqrt{\log(KT)}},S\sqrt{\log K}}}. \]
    In particular, if $T^{\frac{\epsilon}{1+\epsilon}}\geq cK^{\frac{1-\epsilon}{2(1+\epsilon)}}\sqrt{\log K}$ for a sufficiently large universal constant $c>0$, then these bounds simplify, respectively, to $\mathcal{R}_T^{\mathrm{OMD}}=\Omega\prn[\Big]{\frac{S}{\sqrt{\log(KT)}}}$ and $\mathcal{R}_T^{\mathrm{FTRL}}=\Omega\prn{S\sqrt{\log K}}$.
\end{proposition}

\begin{proof}
    Fix an arbitrary expert $i^\ast\in[K]$.
    Let $m=\lfloor T/2\rfloor$ and $B=(\theta Km/2)^{\frac{1}{1+\epsilon}}$.
    Fix a sufficiently small universal constant $c_{\mathrm{gap}}>0$ and set $\delta=c_{\mathrm{gap}}\min\set{\theta^{\frac{1}{1+\epsilon}},S/\sqrt{\log(KT)}}$.
    Let $u=\frac{1}{K}\mathbf{1}$ denote the uniform distribution.
    Define $b_{\mathrm{sp}}=B\sqrt{2/K}$.
    We call the following two-phase loss process the balanced sparse-spike instance.
    This instance is balanced because the two nonzero spike coordinates have opposite signs and sparse because all other coordinates are zero.
    During the first phase, for each $t=1,\ldots,m$, independently draw $Z_t\sim\operatorname{Bernoulli}(1/m)$ and an ordered pair $(J_t^+,J_t^-)$ uniformly from the $K(K-1)$ ordered pairs of distinct coordinates, independently of $Z_t$, and set
    \[
        \ell_t(i)=\begin{cases}
            Z_tB, & i=J_t^+,\\
            -Z_tB, & i=J_t^-,\\
            0, & \text{otherwise}.
        \end{cases}
    \]
    During the second phase, for $t=m+1,\ldots,T$, set $\ell_t(i^\ast)=0$ and $\ell_t(i)=\delta$ for every $i\neq i^\ast$.

    We first verify the assumptions.
    For every $t\leq m$ and $i\in[K]$, each coordinate is selected as one of the two spike coordinates with probability $2/K$, and hence
    \[
        \E_t\brk*{\abs{\ell_t(i)}^{1+\epsilon}}=\frac{1}{m}\frac{2}{K}B^{1+\epsilon}=\theta.
    \]
    For every $t>m$, the same conditional moment is at most $\delta^{1+\epsilon}\leq\theta$.
    All experts have zero conditional mean loss in the first phase, whereas in the second phase expert $i^\ast$ has loss zero and every other expert has loss $\delta$.
    Consequently, expert $i^\ast$ is uniquely optimal and
    \[
        \mathcal{R}_T=\delta\E\brk*{\sum_{t=m+1}^T\prn*{1-p_t(i^\ast)}}.
    \]

    We next use the second-order quantities defined in \cref{alg:loot-free-variants-lower-bound}.
    Let $\mathcal{E}$ be the event that exactly one spike occurs in the first phase.
    Its probability satisfies $\P(\mathcal{E})=\prn*{1-\frac{1}{m}}^{m-1}\geq e^{-1}$.
    On $\mathcal{E}$, let $\tau$ be the unique spike round.
    Since all losses before round $\tau$ are zero, for every $t\leq\tau$, we have $p_t=u$ and $V_{t-1}(i)=\bar V_{t-1}=b_{t-1}(i)=0$ for every $i\in[K]$.
    Therefore, we have $\inpr{p_\tau,\ell_\tau}=0$, $r_\tau(J_\tau^+)=-B$, and $r_\tau(J_\tau^-)=B$, whereas all other instantaneous regrets vanish, and
    \[
        \bar{v}_\tau=\frac{2B^2}{K}=b_{\mathrm{sp}}^2.
    \]
    It follows that $b_\tau(J_\tau^+)=b_\tau(J_\tau^-)=B$, $b_\tau(i)=b_{\mathrm{sp}}$ at every other coordinate, and $b_t(i)\geq b_{\mathrm{sp}}$ for every $t\geq\tau$ and $i\in[K]$.
    For OMD, since $\beta_{\mathrm{O}}=\sqrt{\log(KT)}$, we have $1/\eta_{\tau,J_\tau^+}=1/\eta_{\tau,J_\tau^-}=B/\sqrt{\log(KT)}<B$, so both nonzero instantaneous regrets are truncated and $\tilde{\ell}_\tau=0$.
    For FTRL, $b_{\tau-1}(i)=0$ for every $i\in[K]$, so the rescaling factor $b_{\tau-1}(i)/b_\tau(i)$ also gives $\tilde{\ell}_\tau=0$.
    Since $\tilde{\ell}_\tau=0$ and all losses after round $\tau$ in the first phase vanish on $\mathcal E$, neither algorithm updates its prediction during the first phase. Hence $p_{m+1}=u$, whereas $b_m(i)\geq b_{\mathrm{sp}}$ for every $i\in[K]$.

    The spike-induced scale satisfies
    \[
        b_{\mathrm{sp}}=2^{\frac{1}{2}-\frac{1}{1+\epsilon}}\theta^{\frac{1}{1+\epsilon}}m^{\frac{1}{1+\epsilon}}K^{\frac{1-\epsilon}{2(1+\epsilon)}}=\Theta\prn[\big]{(\theta T)^{\frac{1}{1+\epsilon}}K^{\frac{1-\epsilon}{2(1+\epsilon)}}}=\Theta(S).
    \]
    By choosing $c_{\mathrm{gap}}$ sufficiently small, we therefore have
    \[
        \delta\leq\frac{b_{\mathrm{sp}}}{\sqrt{\log(KT)}}\leq\frac{b_{\mathrm{sp}}}{\sqrt{\log K}}.
    \]

    We first consider OMD.
    Its learning rate is $\eta_{t,i}=\sqrt{\log(KT)}/b_t(i)$, and the Bregman divergence induced by its multi-scale entropy is
    \[
        D_t(p;q)=\frac{1}{\sqrt{\log(KT)}}\sum_{i=1}^Kb_t(i)\brk*{p(i)\log\frac{p(i)}{q(i)}-p(i)+q(i)}.
    \]
    During the second phase, we have $r_t(i^\ast)=\delta\prn*{1-p_t(i^\ast)}$ and $r_t(i)=-\delta p_t(i^\ast)$ for every $i\neq i^\ast$.
    Since $\delta\leq b_{\mathrm{sp}}/\sqrt{\log(KT)}\leq b_t(i)/\sqrt{\log(KT)}=1/\eta_{t,i}$, no second-phase instantaneous regret is truncated, and thus $\tilde{\ell}_t=-r_t$ and $\max_i\tilde{\ell}_t(i)-\min_i\tilde{\ell}_t(i)=\delta$.
    The optimality of the OMD update, the lower bound $b_t(i)\geq b_{\mathrm{sp}}$, and Pinsker's inequality give
    \begin{align*}
        \frac{b_{\mathrm{sp}}}{2\sqrt{\log(KT)}}\nrm{p_{t+1}-p_t}_1^2
        &\leq D_t(p_{t+1};p_t)\leq\inpr{p_t-p_{t+1},\tilde{\ell}_t}\leq\frac{\delta}{2}\nrm{p_{t+1}-p_t}_1.
    \end{align*}
    Consequently, we have
    \[
        \nrm{p_{t+1}-p_t}_1\leq\frac{\delta\sqrt{\log(KT)}}{b_{\mathrm{sp}}}\quad\text{and}\quad p_{t+1}(i^\ast)-p_t(i^\ast)\leq\frac{\delta\sqrt{\log(KT)}}{2b_{\mathrm{sp}}}.
    \]
    Starting from $p_{m+1}(i^\ast)=1/K$, at least $\frac{(1-2/K)b_{\mathrm{sp}}}{\delta\sqrt{\log(KT)}}$ second-phase updates are required for the optimal coordinate to reach $1/2$, up to integer rounding.
    Before the optimal coordinate reaches $1/2$, every round incurs regret at least $\delta/2$ against expert $i^\ast$.
    Since $K\geq16$, there is a universal constant $c_0>0$ such that, conditionally on $\mathcal{E}$, the second-phase regret of OMD is at least $c_0\min\set*{\delta(T-m),\frac{b_{\mathrm{sp}}}{\sqrt{\log(KT)}}}$.

    We now consider FTRL, for which $\beta_{\mathrm{F}}=\sqrt{\log K}$ and the Bregman divergence is
    \[
        D_t(p;q)=\frac{1}{\sqrt{\log K}}\sum_{i=1}^Kb_t(i)\brk*{p(i)\log\frac{p(i)}{q(i)}-p(i)+q(i)}.
    \]
    For every second-phase round, let $a_t(i)=b_{t-1}(i)/b_t(i)\in[0,1]$.
    Again, the condition $\delta\leq b_{\mathrm{sp}}/\sqrt{\log K}$ prevents truncation, and the FTRL surrogate loss satisfies $\tilde{\ell}_t(i^\ast)=-a_t(i^\ast)\delta\prn*{1-p_t(i^\ast)}$ and $\tilde{\ell}_t(i)=a_t(i)\delta p_t(i^\ast)$ for every $i\neq i^\ast$.
    Since $a_t(i)\in[0,1]$, we have $-\delta\prn*{1-p_t(i^\ast)}\leq\tilde{\ell}_t(i^\ast)\leq0\leq\tilde{\ell}_t(i)\leq\delta p_t(i^\ast)$ for every $i\neq i^\ast$, and hence $\max_i\tilde{\ell}_t(i)-\min_i\tilde{\ell}_t(i)\leq\delta\prn*{1-p_t(i^\ast)}+\delta p_t(i^\ast)=\delta$.
    For $n\geq0$, define $L_0=0$, $L_n=\sum_{s=m+1}^{m+n}\tilde{\ell}_s$, and $p^{(n)}=p_{m+n+1}$.
    The preceding one-step range bound implies $\max_iL_n(i)-\min_iL_n(i)\leq n\delta$.
    The FTRL update gives
    \[
        p^{(n)}\in\argmin_{p\in\Delta_K}\set*{\inpr{p,L_n}+D_{m+n}(p;u)}.
    \]
    Comparing its objective value with that of $u$ and using the range bound yields
    \[
        D_{m+n}(p^{(n)};u)\leq\inpr{u-p^{(n)},L_n}\leq n\delta.
    \]
    On the other hand, $x\log(x/y)-x+y\geq0$ for all $x,y\geq0$, with the standard boundary conventions.
    Since $p^{(n)},u\in\Delta_K$, we also have $\sum_{i=1}^K\prn*{p^{(n)}(i)\log\frac{p^{(n)}(i)}{u(i)}-p^{(n)}(i)+u(i)}=\KL(p^{(n)};u)$.
    Therefore, the bound $b_{m+n}(i)\geq b_{\mathrm{sp}}$ implies
    \[
        D_{m+n}(p^{(n)};u)\geq\frac{b_{\mathrm{sp}}}{\sqrt{\log K}}\KL(p^{(n)};u).
    \]
    For $K\geq16$, the log-sum inequality gives
    \[
        p^{(n)}(i^\ast)\geq\frac{1}{2}\quad\Longrightarrow\quad\KL(p^{(n)};u)\geq\frac{1}{4}\log K.
    \]
    Therefore, $p^{(n)}(i^\ast)$ cannot reach $1/2$ unless
    \[
        n\delta\geq\frac{b_{\mathrm{sp}}\sqrt{\log K}}{4}.
    \]
    Each preceding second-phase round incurs regret at least $\delta/2$, so, conditionally on $\mathcal{E}$, the second-phase regret of FTRL is at least
    \[
        c_0\min\set*{\delta(T-m),b_{\mathrm{sp}}\sqrt{\log K}}.
    \]

    Finally, the expected first-phase regret against expert $i^\ast$ is zero because every first-phase loss vector has conditional mean zero, whereas the second-phase regret is nonnegative on every sample path.
    Since $T-m=\Theta(T)$, $b_{\mathrm{sp}}=\Theta(S)$, and $\delta=c_{\mathrm{gap}}\min\set{\theta^{\frac{1}{1+\epsilon}},S/\sqrt{\log(KT)}}$, multiplying the two conditional lower bounds by $\P(\mathcal{E})\geq e^{-1}$ gives
    \[
        \mathcal{R}_T^{\mathrm{OMD}}=\Omega\prn*{\min\set*{\theta^{\frac{1}{1+\epsilon}}T,\frac{S}{\sqrt{\log(KT)}}}}
    \]
    and
    \[
        \mathcal{R}_T^{\mathrm{FTRL}}=\Omega\prn*{\min\set*{\theta^{\frac{1}{1+\epsilon}}T,\frac{ST}{\sqrt{\log(KT)}},S\sqrt{\log K}}}.
    \]
    Under the condition $T^{\frac{\epsilon}{1+\epsilon}}\geq cK^{\frac{1-\epsilon}{2(1+\epsilon)}}\sqrt{\log K}$ and for $K\geq16$, the preceding bounds yield the simplified bounds.
\end{proof}

\subsection{Upper bound by \texorpdfstring{\citet[Algorithm 2]{NEURIPS2025_Moulin_Expert_HTEXP}}{Moulin et al. (2025)}}
In this section, we analyze the regret upper bound obtained by the result of \citet{NEURIPS2025_Moulin_Expert_HTEXP} under \cref{asp:unknown-epsilon_wise}.
Although \citet{NEURIPS2025_Moulin_Expert_HTEXP} establish their guarantee only for $\epsilon=1$, the same algorithm can be analyzed under \cref{asp:unknown-epsilon_wise} to yield an upper bound that is sublinear in both $T$ and $K$ for every fixed $\epsilon\in(0,1)$.
However, the resulting bound is suboptimal in its dependence on $K$.

\begin{corollary}
    \label{cor:moulin-ftrl-componentwise}
    Suppose that $K\geq2$ and \cref{asp:unknown-epsilon_wise} holds.
    \citet[Algorithm 2]{NEURIPS2025_Moulin_Expert_HTEXP} with $\beta=\sqrt{\log K}$ achieves the regret upper bound
    \[
        \mathcal{R}_T
        \leq2(\theta T)^{\frac{1}{1+\epsilon}}\brk*{3\sqrt{\log K}+\prn*{\frac{5}{\sqrt{\log K}}+6\sqrt{\log K}}K^{\frac{1-\epsilon}{2(1+\epsilon)}}+\frac{1}{\sqrt{\log K}}}.
    \]
    In particular, the algorithm requires no prior knowledge of $\epsilon$ or $\theta$.
\end{corollary}
This upper bound matches the simplified algorithm-specific lower bound for LoOT-Free FTRL in \cref{prop:moulin-balanced-sparse-spike} up to constant factors.
\begin{proof}
    Recall that $v_t(i)=r_t(i)^2$, $\bar{v}_t=\sum_{i=1}^Kp_t(i)v_t(i)$, $V_T(i)=\sum_{t=1}^Tv_t(i)$, and $\bar{V}_T=\sum_{t=1}^T\bar{v}_t$.
    \citet[Theorem C.1]{NEURIPS2025_Moulin_Expert_HTEXP} states that, for every $i^\ast\in[K]$ and every $\beta>0$, the following bound holds:
    \[
        \sum_{t=1}^T\inpr{p_t-e_{i^\ast},\ell_t}
        \leq\prn*{\frac{\log K}{\beta}+2\beta}\sqrt{V_T(i^\ast)}+\prn*{\frac{5+\log K}{\beta}+5\beta}\sqrt{\bar{V}_T}+\frac{1}{\beta K}\sum_{i=1}^K\sqrt{V_T(i)}.
    \]
    Since $1+\epsilon\leq2$, monotonicity of finite-dimensional norms, Jensen's inequality, and \cref{lem:known-parameter-centered-moments} imply, for every $i\in[K]$,
    \begin{equation*}
        \E\brk*{\sqrt{V_T(i)}}
        \leq\E\brk*{\prn*{\sum_{t=1}^T\abs{r_t(i)}^{1+\epsilon}}^{\frac{1}{1+\epsilon}}}
        \leq\prn*{\sum_{t=1}^T\E\brk*{\abs{r_t(i)}^{1+\epsilon}}}^{\frac{1}{1+\epsilon}}
        \leq2(\theta T)^{\frac{1}{1+\epsilon}}.
    \end{equation*}
    To control $\E[\sqrt{\bar{V}_T}]$, the subadditivity of $x\mapsto x^{\frac{1+\epsilon}{2}}$ and \cref{lem:known-parameter-centered-moments} give
    \begin{equation*}
        \E_t\brk*{\bar{v}_t^{\frac{1+\epsilon}{2}}}
        \leq\sum_{i=1}^Kp_t(i)^{\frac{1+\epsilon}{2}}\E_t\brk*{\abs{r_t(i)}^{1+\epsilon}}
        \leq2^{1+\epsilon}\theta\sum_{i=1}^Kp_t(i)^{\frac{1+\epsilon}{2}}
        \leq2^{1+\epsilon}\theta K^{\frac{1-\epsilon}{2}},
    \end{equation*}
    where we used Jensen's inequality in the last inequality.
    Applying norm monotonicity and Jensen's inequality once more yields
    \begin{equation*}
        \E\brk*{\sqrt{\bar{V}_T}}
        \leq\E\brk*{\prn*{\sum_{t=1}^T\bar{v}_t^{\frac{1+\epsilon}{2}}}^{\frac{1}{1+\epsilon}}}
        \leq\prn*{\sum_{t=1}^T\E\brk*{\bar{v}_t^{\frac{1+\epsilon}{2}}}}^{\frac{1}{1+\epsilon}}
        \leq2(\theta T)^{\frac{1}{1+\epsilon}}K^{\frac{1-\epsilon}{2(1+\epsilon)}}.
    \end{equation*}
    Taking expectations and choosing $i^\ast\in\argmax_{i\in[K]}\mathcal{R}_T(i)$ therefore give, for every $\beta>0$,
    \[
        \mathcal{R}_T
        \leq2(\theta T)^{\frac{1}{1+\epsilon}}\brk*{\frac{\log K}{\beta}+2\beta+\prn*{\frac{5+\log K}{\beta}+5\beta}K^{\frac{1-\epsilon}{2(1+\epsilon)}}+\frac{1}{\beta}}.
    \]
    Substituting $\beta=\sqrt{\log K}$ proves the explicit bound.
\end{proof}

\subsection{A Natural Extension of \texorpdfstring{\citet[Algorithm 2]{NEURIPS2025_Moulin_Expert_HTEXP}}{Moulin et al. (2025)}}
\label{app:subsec:epsilon_known_moulin}
In this section, we show how a natural extension of \citet[Algorithm 2]{NEURIPS2025_Moulin_Expert_HTEXP} achieves the optimal gap-independent regret bound when $\epsilon$ is known.

The suboptimal dependence on $K$ in \cref{cor:moulin-ftrl-componentwise} is caused by using quadratic scale estimates under a $(1+\epsilon)$-moment condition, rather than by the FTRL structure itself.
When $\epsilon$ is known, a direct extension of LoOT-Free FTRL removes this dependence.

Specifically, after computing the instantaneous regret $r_t(i)=\inpr{p_t,\ell_t}-\ell_t(i)$, replace the quadratic quantities in the original algorithm with
\[
    v_t(i)=\abs{r_t(i)}^{1+\epsilon},\qquad \bar{v}_t=\sum_{i=1}^Kp_t(i)v_t(i),\quad\text{and}\quad b_t(i)=\max\set*{\sum_{s=1}^t\bar{v}_s,\sum_{s=1}^tv_s(i)}^{\frac{1}{1+\epsilon}}.
\]
The learning rate retains the form $\eta_{t,i}=\beta_{\mathrm{F}}/b_t(i)$, but the new definition of $b_t(i)$ makes its scale appropriate for the general moment condition.
The truncation and rescaling step is otherwise unchanged: use $\tilde{\ell}_t(i)=-r_t(i)\frac{b_{t-1}(i)}{b_t(i)}\ind\brk*{\abs{r_t(i)}\leq1/\eta_{t,i}}$, and perform the same multi-scale entropic FTRL update as in \cref{alg:loot-free-variants-lower-bound}.
Thus, the modification only changes the empirical scale from a cumulative second moment to a cumulative $(1+\epsilon)$-moment and changes the corresponding root from $1/2$ to $1/(1+\epsilon)$.

We briefly outline the corresponding analysis.
The quadratic self-normalized sums in the original proof are replaced by $(1+\epsilon)$-moment sums, using bounds of the form $\sum_{t=1}^Tx_t/(\sum_{s=1}^tx_s)^{\frac{\epsilon}{1+\epsilon}}\leq(1+\epsilon)(\sum_{t=1}^Tx_t)^{\frac{1}{1+\epsilon}}$.
Moreover, \cref{asp:unknown-epsilon_wise} gives $\E_t[\abs{r_t(i)}^{1+\epsilon}]\leq2^{1+\epsilon}\theta$ for every expert $i$ and round $t$.
Therefore, for every $i\in[K]$, both $\E\brk{(\sum_{t=1}^T\bar{v}_t)^{\frac{1}{1+\epsilon}}}$ and $\E\brk{(\sum_{t=1}^Tv_t(i))^{\frac{1}{1+\epsilon}}}$ are $O\prn{(\theta T)^{\frac{1}{1+\epsilon}}}$.
Consequently, choosing $\beta_{\mathrm{F}}=(\log K)^{\frac{1}{1+\epsilon}}$ yields
\[
    \mathcal{R}_T=O\prn*{(\theta T)^{\frac{1}{1+\epsilon}}(\log K)^{\frac{\epsilon}{1+\epsilon}}}.
\]
This matches the gap-independent lower bound in \cref{cor:lower_bound_gap_independent}.
The algorithm requires knowledge of $\epsilon$ to form $v_t(i)$ and choose $\beta_{\mathrm{F}}$, but it does not require knowledge of $\theta$.

Thus, when $\epsilon$ is known, replacing the quadratic scale used in the original LoOT-Free FTRL algorithm with a scale adapted to the $(1+\epsilon)$-moment condition removes the additional polynomial dependence on $K$ appearing in \cref{cor:moulin-ftrl-componentwise}.
This shows that the additional dependence is caused by the mismatch between the quadratic scale and \cref{asp:unknown-epsilon_wise}, rather than by the FTRL framework itself.

\section{Lower Bounds}
\label{app:lower_bound}
In this section, we prove \cref{thm:lower_bound_gap_dependent,thm:lower_bound_self_bounded_simple} in \cref{sec:lower_bound}.
We first introduce a family of rare-spike instances and a common prefix lower bound used in the proofs of the theorems below.

For fixed $\epsilon\in(0,1]$ and $\theta>0$, define $q\colon[0,\infty)\to[0,\infty)$ by $q(x)=\theta^{-\frac{1}{\epsilon}}x^{1+\frac{1}{\epsilon}}$.
For $\delta>0$ satisfying $q(\delta)\leq1/8$ and $i^\ast\in[K]$, let $\mathsf{I}(\delta,i^\ast)$ denote the i.i.d.\ instance in which expert $i^\ast$ always incurs zero loss, whereas every expert $i\neq i^\ast$ independently incurs $B_\delta=(\theta/\delta)^{\frac{1}{\epsilon}}$ with probability $q(\delta)$ and zero otherwise.
Thus, every suboptimal expert has expected gap $\delta$ and $(1+\epsilon)$-moment $\theta$ under $\mathsf{I}(\delta,i^\ast)$.

\begin{lemma}
    \label{lem:heavy-tailed-prefix-lower-bound}
    Fix $\epsilon\in(0,1]$, $\theta>0$, $K\geq4$, and $\delta>0$.
    Suppose that $q(\delta)\leq1/8$ and that an integer $n\geq1$ satisfies $nq(\delta)\leq(\log K)/8$.
    Then, for any algorithm, there exists $i^\ast\in[K]$ such that its regret over the first $n$ rounds of $\mathsf{I}(\delta,i^\ast)$ satisfies $\mathcal{R}_n(i^\ast)\geq\delta n/8$.
\end{lemma}

\begin{proof}
    For notational simplicity, we consider a deterministic learner.
    An internal random seed for a randomized learner can be included in the common auxiliary distribution introduced below.
    For each $i^\ast\in[K]$, let $\P_{i^\ast}$ denote the distribution of the first $n$ rounds of $\mathsf{I}(\delta,i^\ast)$, and let $\mathbb{Q}$ denote the symmetric distribution under which every coordinate is independently equal to $B_\delta$ with probability $q(\delta)$ and zero otherwise.
    Since $q(\delta)B_\delta=\delta$, we have
    \[ \mathcal{R}_n(i^\ast)=\delta\sum_{t=1}^n\E_{i^\ast}\brk*{1-p_t(i^\ast)}. \]
    Following the proof of \citet[Proposition 2]{mourtada2019optimality_expert}, introduce independent uniform random variables $U_1,\dots,U_n$ on $[0,1]$ and, for every $p\in\Delta_K$, a measurable map $I_p\colon[0,1]\to[K]$ such that $I_p(U)=i$ with probability $p(i)$ when $U$ is uniform on $[0,1]$.
    Define $I_t=I_{p_t}(U_t)$, $N_i=\sum_{t=1}^n\ind\brk{I_t=i}$, and $\mathcal{E}_i=\set*{N_i>n/2}$, and let $\tilde{\P}_{i^\ast}$ and $\tilde{\mathbb{Q}}$ denote the respective distributions augmented by the variables $U_1,\dots,U_n$.
    The events $\mathcal{E}_1,\dots,\mathcal{E}_K$ are pairwise disjoint, and
    \[ \sum_{t=1}^n\E_{i^\ast}\brk*{1-p_t(i^\ast)}\geq\frac{n}{2}\prn*{1-\tilde{\P}_{i^\ast}\prn*{\mathcal{E}_{i^\ast}}}. \]
    Moreover, we have
    \begin{align*}
        \mathrm{KL}\prn{\tilde{\P}_{i^\ast},\tilde{\mathbb{Q}}}
        &=-n\log(1-q(\delta))
        \leq2nq(\delta)\leq\frac{\log K}{4}.
    \end{align*}
    The form of Fano's inequality used by \citet[Proposition 2]{mourtada2019optimality_expert} and $K\geq4$ therefore give
    \[ \frac{1}{K}\sum_{i^\ast=1}^K\tilde{\P}_{i^\ast}\prn*{\mathcal{E}_{i^\ast}}\leq\frac{3}{4}. \]
    Averaging the preceding regret lower bound over $i^\ast\in[K]$ proves that $\sup_{i^\ast\in[K]}\mathcal{R}_n(i^\ast)\geq\delta n/8$.
\end{proof}

\subsection{Proof of \cref{thm:lower_bound_gap_dependent}}
\label{app:subsec:lower_bound_gap_dependent}
We provide the proof of \cref{thm:lower_bound_gap_dependent}.
\begin{proof}
    Set
    \[ T_0=\floor*{\frac{\log K}{8q(\Delta_{\min})}}=\floor*{\frac{\theta^{\frac{1}{\epsilon}}\log K}{8\Delta_{\min}^{1+\frac{1}{\epsilon}}}}. \]
    The condition on $\Delta_{\min}$ gives $q(\Delta_{\min})\leq1/8$, and the condition on $T$ gives $T_0\leq T$.
    Moreover, since $K\geq4$, the quantity inside the floor is at least one, and hence
    \[ T_0\geq\frac{\log K}{16q(\Delta_{\min})}=\frac{\theta^{\frac{1}{\epsilon}}\log K}{16\Delta_{\min}^{1+\frac{1}{\epsilon}}}. \]
    Applying \cref{lem:heavy-tailed-prefix-lower-bound} with $\delta=\Delta_{\min}$ and $n=T_0$ yields an expert $i^\ast\in[K]$ such that
    \[ \mathcal{R}_{T_0}(i^\ast)\geq\frac{\Delta_{\min}T_0}{8}\geq\frac{\theta^{\frac{1}{\epsilon}}\log K}{128\Delta_{\min}^{\frac{1}{\epsilon}}}. \]
    Every summand of the fixed-comparator regret under $\mathsf{I}(\Delta_{\min},i^\ast)$ is nonnegative, so $\mathcal{R}_T(i^\ast)\geq\mathcal{R}_{T_0}(i^\ast)$.
    By the definition of $\mathsf{I}(\Delta_{\min},i^\ast)$, its losses are nonnegative, every suboptimal expert has expected gap $\Delta_{\min}$ and $(1+\epsilon)$-moment $\theta$, and $i^\ast$ is uniquely optimal because it always incurs zero loss.
\end{proof}

\subsection{Lower bound in a self-bounded environment}
\label{app:lower_bound_self_bounded}
We next extend the preceding construction to self-bounded instances by adapting the construction of \citet[Theorem~5]{NEURIPS2021_ito_robustness_adversarial_decision} to heavy-tailed losses.

We can obtain the following lower bound in a self-bounded environment.
This is the formal version of \cref{thm:lower_bound_self_bounded_simple}.
\begin{theorem}
    \label{thm:lower_bound_self_bounded}
    Fix $\epsilon\in(0,1]$ and $\theta>0$.
    Let $\Delta_{\min}\in(0,8^{-\frac{\epsilon}{1+\epsilon}}\theta^{\frac{1}{1+\epsilon}}]$, $K\geq4$, $T\geq\log K$, and $C\geq0$.
    Then, for any algorithm, there exist a unique optimal expert $i^\ast\in[K]$ and a distribution over nonnegative loss sequences satisfying \cref{asp:unknown-epsilon_wise,asp:self_bounded} with parameters $\epsilon$, $\theta$, $\Delta_{\min}$, and $C$ such that
    \[ \mathcal{R}_T\geq c\min\set*{\prn*{\frac{\theta}{\Delta_{\min}}}^{\frac{1}{\epsilon}}\log K+\prn*{\frac{\theta C}{\Delta_{\min}}}^{\frac{1}{1+\epsilon}}\prn*{\log K}^{\frac{\epsilon}{1+\epsilon}},\prn*{\theta T}^{\frac{1}{1+\epsilon}}\prn*{\log K}^{\frac{\epsilon}{1+\epsilon}}}, \]
    where $c>0$ is a universal constant.
\end{theorem}

\begin{proof}
    For each instance constructed below, let $\delta_t$ denote the common expected gap between the distinguished expert $i^\ast$ and every other expert on round $t$, so that $\E\brk{\ell_t(i)-\ell_t(i^\ast)}=\delta_t$ for every $i\neq i^\ast$.
    Define ${R}_0=\log K\prn*{\frac{\theta}{\Delta_{\min}}}^{\frac{1}{\epsilon}}$, ${R}_C=\prn*{\log K}^{\frac{\epsilon}{1+\epsilon}}\prn*{\frac{\theta C}{\Delta_{\min}}}^{\frac{1}{1+\epsilon}}$, and ${R}_{\mathrm{wc}}=\prn*{\theta T}^{\frac{1}{1+\epsilon}}\prn*{\log K}^{\frac{\epsilon}{1+\epsilon}}$.
    A direct calculation gives ${R}_C^{1+\epsilon}/{R}_0^{1+\epsilon}=C/{R}_0$.
    Also set $H={\log K}/{(8q(\Delta_{\min}))}$.
    The condition on $\Delta_{\min}$ ensures that $q(\Delta_{\min})\leq1/8$ and hence $H\geq\log K$.

    Following the four-case decomposition in \citet[Theorem~5]{NEURIPS2021_ito_robustness_adversarial_decision}, we consider the following four cases: (i) $T<H$, (ii) $C/\Delta_{\min}<H\leq T$, (iii) $H\leq C/\Delta_{\min}\leq T$, and (iv) $H\leq T<C/\Delta_{\min}$.\vspace{3pt}

    \noindent{(i): $T<H$.}
    Define $\delta_T=\theta^{\frac{1}{1+\epsilon}}\prn*{\frac{\log K}{8T}}^{\frac{\epsilon}{1+\epsilon}}$.
    Substituting the definition of $\delta_T$ into $q$ gives $q(\delta_T)=(\log K)/(8T)$.
    Hence, $q(\delta_T)\leq1/8$ follows from $T\geq\log K$, and $Tq(\delta_T)=(\log K)/8$.
    Moreover, $T<H=(\log K)/(8q(\Delta_{\min}))$ implies $q(\delta_T)>q(\Delta_{\min})$.
    Since $q$ is strictly increasing on $[0,\infty)$, it follows that $\delta_T>\Delta_{\min}$.
    Applying \cref{lem:heavy-tailed-prefix-lower-bound} to $\mathsf{I}(\delta_T,i^\ast)$ with $n=T$ gives an instance with $\mathcal{R}_T\geq\frac{\delta_TT}{8}=\Omega\prn{R_{\mathrm{wc}}}$.
    By the definition of $\mathsf{I}(\delta_T,i^\ast)$, its suboptimal experts have $(1+\epsilon)$-moment $\theta$, and the expert $i^\ast$ supplied by the lemma is uniquely optimal because it always incurs zero loss.
    Since its gap is at least $\Delta_{\min}$ at every round, this instance satisfies \cref{asp:self_bounded} with corruption budget zero, and hence also with budget $C$.\vspace{3pt}

    \noindent{(ii): $C/\Delta_{\min}<H\leq T$.}
    The instance $\mathsf{I}(\Delta_{\min},i^\ast)$ used in \cref{thm:lower_bound_gap_dependent} satisfies \cref{asp:self_bounded} with corruption budget zero and gives $\mathcal{R}_T=\Omega(R_0)$.
    Since $\Delta_{\min}H=R_0/8$, the condition $C/\Delta_{\min}<H$ implies $C<R_0/8$.
    The identity $R_C^{1+\epsilon}/R_0^{1+\epsilon}=C/R_0$ therefore implies $R_C<8^{-\frac{1}{1+\epsilon}}R_0\leq R_0$.
    Consequently, this construction gives $\mathcal{R}_T=\Omega(\min\set*{R_0+R_C,R_{\mathrm{wc}}})$.\vspace{3pt}

    \noindent{(iii): $H\leq C/\Delta_{\min}\leq T$.}
    Set $\tilde{T}=\floor*{\frac{C}{\Delta_{\min}}}$ and $\delta_C=\theta^{\frac{1}{1+\epsilon}}\prn*{\frac{\log K}{16\tilde{T}}}^{\frac{\epsilon}{1+\epsilon}}$.
    Since $C/\Delta_{\min}\geq H\geq\log K$, we have $\tilde{T}\geq C/(2\Delta_{\min})\geq(\log K)/2$, whereas $C/\Delta_{\min}\leq T$ gives $\tilde{T}\leq T$.

    Substituting the definition of $\delta_C$ into $q$ gives $q(\delta_C)=(\log K)/(16\tilde{T})\leq1/8$ and $\tilde{T}q(\delta_C)=(\log K)/16$.
    Use $\mathsf{I}(\delta_C,i^\ast)$ for the first $\tilde{T}$ rounds and $\mathsf{I}(\Delta_{\min},i^\ast)$ thereafter.
    The resulting losses satisfy \cref{asp:unknown-epsilon_wise} because, at every round, the $(1+\epsilon)$-moment of each suboptimal expert is exactly $\theta$, whereas the loss of the optimal expert is zero.

    Because the learner is non-anticipating, the joint law of the losses and the learner's actions over the first $\tilde{T}$ rounds coincides with that under the stationary instance $\mathsf{I}(\delta_C,i^\ast)$, so we can use \cref{lem:heavy-tailed-prefix-lower-bound}.
    On every round after $\tilde{T}$, expert $i^\ast$ incurs zero loss while all losses are nonnegative, so every remaining instantaneous regret term is nonnegative and $\mathcal{R}_T(i^\ast)\geq\mathcal{R}_{\tilde{T}}(i^\ast)$.
    Consequently, there is a choice of $i^\ast$ for which
    \begin{align*}
        \mathcal{R}_T(i^\ast)
        &\geq\mathcal{R}_{\tilde{T}}(i^\ast)\geq\frac{\delta_C\tilde{T}}{8}
        =\Omega\prn*{\theta^{\frac{1}{1+\epsilon}}\prn*{\log K}^{\frac{\epsilon}{1+\epsilon}}\tilde{T}^{\frac{1}{1+\epsilon}}}
        =\Omega(R_C),
    \end{align*}
    where the last equality uses $\tilde{T}\geq C/(2\Delta_{\min})$.
    For this instance, $\delta_t=\delta_C$ for $t\leq\tilde{T}$ and $\delta_t=\Delta_{\min}$ for $t>\tilde{T}$.
    Since the optimal expert always has zero loss and every other expert has expected loss $\delta_t$, we have
    \begin{align*}
        \mathcal{R}_T(i^\ast)
        &=\sum_{t=1}^T\delta_t\E\brk*{1-p_t(i^\ast)}
        \geq\Delta_{\min}\sum_{t=1}^T\E\brk*{1-p_t(i^\ast)}-\tilde{T}\prn*{\Delta_{\min}-\delta_C}_+\\
        &\geq\Delta_{\min}\sum_{t=1}^T\E\brk*{1-p_t(i^\ast)}-C,
    \end{align*}
    where the last inequality follows from $\tilde{T}\Delta_{\min}\leq C$.
    Thus, \cref{asp:self_bounded} holds with the above settings.
    Moreover, $i^\ast$ is uniquely optimal because it always incurs zero loss and every other expert has strictly positive expected loss at every round.
    The inequalities $C/\Delta_{\min}\geq H$ and $\Delta_{\min}H=R_0/8$ also imply $C/R_0\geq1/8$.
    Thus, the identity $R_C^{1+\epsilon}/R_0^{1+\epsilon}=C/R_0$ gives $R_0=O(R_C)$, which proves the claimed bound in this case.

    \noindent{(iv): $H\leq T<C/\Delta_{\min}$.}
    Define $\delta_T=\theta^{\frac{1}{1+\epsilon}}\prn*{\frac{\log K}{8T}}^{\frac{\epsilon}{1+\epsilon}}$.
    As in Case~(i), $q(\delta_T)=(\log K)/(8T)\leq1/8$ and $Tq(\delta_T)=(\log K)/8$.
    Applying \cref{lem:heavy-tailed-prefix-lower-bound} to $\mathsf{I}(\delta_T,i^\ast)$ with $n=T$ gives an instance satisfying \cref{asp:unknown-epsilon_wise} and $\mathcal{R}_T\geq\frac{\delta_TT}{8}=\Omega(R_{\mathrm{wc}})$.
    Its optimal expert always incurs zero loss and every other expert has expected loss $\delta_T>0$, so the optimal expert is unique.
    Furthermore, it holds that
    \begin{align*}
        \mathcal{R}_T(i^\ast)
        &=\delta_T\sum_{t=1}^T\E\brk*{1-p_t(i^\ast)}
        \geq\Delta_{\min}\sum_{t=1}^T\E\brk*{1-p_t(i^\ast)}-T\prn*{\Delta_{\min}-\delta_T}_+\\
        &\geq\Delta_{\min}\sum_{t=1}^T\E\brk*{1-p_t(i^\ast)}-C,
    \end{align*}
    where the last inequality follows from $T\Delta_{\min}<C$.
    Thus, \cref{asp:self_bounded} holds, and the worst-case term proves the claimed lower bound.

    Combining Cases~(i)--(iv) completes the proof.
\end{proof}

\section{Details of Algorithms}
\label{app:analysis_of_algorithm}
This section provides the detailed analyses of the algorithms in \cref{sec:algorithms}.
We first establish moment, truncation, and scale bounds shared by the analyses of all algorithms.
Afterward, we provide the specific analyses of \cref{alg:adaptive-multiscale-ftrl-componentwise_simple,alg:parameter-free-heavy-tailed-omd,alg:known-parameter-shannon-ftrl}.

\subsection{Common Analysis}
\label{app:subsec:common_analysis}
The analysis in this section has three components.
We first transfer the moment assumption on the losses to the instantaneous regrets, then bound the truncation cost, including the residual introduced by rescaling for FTRL, and finally control the resulting adaptive scales.

Throughout this section, fix the value of $\epsilon$ in \cref{asp:unknown-epsilon_wise}, and define $w_t(i)=\abs{r_t(i)}^{1+\epsilon}$ and $\bar{w}_t=\sum_{j=1}^Kp_t(j)w_t(j)$.
We use $\mathrm{A}=\mathrm{F}$ for Adaptive Heavy-Tailed FTRL and $\mathrm{A}=\mathrm{O}$ for Adaptive Heavy-Tailed OMD.
The scales in each statement are those generated by the corresponding algorithm.
For $\mathrm{A}=\mathrm{F}$, set $\Lambda_{\mathrm{A}}=\Lambda_{\mathrm{F}}$.
For $\mathrm{A}=\mathrm{O}$, apply the truncation and scale bounds below to a prefix of a fixed epoch $m$ and set $\Lambda_{\mathrm{A}}=\Lambda_m$.
For OMD, round indices in this subsection are understood to be local to the epoch prefix, and the epoch superscript $(m)$ is omitted from the scales.
\begin{lemma}
    \label{lem:known-parameter-centered-moments}
    Under \cref{asp:unknown-epsilon_wise}, for every $t\in[T]$ and $i\in[K]$, it holds that
    \[
        \E_t\brk*{w_t(i)}\leq2^{1+\epsilon}\theta\quad\text{and}\quad \E_t\brk*{\bar{w}_t}\leq2^{1+\epsilon}\theta.
    \]
\end{lemma}

\begin{proof}
    Since $p_t$ is $\mathcal{F}_{t-1}$-measurable, conditional Jensen's inequality and \cref{asp:unknown-epsilon_wise} give
    \[
        \E_t\brk*{\abs{\inpr{p_t,\ell_t}}^{1+\epsilon}}\leq\sum_{j=1}^Kp_t(j)\E_t\brk*{\abs{\ell_t(j)}^{1+\epsilon}}\leq\theta.
    \]
    The inequality $\abs{x-y}^{1+\epsilon}\leq2^\epsilon(\abs{x}^{1+\epsilon}+\abs{y}^{1+\epsilon})$ therefore yields, for every $i\in[K]$,
    \begin{equation*}
        \E_t\brk*{w_t(i)}
        \leq2^\epsilon\E_t\brk{\abs{\ell_t(i)}^{1+\epsilon}+\abs*{\inpr{p_t,\ell_t}}^{1+\epsilon}}
        \leq2^{1+\epsilon}\theta.
    \end{equation*}
    Using again the $\mathcal{F}_{t-1}$-measurability of $p_t$, we obtain
    \[
        \E_t\brk*{\bar{w}_t}=\sum_{j=1}^Kp_t(j)\E_t\brk*{w_t(j)}\leq2^{1+\epsilon}\theta.
    \]
\end{proof}

We next bound the truncation cost, which for FTRL also includes the residual introduced by rescaling.
\begin{lemma}
    \label{lem:multiscale-truncation-error}
    For $\mathrm{A}\in\set{\mathrm{F},\mathrm{O}}$, consider the quantities generated by \cref{alg:adaptive-multiscale-ftrl-componentwise_simple} when $\mathrm{A}=\mathrm{F}$ and by \cref{alg:parameter-free-heavy-tailed-omd} when $\mathrm{A}=\mathrm{O}$.
    Then, for every loss sequence and every $i\in[K]$, it holds that
    \[
        \trunc(p_{1:T})+\trunc(-e_i)\leq
        \begin{cases}
            \Lambda_{\mathrm{F}}\prn*{G_T+B_T(i)}+3G_T+B_T(i), & \mathrm{A}=\mathrm{F},\\
            \Lambda_m\prn*{G_T+B_T(i)}, & \mathrm{A}=\mathrm{O}.
        \end{cases}
    \]
\end{lemma}
\begin{proof}
    Recall that $\trunc(p_{1:T})=\sum_{t=1}^T\inpr{p_t,\bar{\ell}_t-\tilde{\ell}_t}$, and $\trunc(-e_i)=\sum_{t=1}^T\inpr{-e_i,\bar{\ell}_t-\tilde{\ell}_t}$.
    We have
    \begin{equation}
        \label{eq:multiscale-truncation-decomposition}
        \trunc(p_{1:T})+\trunc(-e_i)
        \leq\sum_{t=1}^T\sum_{j=1}^K p_t(j)\abs{\bar{\ell}_t(j)-\tilde{\ell}_t(j)}+\sum_{t=1}^T\abs{\bar{\ell}_t(i)-\tilde{\ell}_t(i)}.
    \end{equation}
    To upper bound this, we evaluate $\abs{\bar{\ell}_t(j)-\tilde{\ell}_t(j)}$.
    For every $t\in[T]$ and $j\in[K]$, let $E_t(j)=\set*{\abs{r_t(j)}>M_t(j)}$.
    Since $M_t(j)\geq M_{t-1}(j)$, the definition of $E_t(j)$, \hyperref[p:scale-monotonicity]{(P1)}, and \hyperref[p:regime-lower-bounds]{(P3)} give
    \[
        \abs{r_t(j)}\ind\brk*{E_t(j)}\leq\phi_{M_{t}(j)}\prn{r_t(j)}\leq\phi_{M_{t-1}(j)}\prn{r_t(j)}.
    \]
    Using this inequality and the definitions of $g_t$ and $d_t(i)$, we obtain
    \begin{equation}
        \label{eq:multiscale-clipping-residual}
        \sum_{t=1}^T \sum_{j=1}^K p_t(j)\abs{r_t(j)}\ind\brk*{E_t(j)}+\sum_{t=1}^T \abs{r_t(i)}\ind\brk*{E_t(i)}
        \leq\sum_{t=1}^T \prn*{g_t + d_t(i)}
        \leq \Lambda_{\mathrm{A}}\prn*{G_T + B_T(i)}.
    \end{equation}
    $\abs{\bar{\ell}_t(j)-\tilde{\ell}_t(j)}=\abs{r_t(j)}\ind\brk*{E_t(j)}$ in \cref{alg:parameter-free-heavy-tailed-omd}, so \eqref{eq:multiscale-truncation-decomposition} and \eqref{eq:multiscale-clipping-residual} prove the claim for $\mathrm{A}=\mathrm{O}$.

    It remains to prove the claim for $\mathrm{A}=\mathrm{F}$.
    Since $M_t(j)$ is nondecreasing in $t$, we have $\rho_t(j)\in[0,1]$.
    The definition of $\tilde{\ell}_t(j)$ in \cref{alg:adaptive-multiscale-ftrl-componentwise_simple} gives $\abs{\bar{\ell}_t(j)-\tilde{\ell}_t(j)}\leq\abs{r_t(j)}\prn*{1-\rho_t(j)}\ind\brk*{E_t(j)^{\mathsf c}}+\abs{r_t(j)}\ind\brk*{E_t(j)}$.
    Thus, combining this with \eqref{eq:multiscale-truncation-decomposition} and \eqref{eq:multiscale-clipping-residual} gives
    \begin{align}
        \label{eq:multiscale-residual-decomposition}
        &\trunc(p_{1:T})+\trunc(-e_i)\nonumber\\
        &\leq \sum_{t=1}^T \sum_{j=1}^K p_t(j)\abs{r_t(j)}\prn*{1-\rho_t(j)}\ind\brk*{E_t(j)^{\mathsf c}}+\sum_{t=1}^T \abs{r_t(i)}\prn*{1-\rho_t(i)}\ind\brk*{E_t(i)^{\mathsf c}}+ \Lambda_{\mathrm{F}}\prn*{G_T + B_T(i)}.
    \end{align}

    If $M_t(j)=0$, then $E_t(j)^c$ implies $r_t(j)=0$, so the claim is trivial. 
    We consider the case where $M_t(j)>0$. On $E_t(j)^{\mathsf c}$, the definitions of $\rho_t(j)$ and $E_t(j)$ give
    \begin{equation}
        \label{eq:multiscale-scale-increment-trivial}
        \abs{r_t(j)}\prn{1-\rho_t(j)}
        \leq
        M_t(j)\prn*{1-\frac{M_{t-1}(j)}{M_t(j)}}
        =
        M_t(j)-M_{t-1}(j).
    \end{equation}
    Thus, on $E_t(j)^{\mathsf c}$, we have
    \begin{align}
        \label{eq:multiscale-scale-increment}
        &\abs{r_t(j)}\prn{1-\rho_t(j)}
        \leq M_t(j)-M_{t-1}(j)\nonumber\\
        &\leq
        \max\set*{
            G_{t-1}+\frac{g_t}{\Lambda_{\mathrm{F}}},
            B_{t-1}(j)+\frac{d_t(j)}{\Lambda_{\mathrm{F}}}
        }
        -
        \max\set*{G_{t-1},B_{t-1}(j)}
        \leq
        \frac{1}{\Lambda_{\mathrm{F}}}\prn*{g_t+d_t(j)},
    \end{align}
    where we used the fact that
    $\max\set{a+x,b+y}-\max\set{a,b}\leq x+y$
    for all $a,b\in\R$ and $x,y\geq0$.

    Since $g_t=\sum_{j=1}^Kp_t(j)d_t(j)$, multiplying \eqref{eq:multiscale-scale-increment} by $p_t(j)$ and summing over $j$ and $t$ gives
    \begin{equation}
        \label{eq:multiscale-learner-scale-increments}
        \sum_{t=1}^T \sum_{j=1}^K p_t(j)\abs{r_t(j)}\prn{1-\rho_t(j)}\ind\brk{E_t(j)^{\mathsf c}}\leq 2G_T.
    \end{equation}
    For the comparator coordinate, \eqref{eq:multiscale-scale-increment-trivial} and the definition of $M_t(i)$ give
    \begin{equation}
        \label{eq:multiscale-comparator-scale-increments}
        \sum_{t=1}^T \abs{r_t(i)}\prn{1-\rho_t(i)}\ind\brk{E_t(i)^{\mathsf c}}\leq \sum_{t=1}^T \prn{M_t(i)-M_{t-1}(i)} \leq G_T+B_T(i).
    \end{equation}
    Combining \eqref{eq:multiscale-residual-decomposition} with \eqref{eq:multiscale-learner-scale-increments} and \eqref{eq:multiscale-comparator-scale-increments} proves the claim for $\mathrm{A}=\mathrm{F}$.
\end{proof}

We next control the adaptive scales $G_t$ and $B_t(i)$.
\begin{lemma}
    \label{lem:multiscale-scale-recursion}
    Fix $\epsilon\in(0,1]$ and $\mathrm{A}\in\{\mathrm{F},\mathrm{O}\}$, and consider the quantities generated by Algorithm~1 when $\mathrm{A}=\mathrm{F}$ and by Algorithm~2 when $\mathrm{A}=\mathrm{O}$.
    Then, for every $i\in\brk{K}$, it holds that
    \begin{equation*}
        G_T^{1+\epsilon} \leq \frac{1+\epsilon}{\Lambda_{\mathrm{A}}}\prn*{1+\frac{1}{\Lambda_{\mathrm{A}}}}^{\epsilon} \sum_{t=1}^T\bar{w}_t, \quad \text{and} \quad
        B_T(i)^{1+\epsilon} \leq \frac{1+\epsilon}{\Lambda_{\mathrm{A}}}\prn*{1+\frac{1}{\Lambda_{\mathrm{A}}}}^{\epsilon} \sum_{t=1}^T w_t(i).
    \end{equation*}
\end{lemma}

\begin{proof}
    We first record a scalar argument.
    Let $\alpha>0$, and let $\set*{S_t}_{t=0}^T$, $\set*{h_t}_{t=1}^T$, and $\set*{V_t}_{t=1}^T$
    be nonnegative sequences such that $S_0=0$ and $S_t=S_{t-1}+\alpha h_t, $
    with $h_t\leq V_t^{\frac{1}{1+\epsilon}}$ for every $t$, and, whenever $S_{t-1}>0$,
    $h_t\leq S_{t-1}^{-\epsilon}V_t$.
    When $S_{t-1}>0$, set $z=\frac{V_t^{\frac{1}{1+\epsilon}}}{S_{t-1}}$ and $y=\frac{\alpha h_t}{S_{t-1}}$.
    Then, we have $y\geq0$ and $y \leq \alpha\min\set{z,z^{1+\epsilon}}$, where the first bound on $h_t$ gives $y\leq\alpha z$, and the second gives $y\leq\alpha z^{1+\epsilon}$.
    By the convexity of $x\mapsto x^{1+\epsilon}$, for all $y\geq0$, it holds that
    \[
        (1+y)^{1+\epsilon}-1 \leq (1+\epsilon)y(1+y)^\epsilon.
    \]
    Hence, this inequality and the definition of $y$ give
    \[
        S_t^{1+\epsilon}-S_{t-1}^{1+\epsilon}=S_{t-1}^{1+\epsilon} \prn*{(1+y)^{1+\epsilon}-1}\leq (1+\epsilon)y(1+y)^\epsilon S_{t-1}^{1+\epsilon}.
    \]
    If $z\leq1$, then $y\leq\alpha z^{1+\epsilon}$ and $1+y\leq1+\alpha$, which imply
    \[
        S_t^{1+\epsilon}-S_{t-1}^{1+\epsilon}
        \leq(1+\epsilon)\alpha(1+\alpha)^\epsilon S_{t-1}^{1+\epsilon}z^{1+\epsilon}
        =(1+\epsilon)\alpha(1+\alpha)^\epsilon V_t.
    \]

    If $z>1$, then $y\leq\alpha z$ and $1+y\leq1+\alpha z\leq(1+\alpha)z$, which imply
    \[
        S_t^{1+\epsilon}-S_{t-1}^{1+\epsilon}
        \leq(1+\epsilon)\alpha(1+\alpha)^\epsilon S_{t-1}^{1+\epsilon}z^{1+\epsilon}
        =(1+\epsilon)\alpha(1+\alpha)^\epsilon V_t.
    \]
    The same inequality holds when $S_{t-1}=0$, since then $S_t^{1+\epsilon}\leq\alpha^{1+\epsilon}V_t$
    and $\alpha^\epsilon\leq(1+\epsilon)(1+\alpha)^\epsilon$.
    Summing over $t$ gives
    \begin{equation}
        \label{eq:multiscale-scalar-recursion}
        S_T^{1+\epsilon} \leq (1+\epsilon)(1+\alpha)^\epsilon\alpha \sum_{t=1}^T V_t .
    \end{equation}

    For $g_t$, the inequalities $\phi_s(x)\leq\abs{x}$ from \hyperref[p:scale-monotonicity]{(P1)} and $\phi_s(x)\leq s^{-\epsilon}\abs{x}^{1+\epsilon}$ from \hyperref[p:moment-bound]{(P2)} for $s>0$, together with Jensen's inequality, imply $g_t\leq\bar{w}_t^{\frac{1}{1+\epsilon}}$ and $g_t\leq\sum_{i=1}^Kp_t(i)M_{t-1}(i)^{-\epsilon}w_t(i)\leq G_{t-1}^{-\epsilon}\bar{w}_t$ if $G_{t-1}>0$.
    Applying \eqref{eq:multiscale-scalar-recursion} with
    \[
        S_t=G_t,\quad h_t=g_t,\quad V_t=\bar{w}_t,\quad\text{and}\quad \alpha=\frac{1}{\Lambda_{\mathrm{A}}}
    \]
    gives
    \begin{equation*}
        G_T^{1+\epsilon} \leq \frac{1+\epsilon}{\Lambda_{\mathrm{A}}}\prn*{1+\frac{1}{\Lambda_{\mathrm{A}}}}^\epsilon\sum_{t=1}^T \bar{w}_t .
    \end{equation*}

    For each fixed $i\in\brk{K}$, the same scalar inequalities give $d_t(i)\leq w_t(i)^{\frac{1}{1+\epsilon}}$ and $d_t(i)\leq M_{t-1}(i)^{-\epsilon}w_t(i)\leq B_{t-1}(i)^{-\epsilon}w_t(i)$ if $B_{t-1}(i)>0$.
    Applying \eqref{eq:multiscale-scalar-recursion} with
    \[
        S_t=B_t(i),\quad h_t=d_t(i),\quad V_t=w_t(i),\quad\text{and}\quad \alpha=\frac{1}{\Lambda_{\mathrm{A}}}
    \]
    gives
    \begin{equation*}
        B_T(i)^{1+\epsilon} \leq \frac{1+\epsilon}{\Lambda_{\mathrm{A}}}\prn*{1+\frac{1}{\Lambda_{\mathrm{A}}}}^\epsilon\sum_{t=1}^T w_t(i) .
    \end{equation*}
\end{proof}

\begin{remark}[Tightness of the robust increment $\phi_s$]
    \label{rem:robust-scale-increment-tightness}
    The choice of $\phi_s$ in \cref{subsec:robust-scale-increment} is pointwise tight with respect to \hyperref[p:scale-monotonicity]{(P1)}--\hyperref[p:regime-lower-bounds]{(P3)}.
    Indeed, for $s>0$, \hyperref[p:scale-monotonicity]{(P1)} gives $\phi_s(x)\leq\abs{x}$, while \hyperref[p:moment-bound]{(P2)} with $\epsilon=1$ gives $\phi_s(x)\leq x^2/s$.
    Hence, any family satisfying \hyperref[p:scale-monotonicity]{(P1)} and \hyperref[p:moment-bound]{(P2)} must obey $\phi_s(x)\leq\min\set*{\frac{x^2}{s},\abs{x}}$.
    Conversely, when $\abs{x}\leq s$, \hyperref[p:regime-lower-bounds]{(P3)} implies $\phi_s(x)\geq x^2/s$, whereas when $\abs{x}>s$, it implies $\phi_s(x)\geq\abs{x}$.
    These two lower bounds imply that $\phi_s(x)\geq\min\set*{\frac{x^2}{s},\abs{x}}$ in both regimes.
    Thus, for $s>0$, properties \hyperref[p:scale-monotonicity]{(P1)}--\hyperref[p:regime-lower-bounds]{(P3)} uniquely determine $\phi_s(x)=\min\set*{\frac{x^2}{s},\abs{x}}$, while the case $s=0$ is determined by \hyperref[p:scale-monotonicity]{(P1)} and \hyperref[p:regime-lower-bounds]{(P3)} as $\phi_0(x)=\abs{x}$.
    More generally, suppose that \hyperref[p:regime-lower-bounds]{(P3)} is replaced by
    \[
        \abs{x}\leq c\phi_s(x)\quad\text{if $\abs{x}>s$}\quad\text{and}\quad x^2\leq cs\phi_s(x)\quad\text{if $\abs{x}\leq s$}
    \]
    for some uniform constant $c>0$.
    Compatibility with \hyperref[p:scale-monotonicity]{(P1)} and \hyperref[p:moment-bound]{(P2)} requires $c\geq1$, and the above choice of $\phi_s$ attains $c=1$.
    This pointwise tightness concerns only the constants in \hyperref[p:scale-monotonicity]{(P1)}--\hyperref[p:regime-lower-bounds]{(P3)} and does not imply optimality of the resulting regret constants.
\end{remark}

\subsection{Details of \cref{alg:adaptive-multiscale-ftrl-componentwise_simple}}
\label{app:subsec:ftrl}
This section provides the analyses related to \cref{alg:adaptive-multiscale-ftrl-componentwise_simple}.
\subsubsection{Lemma Used in the Proof of \cref{thm:upper-bound-multiscale-ftrl-componentwise_simple}}
Here, we bound the surrogate regret $\tilde{\mathcal{R}}_T(e_i)$.
Combining this bound with the results in \cref{app:subsec:common_analysis} yields the regret guarantee stated in \cref{thm:upper-bound-multiscale-ftrl-componentwise_simple}.
\begin{lemma}
    \label{lem:ftrl-term-componentwise_simple}
    For every loss sequence and every $i\in[K]$, \cref{alg:adaptive-multiscale-ftrl-componentwise_simple} satisfies
    \[
        \tilde{\mathcal{R}}_T(e_i)\leq \Lambda_{\mathrm{F}}\prn*{1+\frac{e}{2}}G_T + B_T(i)\log K +\frac{1}{K}\sum_{j=1}^K B_T(j).
    \]
\end{lemma}
\begin{proof}
    Rounds $t\in[T]$ such that $M_s(j)=0$ for all $j\in[K]$ and $s\leq t$ can be ignored.
    Indeed, whenever $\max_{j\in[K]} M_{t-1}(j)=0$, we have $\tilde{\ell}_t=0$. 
    If $M_T(j)=0$ for all $j\in[K]$, then $\tilde{\mathcal{R}}_T(e_i)=0$ and the claim is trivial. 
    After discarding these initial rounds and relabeling the remaining rounds, we may assume that $M_1(j)>0$ for every $j\in[K]$.

    Set $\varphi_t(p)=D_{t-1}(p;\pi)$.
    For $t\geq 2$, by the definition of $\rho_t(j)$ and the truncation rule, it holds that $\frac{|\tilde{\ell}_t(j)|}{M_{t-1}(j)}=\frac{|r_t(j)|}{M_t(j)}\ind[|r_t(j)|\leq M_t(j)]\leq 1.$
    Since $\varphi_t(p)=D_{t-1}(p;\pi)$ is nonnegative and  twice-differentiable on $\R_{>0}^K$, and satisfies $\varphi_{t+1}(p)\geq\varphi_t(p)$ for every $p\in\R_{\geq0}^K$, applying the time-varying FTRL argument of \citet[proof of Theorem C.1, Eq. (10)]{NEURIPS2025_Moulin_Expert_HTEXP} to rounds $t=2,\dots,T$ gives
    \begin{equation}
        \label{eq:ftrl-penalty-stability-componentwise_simple}
        \tilde{\mathcal{R}}_T(e_i)\leq D_T(e_i;\pi)+\frac{1}{2}\sum_{t=2}^T \nrm{\tilde{\ell}_t}^2_{(\nabla^2\varphi_t(z_t))^{-1}},
    \end{equation}
    where $z_t(j)$ lies between $p_t(j)$ and $\hat{p}_{t+1}(j)=p_t(j)\exp\prn*{-\frac{\tilde{\ell}_t(j)}{M_{t-1}(j)}}$.

    Using $\frac{|\tilde{\ell}_t(j)|}{M_{t-1}(j)}\leq 1$, we have $z_t(j)\leq e\,p_t(j)$ for every $j\in[K]$.
    Consequently, the second term in \eqref{eq:ftrl-penalty-stability-componentwise_simple} is at most
    \[
        \frac{1}{2}\sum_{t=2}^T \nrm{\tilde{\ell}_t}_{(\nabla^2\varphi_t(z_t))^{-1}}^2
        =\frac{1}{2}\sum_{t=2}^T \sum_{j=1}^K \frac{z_t(j)\tilde{\ell}_t(j)^2}{M_{t-1}(j)}
        \leq \frac{e}{2}\sum_{t=2}^T \sum_{j=1}^K \frac{p_t(j)\tilde{\ell}_t(j)^2}{M_{t-1}(j)},
    \]
    where the first equality follows from $\nabla^2\varphi_t(z_t)=\operatorname{Diag}\prn*{\frac{M_{t-1}(j)}{z_t(j)}}_{j=1}^K$.

    Since $M_{t-1}(j)\leq M_t(j)$, it follows that
    \begin{align*}
        \frac{\tilde{\ell}_t(j)^2}{M_{t-1}(j)}
        &=\frac{r_t(j)^2M_{t-1}(j)}{M_t(j)^2}\ind\brk*{\abs{r_t(j)}\leq M_t(j)}
        \leq \frac{r_t(j)^2}{M_t(j)}\ind\brk*{\abs{r_t(j)}\leq M_t(j)}
        \\
        &\leq\phi_{M_t(j)}\prn{r_t(j)}
        \leq\phi_{M_{t-1}(j)}\prn{r_t(j)}
        .
    \end{align*}
    The second inequality follows from \hyperref[p:regime-lower-bounds]{(P3)}, and the last inequality follows from \hyperref[p:scale-monotonicity]{(P1)}.
    Thus, it holds that
    \begin{equation}
        \label{eq:ftrl-stability-bound-componentwise_simple}
        \frac{1}{2}\sum_{t=2}^T \nrm{\tilde{\ell}_t}_{(\nabla^2\varphi_t(z_t))^{-1}}^2\leq \frac{e}{2}\Lambda_{\mathrm{F}}G_T.
    \end{equation}

    It remains to bound the first term in \eqref{eq:ftrl-penalty-stability-componentwise_simple}.
    Since $\pi(j)=1/K$, direct evaluation gives
    \begin{align*}
        D_T(e_i;\pi)
        &=\prn*{\log K-1}M_T(i)+\frac{1}{K}\sum_{j=1}^K M_T(j)
        \leq M_T(i)\log K+\frac{1}{K}\sum_{j=1}^K M_T(j)
        \\
        &\leq \Lambda_{\mathrm{F}}G_T+B_T(i)\log K+\frac{1}{K}\sum_{j=1}^K B_T(j),
    \end{align*}
    where the last inequality uses $M_T(j)=\max\set*{G_T,B_T(j)}\leq G_T+B_T(j)$.
    Combining this inequality with \eqref{eq:ftrl-penalty-stability-componentwise_simple} and \eqref{eq:ftrl-stability-bound-componentwise_simple} proves the claim.
\end{proof}

\begin{remark}[Why we need the truncation]
    For simplicity, we assume that $M_t(i)>0$ for all $t\in[T]$ and $i\in[K]$ in this remark.
    Suppose that the indicator in the surrogate loss is removed and we use the rescaled loss $\check{\ell}_t(j)=-r_t(j)\frac{M_{t-1}(j)}{M_t(j)}$.
    Then the unconstrained one-step FTRL update would satisfy $\hat{p}_{t+1}(j)=p_t(j)\exp\prn*{\frac{r_t(j)}{M_t(j)}}$.
    The proof above uses the truncation to obtain $\abs{\tilde{\ell}_t(j)}/M_{t-1}(j)\leq1$, and hence $z_t(j)\leq\max\crl{p_t(j),\hat{p}_{t+1}(j)}\leq e p_t(j)$.
    Without truncation, the normalized increment is $\abs{r_t(j)}/M_t(j)$, which is not bounded by a universal constant under the present scale recursion.
    Thus the standard FTRL stability argument would introduce a factor depending on $\exp\prn*{\abs{r_t(j)}/M_t(j)}$ instead of the absolute constant $e$.
\end{remark}

\subsubsection{A horizon-uniform self-bounded counterexample for Adaptive Heavy-Tailed FTRL}
\label{app:subsubsec:adaptive-ftrl-self-bounded-counterexample}
We next show that Adaptive Heavy-Tailed FTRL does not in general achieve the optimal gap-dependent rate.
For $\epsilon\in(0,1)$, the algorithm fails even up to logarithmic factors.

\begin{proposition}
    \label{prop:adaptive-ftrl-self-bounded-counterexample}
    For every $\epsilon\in(0,1]$, there exists a loss process satisfying \cref{asp:unknown-epsilon_wise,asp:self_bounded} with $\theta=1$, $\Delta_{\min}=1/64$, and $C=0$ such that, for every integer $T\geq2$, \cref{alg:adaptive-multiscale-ftrl-componentwise_simple} satisfies
    \begin{equation}
        \label{eq:adaptive-ftrl-fixed-process-lower-bound}
        \mathcal{R}_T\geq\frac{\epsilon}{4096}\sum_{\tau=1}^{\floor{T/2}}\tau^{-\epsilon}.
    \end{equation}
    In particular, the regret is $\Omega(T^{1-\epsilon})$ when $\epsilon\in(0,1)$ and $\Omega(\log T)$ when $\epsilon=1$.
\end{proposition}

For every $\epsilon\in(0,1)$, the failure of Adaptive Heavy-Tailed FTRL to attain the optimal gap-dependent rate under \cref{asp:self_bounded} is intrinsic to the algorithm rather than an artifact of our upper-bound analysis.
For $\epsilon=1$, Adaptive Heavy-Tailed FTRL cannot match the gap-dependent lower bound in \cref{thm:lower_bound_self_bounded_simple}, which has no explicit dependence on $T$.
However, it remains open whether the $\Omega(\log T)$ dependence is tight for Adaptive Heavy-Tailed FTRL.

\begin{proof}
    Fix $\epsilon\in(0,1]$, and set $K=2$ and $\Delta=1/64$.
    Let $(Z_t,\sigma_t)_{t\geq1}$ be independent across rounds, where we take $Z_t\sim\operatorname{Bernoulli}(q_t)$, $q_t=\frac{\epsilon}{8t^{1+\epsilon}}$, and $\P(\sigma_t=1)=\P(\sigma_t=-1)=\frac{1}{2}$, and let $Z_t$ and $\sigma_t$ be independent.
    Define
    \begin{equation}
        \label{eq:adaptive-ftrl-fixed-process-instance}
        \ell_t(1)=Z_t\sigma_t t\quad\text{and}\quad\ell_t(2)=\Delta-Z_t\sigma_t t.
    \end{equation}
    We call an occurrence of $Z_t=1$ a spike.

    We first verify \cref{asp:unknown-epsilon_wise,asp:self_bounded}.
    The first coordinate satisfies
    \[
        \E_t\brk{\abs{\ell_t(1)}^{1+\epsilon}}=q_tt^{1+\epsilon}=\frac{\epsilon}{8}\leq\frac{1}{8}.
    \]
    Since $t\geq1>\Delta$, the second coordinate satisfies
    \begin{equation*}
        \E_t\brk{\abs{\ell_t(2)}^{1+\epsilon}}
        \leq\Delta^{1+\epsilon}+q_t(t+\Delta)^{1+\epsilon}
        \leq\frac{1}{64}+\frac{\epsilon 2^{1+\epsilon}}{8}<1.
    \end{equation*}
    Thus, \cref{asp:unknown-epsilon_wise} holds with $\theta=1$.
    The symmetry of $\sigma_t$ gives $\E_t[\ell_t(1)]=0$ and $\E_t[\ell_t(2)]=\Delta$.
    Hence, expert $1$ is uniquely optimal and it holds that
    \begin{equation}
        \label{eq:adaptive-ftrl-fixed-process-self-bounded-identity}
        \mathcal{R}_T=\mathcal{R}_T(1)=\Delta\sum_{t=1}^T\E[p_t(2)].
    \end{equation}
    Therefore, \cref{asp:self_bounded} holds with $i^\ast=1$, $\Delta_{\min}=\Delta$, and $C=0$.

    We next consider the lower bound for the algorithm.
    Let $\Lambda_{\mathrm{F}}=1+\log2$ and $H_t=\tilde{L}_t(2)-\tilde{L}_t(1)$ for $t\in[T]$.
    On every round with $Z_t=0$, we have $r_t(1)=p_t(2)\Delta$ and $r_t(2)=-p_t(1)\Delta$.
    For $t\geq2$, let $f_{t-1}$ denote the FTRL objective $\inpr{\tilde{L}_{t-1},p}+D_{t-1}(p;\pi)$ viewed as a function of $p(2)$.
    Since $d_1(i)>0$ for both $i\in\set{1,2}$, we have $M_{t-1}(i)>0$ for every $t\geq2$, which implies that $f_{t-1}$ is strictly convex.
    In addition, its derivative is
    \begin{equation}
        \label{eq:adaptive-ftrl-fixed-process-derivative}
        f_{t-1}^\prime(p(2))=H_{t-1}+M_{t-1}(2)\log(2p(2))-M_{t-1}(1)\log(2(1-p(2))).
    \end{equation}

    Suppose that the first spike occurs at round $\tau$.
    Then, for $t\in\set{0,\dots,\tau-1}$ and $i\in\set{1,2}$, we have
    \begin{equation}
        \label{eq:adaptive-ftrl-fixed-process-pre-spike}
        0\leq H_t\leq t\Delta\quad\text{and}\quad M_t(i)\leq\frac{t\Delta}{\Lambda_{\mathrm{F}}},
    \end{equation}
    where we used $\tilde{\ell}_s(1)\leq0\leq\tilde{\ell}_s(2)$, $\abs{\tilde{\ell}_s(i)}\leq\abs{r_s(i)}$, $d_s(i)\leq\abs{r_s(i)}\leq\Delta$, and $g_s\leq\Delta$ for every $s\leq t$.
    For every $2\leq t\leq\tau$, \eqref{eq:adaptive-ftrl-fixed-process-pre-spike} gives $f_{t-1}^\prime(1/2)=H_{t-1}\geq0$, so strict convexity places the minimizer at or below $1/2$.
    Together with $p_1(2)=1/2$, this shows that $p_t(2)\leq1/2$ for every $t\leq\tau$.
    Then, $p_\tau(2)\leq1/2$ gives $p_\tau(1)\geq1/2$, and hence
    \[
        \abs{r_\tau(2)}=p_\tau(1)\abs{\ell_\tau(1)-\ell_\tau(2)}\geq\tau-\frac{\Delta}{2}\geq\frac{\tau}{2}.
    \]
    Combining this inequality with \eqref{eq:adaptive-ftrl-fixed-process-pre-spike} gives $M_{\tau-1}(2)\leq \tau\Delta/\Lambda_{\mathrm{F}}<\tau/2\leq \abs{r_\tau(2)}$.
    Thus, \hyperref[p:regime-lower-bounds]{(P3)} and the definition of $B_\tau(2)$ give
    \begin{equation}
        \label{eq:adaptive-ftrl-fixed-process-large-scale}
        M_\tau(2)\geq B_\tau(2)\geq\frac{{\phi_{M_{\tau-1}(2)}(r_\tau(2))}}{\Lambda_{\mathrm{F}}}\geq\frac{\abs{r_\tau(2)}}{\Lambda_{\mathrm{F}}}\geq\frac{\tau}{2\Lambda_{\mathrm{F}}}\geq\frac{\tau}{4\Lambda_{\mathrm{F}}}.
    \end{equation}
    Moreover, since the truncation and rescaling rules give $\abs{\tilde{\ell}_\tau(i)}\leq M_{\tau-1}(i)$ for both coordinates, \eqref{eq:adaptive-ftrl-fixed-process-pre-spike} yields 
    \begin{equation}
        \label{eq:adaptive-ftrl-fixed-process-pre-spike-2}
        H_\tau\leq H_{\tau-1}+M_{\tau-1}(1)+M_{\tau-1}(2)\leq\tau\Delta+\frac{2\tau\Delta}{\Lambda_{\mathrm{F}}}.
    \end{equation}

    If no further spike occurs through round $2\tau$, then it holds that $H_s-H_{s-1}\leq\Delta$ for every $s\in\set{\tau+1,\ldots,2\tau}$, while $M_s(2)$ remains nondecreasing.
    Hence, every $t\in\set{\tau+1,\ldots,2\tau}$ satisfies
    \begin{equation}
        \label{eq:adaptive-ftrl-fixed-process-post-spike}
        H_{t-1}
        \leq2\tau\Delta\prn*{1+\frac{1}{\Lambda_{\mathrm{F}}}}
        <\frac{\tau\log2}{4\Lambda_{\mathrm{F}}}
        \leq M_{t-1}(2)\log2,
    \end{equation}
    where the first inequality follows from \eqref{eq:adaptive-ftrl-fixed-process-pre-spike-2}, the second inequality follows from $8\Delta(\Lambda_{\mathrm{F}}+1)<\log2$, which is equivalent to $2<7\log2$ when $\Delta=1/64$, and the last inequality follows from \eqref{eq:adaptive-ftrl-fixed-process-large-scale}.
    Thus, $f_{t-1}^\prime(1/4)=H_{t-1}-M_{t-1}(2)\log2-M_{t-1}(1)\log(3/2)<0$.
    Strict convexity therefore yields $p_t(2)>1/4$ for every $t\in\set{\tau+1,\ldots,2\tau}$.

    It remains to sum over disjoint single-spike events.
    For every $\tau\in\brk{\floor{T/2}}$, let $\mathcal{E}_\tau$ denote the event that the unique spike through round $T$ occurs at round $\tau$.
    Integral comparison gives
    \[
        \sum_{s=1}^\infty q_s\leq\frac{\epsilon}{8}\prn*{1+\int_1^\infty x^{-1-\epsilon}\,\mathrm{d}x}=\frac{1+\epsilon}{8}\leq\frac{1}{4}.
    \]
    Since $\prod_s(1-q_s)\geq1-\sum_s q_s$ for every finite product, we obtain
    \begin{equation}
        \label{eq:adaptive-ftrl-fixed-process-spike-probability}
        \P(\mathcal{E}_\tau)=q_\tau\prod_{\substack{s\leq T\\s\neq\tau}}(1-q_s)\geq\frac{3}{4}q_\tau\geq\frac{1}{2}q_\tau.
    \end{equation}
    The events $(\mathcal{E}_\tau)_{\tau=1}^{\floor{T/2}}$ are pairwise disjoint.
    Using \eqref{eq:adaptive-ftrl-fixed-process-self-bounded-identity} and the nonnegativity of $p_t(2)$, we conclude that
    \begin{equation*}
        \mathcal{R}_T
        \geq\Delta\sum_{\tau=1}^{\floor{T/2}}\E\brk*{\ind\brk{\mathcal{E}_\tau}\sum_{t=\tau+1}^{2\tau}p_t(2)}
        \geq\frac{\Delta}{8}\sum_{\tau=1}^{\floor{T/2}}\tau q_\tau
        =\frac{\epsilon}{4096}\sum_{\tau=1}^{\floor{T/2}}\tau^{-\epsilon},
    \end{equation*}
    where we used $p_t(2)>1/4$ on $\mathcal{E}_\tau$ for every $t\in\set{\tau+1,\ldots,2\tau}$ in the second inequality and \eqref{eq:adaptive-ftrl-fixed-process-spike-probability} in the last inequality.
    This proves \eqref{eq:adaptive-ftrl-fixed-process-lower-bound}.

    When $\epsilon=1$, the harmonic sum is at least $\log(\floor{T/2}+1)$, and when $\epsilon\in(0,1)$, it is $\Omega(T^{1-\epsilon})$.
\end{proof}

\subsection{Details of \cref{subsec:omd}}
\label{app:subsec:omd}
We provide the details of the analysis of \cref{alg:parameter-free-heavy-tailed-omd}.

We first analyze the regret over a prefix of a fixed epoch.
Fix an epoch $m$ and a prefix endpoint $T_m\in I_m$, and let $J_m=\set{\min I_m,\ldots,T_m}$.
Fix $\epsilon\in(0,1]$.
For an arbitrary comparator $u\in\Delta_K$, write $\tilde{\mathcal{R}}_{J_m}(u)=\sum_{t\in J_m}\inpr{p_t-u,\tilde{\ell}_t}$.
Also define $\trunc_{J_m}(p)=\sum_{t\in J_m}\inpr{p_t,\bar{\ell}_t-\tilde{\ell}_t}$ and $\trunc_{J_m}(-e_i)=\sum_{t\in J_m}\inpr{-e_i,\bar{\ell}_t-\tilde{\ell}_t}$.
For each $i\in[K]$, let
\[
    u_i^{(m)}=(1-a_m)e_i+\frac{a_m}{K}\bm{1}\in\mathcal{P}_{a_m},
\]
where $\bm{1}\in\R^K$ denotes the all-ones vector.

The analysis proceeds by first bounding $\tilde{\mathcal{R}}_{J_m}(u_i^{(m)})$, then accounting for comparator smoothing and truncation, and finally controlling the adaptive scales.

\begin{lemma}
    \label{lem:omd-smoothed-comparator}
    For every loss sequence, every epoch $m$, every prefix endpoint $T_m\in I_m$, and every $i\in[K]$, \cref{alg:parameter-free-heavy-tailed-omd} satisfies
    \[
        \tilde{\mathcal{R}}_{J_m}(u_i^{(m)})\leq\prn*{2+\frac{e}{2}}\Lambda_mG_{T_m}^{(m)}+\Lambda_mB_{T_m}^{(m)}(i),
    \]
    where $J_m=\set{\min I_m,\ldots,T_m}$.
\end{lemma}

\begin{proof}
    Within this proof, let $p_{T_m+1}$ denote the point obtained by the one-step mirror update in epoch $m$, even when $T_m=\max I_m$.
    The standard three-point identity for the OMD update gives, for every $t\in J_m$,
    \begin{equation*}
        \inpr{p_t-u_i^{(m)},\tilde{\ell}_t}\leq D_t^{(m)}(u_i^{(m)};p_t)-D_t^{(m)}(u_i^{(m)};p_{t+1})+\operatorname{Stab}_t,
    \end{equation*}
    where we define $\operatorname{Stab}_t=\inpr{p_t-p_{t+1},\tilde{\ell}_t}-D_t^{(m)}(p_{t+1};p_t)$.

    We first bound the stability term.
    Rounds with $\max_{j\in[K]}M_t^{(m)}(j)=0$ have $\tilde{\ell}_t=0$ and contribute nothing.
    If $\max_{j\in[K]}M_t^{(m)}(j)>0$, then $G_t^{(m)}>0$, since any positive increment of some $B_t^{(m)}(j)$ also yields a positive increment of $G_t^{(m)}$ by $p_t(j)\geq a_m/K>0$.
    Hence, we have $M_t^{(m)}(j)\geq G_t^{(m)}>0$ for every $j\in[K]$.
    Define the unconstrained point $\hat{p}_{t+1}=\argmin_{p\in\R_{>0}^K}\set{\inpr{p,\tilde{\ell}_t}+D_t^{(m)}(p;p_t)}$.
    Its first-order optimality condition gives
    \begin{equation}
        \label{eq:omd-stability-auxiliary}
        \hat{p}_{t+1}(j)=p_t(j)\exp\prn*{-\frac{\tilde{\ell}_t(j)}{M_t^{(m)}(j)}}
    \end{equation}
    for every $j\in[K]$.
    Since $p_{t+1}$ minimizes the same objective over $\mathcal{P}_{a_m}$, comparison with $\hat{p}_{t+1}$ yields
    \begin{align*}
        \operatorname{Stab}_t
        &\leq \inpr{p_t,\tilde{\ell}_t}-\prn{\inpr{\hat{p}_{t+1},\tilde{\ell}_t}+D_t^{(m)}(\hat{p}_{t+1};p_t)}\\
        &=\sum_{j=1}^KM_t^{(m)}(j)p_t(j)\prn*{\exp\prn*{-\frac{\tilde{\ell}_t(j)}{M_t^{(m)}(j)}}-1+\frac{\tilde{\ell}_t(j)}{M_t^{(m)}(j)}}\\
        &\leq\frac{e}{2}\sum_{j=1}^K\frac{p_t(j)\tilde{\ell}_t(j)^2}{M_t^{(m)}(j)},
    \end{align*}
    where the equality follows from \eqref{eq:omd-stability-auxiliary} and the definition of $D_t^{(m)}$, and the final inequality follows from $e^{-x}-1+x\leq ex^2/2$ for $\abs{x}\leq1$, which applies because
    \[
        \frac{\abs{\tilde{\ell}_t(j)}}{M_t^{(m)}(j)}=\frac{\abs{r_t(j)}}{M_t^{(m)}(j)}\ind\brk*{\abs{r_t(j)}\leq M_t^{(m)}(j)}\leq1.
    \]
    Moreover, \hyperref[p:scale-monotonicity]{(P1)}, \hyperref[p:regime-lower-bounds]{(P3)}, and $M_{t-1}^{(m)}(j)\leq M_t^{(m)}(j)$ give
    \begin{equation*}
        \frac{\tilde{\ell}_t(j)^2}{M_t^{(m)}(j)}
        =\frac{r_t(j)^2}{M_t^{(m)}(j)}\ind\brk*{\abs{r_t(j)}\leq M_t^{(m)}(j)}
        \leq\phi_{M_t^{(m)}(j)}\prn{r_t(j)}
        \leq\phi_{M_{t-1}^{(m)}(j)}\prn{r_t(j)}.
    \end{equation*}
    Consequently, it holds that
    \begin{equation}
        \label{eq:omd-stability-bound}
        \sum_{t\in J_m}\operatorname{Stab}_t\leq\frac{e}{2}\sum_{t\in J_m} \sum_{j=1}^K p_t(j)\phi_{M_{t-1}^{(m)}(j)}(r_t(j)) \leq \frac{e}{2}\sum_{t\in J_m}g_t=\frac{e}{2}\Lambda_mG_{T_m}^{(m)}.
    \end{equation}

    We next bound the drift of the time-varying regularizer $\sum_{t\in J_m}\prn{D_t^{(m)}(u_i^{(m)};p_t)-D_t^{(m)}(u_i^{(m)};p_{t+1})}$.
    Set $D_{\min I_m-1}^{(m)}=0$ and $h(x,y)=x\log(x/y)-x+y$.
    Telescoping the Bregman terms gives
    \begin{align*}
        &\sum_{t\in J_m}\prn*{D_t^{(m)}(u_i^{(m)};p_t)-D_t^{(m)}(u_i^{(m)};p_{t+1})}\\
        &=\sum_{t\in J_m}\prn*{D_t^{(m)}(u_i^{(m)};p_t)-D_{t-1}^{(m)}(u_i^{(m)};p_t)} - D_{T_m}^{(m)}(u_i^{(m)};p_{T_m+1})\\
        &\leq\sum_{t\in J_m}\sum_{j=1}^K\prn*{M_t^{(m)}(j)-M_{t-1}^{(m)}(j)}h\prn{u_i^{(m)}(j),p_t(j)}.
    \end{align*}
    For $j\neq i$, the inequalities $u_i^{(m)}(j)=a_m/K\leq p_t(j)$ and $h(x,y)\leq y$ for $x\leq y$ imply $h(u_i^{(m)}(j),p_t(j))\leq p_t(j)$.
    For coordinate $i$, the bounds $u_i^{(m)}(i)\leq1$ and $p_t(i)\geq a_m/K$ give $h\prn{u_i^{(m)}(i),p_t(i)}\leq\log\frac{K}{a_m}+p_t(i)$.
    By the definition of $M_t^{(m)}(j)$ and the update rules of \cref{alg:parameter-free-heavy-tailed-omd}, we have
    \[
        M_t^{(m)}(j)-M_{t-1}^{(m)}(j)\leq\frac{1}{\Lambda_m}\prn{g_t+d_t(j)}.
    \]
    Using $g_t=\sum_{j=1}^Kp_t(j)d_t(j)$, we obtain
    \[
        \sum_{t\in J_m}\sum_{j=1}^Kp_t(j)\prn{M_t^{(m)}(j)-M_{t-1}^{(m)}(j)}\leq\frac{1}{\Lambda_m}\sum_{t\in J_m}\prn*{g_t+\sum_{j=1}^Kp_t(j)d_t(j)}=2G_{T_m}^{(m)}.
    \]
    Moreover,
    \[
        \sum_{t\in J_m}\prn{M_t^{(m)}(i)-M_{t-1}^{(m)}(i)}=M_{T_m}^{(m)}(i)\leq G_{T_m}^{(m)}+B_{T_m}^{(m)}(i).
    \]
    Therefore, we have
    \begin{align}
        \label{eq:omd-regularizer-drift-bound}
        &\sum_{t\in J_m}\prn*{D_t^{(m)}(u_i^{(m)};p_t)-D_t^{(m)}(u_i^{(m)};p_{t+1})}\nonumber\\
        &\leq\sum_{t\in J_m}\sum_{j=1}^Kp_t(j)\prn*{M_t^{(m)}(j)-M_{t-1}^{(m)}(j)}+\log\frac{K}{a_m}\sum_{t\in J_m}\prn*{M_t^{(m)}(i)-M_{t-1}^{(m)}(i)}\nonumber\\
        &\leq2G_{T_m}^{(m)}+(\Lambda_m-1)\prn*{G_{T_m}^{(m)}+B_{T_m}^{(m)}(i)}.
    \end{align}
    Combining \eqref{eq:omd-stability-bound} and \eqref{eq:omd-regularizer-drift-bound}, and using $\Lambda_m\geq1$, proves the claim.
\end{proof}

We next return from the surrogate losses to the original losses while accounting for comparator smoothing and truncation.
\begin{lemma}
    \label{lem:omd-regret-scales}
    For every loss sequence, every epoch $m$, every prefix endpoint $T_m\in I_m$, and every $i\in[K]$, \cref{alg:parameter-free-heavy-tailed-omd} satisfies
    \[
        \sum_{t\in J_m} r_t(i)\leq {(7+e)\Lambda_m^{\frac{\epsilon}{1+\epsilon}}\prn*{\sum_{t\in J_m}\bar{w}_t}^{\frac{1}{1+\epsilon}}+5\Lambda_m^{\frac{\epsilon}{1+\epsilon}}\prn*{\sum_{t\in J_m}w_t(i)}^{\frac{1}{1+\epsilon}}},
    \]
    where $J_m=\set{\min I_m,\ldots,T_m}$.
\end{lemma}

\begin{proof}
    Fix $i\in[K]$ and recall that $\bar{\ell}_t=-r_t$.
    The definitions of $u_i^{(m)}$ and $\tilde{\mathcal{R}}_{J_m}$, together with $\abs{\tilde{\ell}_t(j)}\leq\abs{r_t(j)}$, give
    \[
        \tilde{\mathcal{R}}_{J_m}(e_i)-\tilde{\mathcal{R}}_{J_m}(u_i^{(m)})\leq\frac{a_m}{K}\sum_{t\in J_m}\sum_{j=1}^K\abs{r_t(j)}+a_m\sum_{t\in J_m}\abs{r_t(i)}.
    \]
    Since $a_m\abs{J_m}\leq1$ and $p_t(j)\geq a_m/K$, H\"older's inequality yields
    \begin{align*}
        \frac{a_m}{K}\sum_{t\in J_m}\sum_{j=1}^K\abs{r_t(j)}
        &\leq(a_m\abs{J_m})^{\frac{\epsilon}{1+\epsilon}}\prn*{\frac{a_m}{K}\sum_{t\in J_m}\sum_{j=1}^K w_t(j)}^{\frac{1}{1+\epsilon}}
        \leq\prn*{\sum_{t\in J_m}\bar{w}_t}^{\frac{1}{1+\epsilon}},\\
        a_m\sum_{t\in J_m}\abs{r_t(i)}
        &\leq(a_m\abs{J_m})^{\frac{\epsilon}{1+\epsilon}}\prn*{a_m\sum_{t\in J_m}w_t(i)}^{\frac{1}{1+\epsilon}}
        \leq \prn*{\sum_{t\in J_m}w_t(i)}^{\frac{1}{1+\epsilon}}.
    \end{align*}
    Since $\inpr{p_t,r_t}=0$, the pathwise cumulative regret decomposes as
    \[
        \sum_{t\in J_m} r_t(i)=\tilde{\mathcal{R}}_{J_m}(e_i)+\trunc_{J_m}(p)+\trunc_{J_m}(-e_i).
    \]
    Applying \cref{lem:omd-smoothed-comparator,lem:multiscale-truncation-error} with $\mathrm{A}=\mathrm{O}$ therefore gives
    \[
        \sum_{t\in J_m} r_t(i)\leq\prn*{3+\frac{e}{2}}\Lambda_mG_{T_m}^{(m)}+2\Lambda_mB_{T_m}^{(m)}(i)+\prn*{\sum_{t\in J_m}\bar{w}_t}^{\frac{1}{1+\epsilon}}+\prn*{\sum_{t\in J_m}w_t(i)}^{\frac{1}{1+\epsilon}}.
    \]
    By \cref{lem:multiscale-scale-recursion} and $\prn{(1+\epsilon)(1+1/\Lambda_m)^\epsilon}^{\frac{1}{1+\epsilon}}\leq2$, we have
    \[
        G_{T_m}^{(m)}\leq2\prn*{\frac{\sum_{t\in J_m}\bar{w}_t}{\Lambda_m}}^{\frac{1}{1+\epsilon}},\quad\text{and}\quad B_{T_m}^{(m)}(i)\leq2\prn*{\frac{\sum_{t\in J_m}w_t(i)}{\Lambda_m}}^{\frac{1}{1+\epsilon}}.
    \]
    Substituting these bounds and using $\Lambda_m\geq1$ proves the claim.
\end{proof}

We now combine the regret bounds across epochs.
\begin{lemma}
    \label{lem:omd-anytime-pathwise}
    For every loss sequence, $\epsilon\in(0,1]$, integer $T\geq2$, and $i\in[K]$, \cref{alg:parameter-free-heavy-tailed-omd} satisfies
    \[
        \sum_{t=1}^T r_t(i)\leq\prn{8\log(KT)}^{\frac{\epsilon}{1+\epsilon}}\prn*{(7+e)\prn*{\sum_{t=1}^T\bar{w}_t}^{\frac{1}{1+\epsilon}}+5\prn*{\sum_{t=1}^T w_t(i)}^{\frac{1}{1+\epsilon}}}.
    \]
\end{lemma}

\begin{proof}
    Let $M$ be the index of the last epoch started by round $T$, and write $J_m=I_m\cap[T]$ for the executed prefix of epoch $m$.
    Recall that $\Lambda_m=1+\log(KH_m)$.
    By \cref{lem:omd-regret-scales}, for every $m\in\set{0,\dots,M}$, we have
    \[
        \sum_{t\in J_m}r_t(i)\leq {(7+e)\Lambda_m^{\frac{\epsilon}{1+\epsilon}}\prn*{\sum_{t\in J_m}\bar{w}_t}^{\frac{1}{1+\epsilon}}+5\Lambda_m^{\frac{\epsilon}{1+\epsilon}}\prn*{\sum_{t\in J_m}w_t(i)}^{\frac{1}{1+\epsilon}}}.
    \]
    The definition of $H_m$ gives $\log(KH_{m+1})=2\log(KH_m)$ and $\log(KH_m)\geq\log4>1$.
    Thus, it holds that
    \[
    \sum_{m=0}^M \Lambda_m\leq2\sum_{m=0}^M\log(KH_m)=2\log(KH_M)\sum_{m=0}^M2^{m-M}\leq4\log(KH_M).
    \]
    If $M\geq1$, epoch $M-1$ has ended, and hence $T>H_{M-1}$ gives $\log(KH_M)=2\log(KH_{M-1})\leq2\log(KT)$.
    If $M=0$, the same inequality follows from $\log(KH_0)=2\log K\leq2\log(KT)$.
    Therefore, we have
    \begin{equation}
        \label{eq:omd-epoch-log-sum}
        \sum_{m=0}^M \Lambda_m\leq8\log(KT).
    \end{equation}
    Combining these bounds and applying H\"older's inequality give
    \begin{align*}
        \sum_{t=1}^T r_t(i)
        &\leq\prn*{\sum_{m=0}^M \Lambda_m}^{\frac{\epsilon}{1+\epsilon}}\prn*{(7+e)\prn*{\sum_{m=0}^M\sum_{t\in J_m}\bar{w}_t}^{\frac{1}{1+\epsilon}}+5\prn*{\sum_{m=0}^M \sum_{t\in J_m}w_t(i)}^{\frac{1}{1+\epsilon}}}\\
        &\leq\prn{8\log(KT)}^{\frac{\epsilon}{1+\epsilon}}\prn*{(7+e)\prn*{\sum_{t=1}^T\bar{w}_t}^{\frac{1}{1+\epsilon}}+5\prn*{\sum_{t=1}^T w_t(i)}^{\frac{1}{1+\epsilon}}}.
    \end{align*}
    This proves the claim.
\end{proof}

We can now prove \cref{thm:omd-regret-adversarial}.
\begin{proof}[Proof of \cref{thm:omd-regret-adversarial}]
    Fix $i\in[K]$.
    By \cref{lem:known-parameter-centered-moments}, we have
    \[
        \E\brk*{\sum_{t=1}^T\bar{w}_t}\leq2^{1+\epsilon}\theta T,\qquad \E\brk*{\sum_{t=1}^T w_t(i)}\leq2^{1+\epsilon}\theta T.
    \]
    Taking expectations in \cref{lem:omd-anytime-pathwise} and applying Jensen's inequality gives
    \[
        \mathcal{R}_T(i)=O\prn[\big]{(\theta T)^{\frac{1}{1+\epsilon}}\prn{\log(KT)}^{\frac{\epsilon}{1+\epsilon}}}.
    \]
    Taking the maximum over $i\in[K]$ proves the claim.
\end{proof}

For the self-bounded analysis, we need to bound $\bar{w}_t$ and $w_t(i^\ast)$ by using $\prn{1-p_t(i^\ast)}$.
\begin{lemma}
    \label{lem:omd-localized-centered-moments}
    Under \cref{asp:unknown-epsilon_wise}, for every distinguished expert $i^\ast\in[K]$ and $t\in[T]$, it holds that
    \[
        \E_t\brk{\bar{w}_t+w_t(i^\ast)}\leq12\theta\prn{1-p_t(i^\ast)}.
    \]
\end{lemma}

\begin{proof}
    Conditional Jensen's inequality gives
    \[
        w_t(i^\ast)\leq\sum_{j=1}^Kp_t(j)\abs{\ell_t(j)-\ell_t(i^\ast)}^{1+\epsilon}.
    \]
    The summand with $j=i^\ast$ vanishes, whereas \cref{asp:unknown-epsilon_wise} implies $\E_t\brk*{\abs{\ell_t(j)-\ell_t(i^\ast)}^{1+\epsilon}}\leq2^{1+\epsilon}\theta$.
    Hence $\E_t\brk{w_t(i^\ast)}\leq2^{1+\epsilon}\theta\prn{1-p_t(i^\ast)}$.
    Applying conditional Jensen's inequality once more gives
    \[
        \bar{w}_t\leq\sum_{i=1}^K\sum_{j=1}^Kp_t(i)p_t(j)\abs{\ell_t(j)-\ell_t(i)}^{1+\epsilon}.
    \]
    The $(i^\ast,i^\ast)$ term vanishes, the remaining weights sum to $1-p_t(i^\ast)^2\leq2\prn{1-p_t(i^\ast)}$, and every pairwise difference has conditional $(1+\epsilon)$-moment at most $2^{1+\epsilon}\theta$.
    Combining the two bounds and using $1+\epsilon\leq2$ proves the claim.
\end{proof}

Now, we are ready to prove \cref{thm:omd-regret-self-bounded}.
\begin{proof}[Proof of \cref{thm:omd-regret-self-bounded}]
    Define $Q_T=\sum_{t=1}^T\E\brk{1-p_t(i^\ast)}$.
    By \cref{lem:omd-localized-centered-moments}, we have
    \[
        \E\brk*{\sum_{t=1}^T\prn{\bar{w}_t+w_t(i^\ast)}}\leq12\theta Q_T.
    \]
    Taking expectations in \cref{lem:omd-anytime-pathwise} and applying Jensen's inequality yields
    \begin{equation}
        \label{eq:omd-self-bounded-upper}
        \mathcal{R}_T\leq c_{\mathrm{O}}Q_T^{\frac{1}{1+\epsilon}},
    \end{equation}
    where we set $c_{\mathrm{O}}=O\prn{\theta^{\frac{1}{1+\epsilon}}\prn{\log(KT)}^{\frac{\epsilon}{1+\epsilon}}}$ and used $\mathcal{R}_T=\mathcal{R}_T(i^\ast)$ by \cref{asp:self_bounded}.
    For every $\lambda\in(0,1]$, \cref{asp:self_bounded} and \eqref{eq:omd-self-bounded-upper} imply
    \begin{equation*}
        \mathcal{R}_T
        =(1+\lambda)\mathcal{R}_T-\lambda\mathcal{R}_T
        \leq(1+\lambda)c_{\mathrm{O}}Q_T^{\frac{1}{1+\epsilon}}-\lambda\Delta_{\min}Q_T+\lambda C.
    \end{equation*}
    The maximization formula $\sup_{x\geq0}\set{bx^{\frac{1}{1+\epsilon}}-cx}=O(b^{\frac{1+\epsilon}{\epsilon}}c^{-\frac{1}{\epsilon}})$ gives
    \[
        \mathcal{R}_T=O\prn*{X\lambda^{-\frac{1}{\epsilon}}+\lambda C}\quad\text{and}\quad X=c_{\mathrm{O}}^{\frac{1+\epsilon}{\epsilon}}\Delta_{\min}^{-\frac{1}{\epsilon}}=O\prn*{\log(KT)\prn*{\frac{\theta}{\Delta_{\min}}}^{\frac{1}{\epsilon}}}.
    \]
    If $C=0$, take $\lambda=1$.
    If $C>0$, choosing $\lambda$ of order $\min\set{1,(X/C)^{\frac{\epsilon}{1+\epsilon}}}$ gives
    \[
        \mathcal{R}_T=O\prn{X+X^{\frac{\epsilon}{1+\epsilon}}C^{\frac{1}{1+\epsilon}}}.
    \]
    Substituting the definition of $X$ therefore proves both claims.
\end{proof}

\begin{remark}[Additional logarithmic factor under geometric doubling]
    \label{rem:additional-logarithmic-factor-geometric-doubling}
    Suppose that the epoch length in \cref{alg:parameter-free-heavy-tailed-omd} is replaced by geometric doubling, with $H_0=K$ and $H_{m+1}=2H_m$, while the updates within each epoch are unchanged.
    Let $M$ be the index of the last epoch started by round $T$.
    Then $H_m=K2^m$, and the epoch sum in \eqref{eq:omd-epoch-log-sum} instead satisfies
    \[
        \sum_{m=0}^M \Lambda_m=(M+1)(1+2\log K)+\frac{M(M+1)}{2}\log2=O\prn{\log(KT)+\log^2 T},
    \]
    where we used the fact that $M=0$ when $T\leq K$ and the bounds $M\leq1+\log_2(T/K)$ and $\log K\leq\log T$ when $T>K$.
    Consequently, the same self-bounded analysis yields a gap-dependent regret bound with a logarithmic factor of $O\prn{\log(KT)+\log^2 T}$ rather than $O\prn{\log(KT)}$.
    In contrast, $H_{m+1}=KH_m^2$ gives $\log(KH_{m+1})=2\log(KH_m)$ and hence $\sum_{m=0}^M \Lambda_m=O\prn{\log(KT)}$, as shown above.
\end{remark}

\subsection{Details of \cref{subsec:non-adaptive-bobw}}
\label{app:subsec:non-adaptive}
\subsubsection{Common lemmas}
We introduce and prove some lemmas which we use in the analysis of both the adversarial and self-bounded regret bounds.
\begin{lemma}
    \label{lem:known-parameter-stability-clipping}
    Consider \cref{alg:known-parameter-shannon-ftrl}.
    For every round $t\in[T]$, it holds that
    \begin{equation*}
        \frac{1}{\beta_t}\sum_{j=1}^Kp_t(j)\tildeell_t(j)^2\leq\beta_t^{-\epsilon}\bar{w}_t.
    \end{equation*}
    Moreover, for any $i\in[K]$, the one-round truncation residual satisfies
    \begin{equation*}
        \abs{\inpr{p_t-e_i,-r_t-\tildeell_t}}\leq\beta_t^{-\epsilon}\prn{\bar{w}_t+w_t(i)}.
    \end{equation*}
\end{lemma}

\begin{proof}
    Fix $t\in[T]$.
    If $\abs{x}\leq \beta_t$, then $x^2\leq \beta_t^{1-\epsilon}\abs{x}^{1+\epsilon}$.
    The definition of $\tildeell_t$ therefore gives
    \begin{align*}
        \frac{1}{\beta_t}\sum_{j=1}^Kp_t(j)\tildeell_t(j)^2
        &=\frac{1}{\beta_t}\sum_{j=1}^Kp_t(j)r_t(j)^2\ind\brk*{\abs{r_t(j)}\leq \beta_t}
        \leq\frac{\beta_t^{1-\epsilon}}{\beta_t}\sum_{j=1}^Kp_t(j)\abs{r_t(j)}^{1+\epsilon}
        =\beta_t^{-\epsilon}\bar{w}_t,
    \end{align*}
    which proves the first claim.
    Moreover, we have $-r_t(j)-\tildeell_t(j)=-r_t(j)\ind\brk*{\abs{r_t(j)}> \beta_t}$ and $\abs{x}\ind\brk*{\abs{x}> \beta_t}\leq \beta_t^{-\epsilon}\abs{x}^{1+\epsilon}$.
    Hence, for any $i\in[K]$, it holds that
    \begin{align*}
        \abs{\inpr{p_t-e_i,-r_t-\tildeell_t}}
        &\leq\sum_{j=1}^Kp_t(j)\abs{r_t(j)}\ind\brk{\abs{r_t(j)}> \beta_t}+\abs{r_t(i)}\ind\brk{\abs{r_t(i)}> \beta_t}\\
        &\leq \beta_t^{-\epsilon}\prn{\bar{w}_t+w_t(i)},
    \end{align*}
    which proves the second claim.
\end{proof}

The following lemma is a standard second-order bound for entropy-FTRL with a time-varying regularization scale (\textit{e.g.}, \citet[lemma 6]{NEURIPS2020_amir_prediction_expert}, \citet[lemma 1]{NEURIPS2021_ito_robustness_adversarial_decision}). 
We state the entropy-drift term explicitly, as this form will be needed in the self-bounded analysis.
\begin{lemma}
    \label{lem:known-parameter-refined-ftrl-formal}
    Let $\pi$ be the uniform distribution on $[K]$ and $0<\beta_1\leq\cdots\leq\beta_T$.
    Suppose that $\tildeell_1,\ldots,\tildeell_T\in\R^K$ satisfy $\max_j\abs{\tildeell_t(j)}\leq\beta_t$.
    Set $\tildeL_0=0$ and $\tildeL_t=\sum_{s=1}^t\tildeell_s$, and define $p_t\in\argmin_{p\in\Delta_K}\set{\inpr{\tildeL_{t-1},p}+\beta_t\KL(p;\pi)}$.
    Then, for every $u\in\Delta_K$, it holds that
    \begin{equation}
        \label{eq:known-parameter-refined-ftrl-target}
        \sum_{t=1}^T\inpr{p_t-u,\tildeell_t}
        \leq\beta_1\log K+\sum_{t=1}^{T-1}(\beta_{t+1}-\beta_t)H(p_{t+1})+\frac{e}{2}\sum_{t=1}^T\sum_{j=1}^K \frac{p_t(j)\tildeell_t(j)^2}{\beta_t},
    \end{equation}
    where $H(p)=-\sum_jp(j)\log p(j)$.
    Consequently, it holds that
    \begin{equation}
        \label{eq:known-parameter-ftrl-master-target}
        \sum_{t=1}^T\inpr{p_t-e_i,\tildeell_t}\leq\beta_T\log K+\frac{e}{2}\sum_{t=1}^T\sum_{j=1}^K \frac{p_t(j)\tildeell_t(j)^2}{\beta_t}.
    \end{equation}
\end{lemma}

\begin{proof}
    For $\beta>0$ and $y\in\R^K$, let $\Phi_\beta(y)$ denote the optimal value of the entropy-regularized objective with cumulative loss vector $y$ and regularization scale $\beta$.
    We can write $\Phi_\beta(y)$ as
    \[
        \Phi_\beta(y)=\min_{p\in\Delta_K}\set*{\inpr{y,p}+\beta\KL(p;\pi)}=-\beta\log\prn*{\sum_{j=1}^K\pi(j)e^{-\frac{y(j)}{\beta}}}.
    \]
    At round $t$, define the mix loss $m_t$ as
    \[
        m_t=\Phi_{\beta_t}(\tildeL_t)-\Phi_{\beta_t}(\tildeL_{t-1})=-\beta_t\log\prn*{\sum_{j=1}^Kp_t(j)e^{-\frac{\tildeell_t(j)}{\beta_t}}},
    \]
    where the second equality follows from the definition of $p_t$ as the minimizer associated with $\tildeL_{t-1}$ and $\beta_t$.

    For each round, the loss difference against $u$ is the sum of the one-step mixability gap and the difference between the mix loss and the loss of $u$.
    Summing this decomposition gives
    \[
        \sum_{t=1}^T\inpr{p_t-u,\tildeell_t}=\sum_{t=1}^T\prn{\inpr{p_t,\tildeell_t}-m_t}+\sum_{t=1}^T\prn{m_t-\inpr{u,\tildeell_t}}.
    \]

    First, we bound the first term $\sum_{t=1}^T\prn{\inpr{p_t,\tildeell_t}-m_t}$.
    Set $z_j=-\tildeell_t(j)/\beta_t$, which satisfies $\abs{z_j}\leq1$ by the definition.
    Using $\log u\leq u-1$ for $u>0$ and $e^z\leq1+z+\frac{e}{2}z^2$ for every $z\in[-1,1]$, we obtain
    \begin{equation}
        \label{eq:known-parameter-formal-one-step-mix-loss}
        \sum_{t=1}^T\prn{\inpr{p_t,\tildeell_t}-m_t}
        =\sum_{t=1}^T\beta_t\prn*{\log\prn*{\sum_{j=1}^Kp_t(j)e^{z_j}}-\sum_{j=1}^Kp_t(j)z_j}
        \leq \frac{e}{2}\sum_{t=1}^T\sum_{j=1}^K \frac{p_t(j)\tildeell_t(j)^2}{\beta_t}.
    \end{equation}
    We next bound the second term $\sum_{t=1}^T\prn{m_t-\inpr{u,\tildeell_t}}$.
    From the definition of $m_t$, we have
    \begin{equation*}
        \sum_{t=1}^Tm_t=\Phi_{\beta_T}(\tildeL_T)-\Phi_{\beta_1}(0)+\sum_{t=1}^{T-1}\prn{\Phi_{\beta_t}(\tildeL_t)-\Phi_{\beta_{t+1}}(\tildeL_t)}.
    \end{equation*}
    For every $t\in[T-1]$, the definition of the FTRL update shows that $p_{t+1}$ minimizes the objective defining $\Phi_{\beta_{t+1}}(\tildeL_t)$.
    Thus, it holds that
    \begin{equation*}
        \Phi_{\beta_t}(\tildeL_t)
        \leq\inpr{\tildeL_t,p_{t+1}}+\beta_t\KL(p_{t+1};\pi)
        =\Phi_{\beta_{t+1}}(\tildeL_t)-(\beta_{t+1}-\beta_t)\KL(p_{t+1};\pi).
    \end{equation*}
    In addition, we have
    \begin{equation*}
        \Phi_{\beta_T}(\tildeL_T)\leq\inpr{\tildeL_T,u}+\beta_T\KL(u;\pi).
    \end{equation*}
    Substituting these inequalities into the telescoping identity gives
    \begin{equation}
        \label{eq:known-parameter-formal-cumulative-mix-loss}
        \sum_{t=1}^T\prn{m_t-\inpr{u,\tildeell_t}}\leq\beta_T\KL(u;\pi)-\sum_{t=1}^{T-1}(\beta_{t+1}-\beta_t)\KL(p_{t+1};\pi).
    \end{equation}
    
    Under the convention $0\log0=0$, the uniformity of $\pi$ implies $\KL(p;\pi)=\log K-H(p)\leq \log K$.
    Using \eqref{eq:known-parameter-formal-one-step-mix-loss}, \eqref{eq:known-parameter-formal-cumulative-mix-loss}, and the identity $\beta_T=\beta_1+\sum_{t=1}^{T-1}(\beta_{t+1}-\beta_t)$ proves \eqref{eq:known-parameter-refined-ftrl-target}.
    Finally, replacing each $H(p_{t+1})$ by its upper bound $\log K$ and using the same identity for $\beta_T$ proves \eqref{eq:known-parameter-ftrl-master-target}.
\end{proof}

For every $t\in[T]$ and $i\in[K]$, we define
\begin{equation}
    \label{eq:known-parameter-gamma-def}
    W_t(i)=\prn*{\frac{e}{2}+1}\bar{w}_t+w_t(i).
\end{equation}

\subsubsection{Adversarial}
\label{subsec:known-parameter-formal-adversarial}
In this section, we prove \cref{thm:known-parameter-adversarial}.
First, we state \cref{thm:known-parameter-adversarial} in a more formal way.
\begin{theorem}[Formal version of \cref{thm:known-parameter-adversarial}]
    \label{thm:known-parameter-adversarial-formal}
    Suppose that \cref{asp:unknown-epsilon_wise} holds.
    Define $C_\epsilon=2^\epsilon e+2^{2+\epsilon}$.
    Then, for any $T\geq1$ and $i\in[K]$, \cref{alg:known-parameter-shannon-ftrl} satisfies
    \begin{equation}
        \label{eq:known-parameter-formal-adversarial-intermediate}
        \mathcal{R}_T(i)\leq\beta_T\log K+C_\epsilon\theta\sum_{t=1}^T\beta_t^{-\epsilon}.
    \end{equation}
    In particular, it holds that
    \begin{equation}
        \label{eq:known-parameter-formal-adversarial-rate}
        \mathcal{R}_T(i)\leq\prn*{1+(1+\epsilon)C_\epsilon}(\theta T)^{\frac{1}{1+\epsilon}}(1+\log K)^{\frac{\epsilon}{1+\epsilon}}.
    \end{equation}
\end{theorem}

\begin{proof}
    The truncation rule gives $\abs{\tildeell_t(j)}\leq\beta_t$, so \cref{lem:known-parameter-refined-ftrl-formal} applies to every realized loss sequence.
    Using \cref{lem:known-parameter-stability-clipping,lem:known-parameter-refined-ftrl-formal} gives
    \begin{equation*}
        \sum_{t=1}^T\prn*{\inpr{p_t,\ell_t}-\ell_t(i)}\leq\beta_T\log K+\sum_{t=1}^T\beta_t^{-\epsilon}W_t(i).
    \end{equation*}

    Hence, the definition of $W_t(i)$ and \cref{lem:known-parameter-centered-moments} give
    \[
        \E_t\brk*{W_t(i)}\leq2^{1+\epsilon}\theta\prn*{\frac{e}{2}+2}=C_\epsilon\theta.
    \]
    Since the quantities $p_t$ and $\beta_t$ are $\mathcal{F}_{t-1}$-measurable, taking expectations in the pathwise bound gives
    \begin{equation*}
        \mathcal{R}_T(i)\leq\beta_T\log K+C_\epsilon\theta\sum_{t=1}^T\beta_t^{-\epsilon}.
    \end{equation*}

    Since $\log K\leq1+\log K$, we have
    \begin{align*}
        \beta_T\log K
        &=\prn{\theta T}^{\frac{1}{1+\epsilon}}(1+\log K)^{-\frac{1}{1+\epsilon}}\log K
        \leq(\theta T)^{\frac{1}{1+\epsilon}}(1+\log K)^{\frac{\epsilon}{1+\epsilon}},
    \end{align*}
    where we recall $\beta_T=\prn{\theta T}^{\frac{1}{1+\epsilon}}(1+\log K)^{-\frac{1}{1+\epsilon}}$.
    
    Moreover, we can upper bound $\sum_{t=1}^T\beta_t^{-\epsilon}$ as follows:
    \begin{align*}
        \theta\sum_{t=1}^T\beta_t^{-\epsilon}
        &=\theta^{\frac{1}{1+\epsilon}}(1+\log K)^{\frac{\epsilon}{1+\epsilon}}\sum_{t=1}^Tt^{-\frac{\epsilon}{1+\epsilon}}
        \leq(1+\epsilon)(\theta T)^{\frac{1}{1+\epsilon}}(1+\log K)^{\frac{\epsilon}{1+\epsilon}},
    \end{align*}
    where the last inequality follows from
    \begin{equation}
        \label{eq:known-parameter-sum-bound}
        \sum_{t=1}^Tt^{-\frac{\epsilon}{1+\epsilon}}\leq1+\int_1^Tx^{-\frac{\epsilon}{1+\epsilon}} \mathrm{d}x=(1+\epsilon)T^{\frac{1}{1+\epsilon}}-\epsilon\leq(1+\epsilon)T^{\frac{1}{1+\epsilon}}.
    \end{equation}

    Combining these bounds proves \eqref{eq:known-parameter-formal-adversarial-rate}.
\end{proof}

\subsubsection{Self-bounded}
\label{subsec:stochastic}
Here, we provide the proof of \cref{thm:known-parameter-stochastic}.

For convenience, define $M_\epsilon=1+\epsilon+\frac{(1+\epsilon)2^{\frac{1}{1+\epsilon}}\log2}{\prn*{1-2^{-\frac{\epsilon}{1+\epsilon}}}^2}$, $A_\epsilon=3\cdot2^\epsilon e+3\cdot2^{1+\epsilon}$ and $B_\epsilon=\frac{M_\epsilon}{1+\epsilon}+(1+\epsilon)A_\epsilon$.
First, we restate the theorem formally.
\begin{theorem}[Formal version of \cref{thm:known-parameter-stochastic}]
    \label{thm:known-parameter-stochastic-formal}
    Suppose that \cref{asp:unknown-epsilon_wise,asp:self_bounded} hold.
    Then, for any $T\geq1$, \cref{alg:known-parameter-shannon-ftrl} satisfies
    \begin{equation*}
        \mathcal{R}_T\leq 2\kappa_\epsilon (B_\epsilon+2^{\frac{1}{1+\epsilon}})^{\frac{1+\epsilon}{\epsilon}}\prn*{\frac{\theta}{\Delta_{\min}}}^{\frac{1}{\epsilon}}(1+\log K)+2(B_\epsilon+2^{\frac{1}{1+\epsilon}})\prn*{\frac{\theta C}{\Delta_{\min}}}^{\frac{1}{1+\epsilon}}(1+\log K)^{\frac{\epsilon}{1+\epsilon}},
    \end{equation*}
    where $\kappa_\epsilon=\frac{\epsilon}{1+\epsilon}\prn*{\frac{2}{1+\epsilon}}^{\frac{1}{\epsilon}}$.
    In particular, when $C=0$,
    \[
        \mathcal{R}_T\leq 2\kappa_\epsilon (B_\epsilon+2^{\frac{1}{1+\epsilon}})^{\frac{1+\epsilon}{\epsilon}}\prn*{\frac{\theta}{\Delta_{\min}}}^{\frac{1}{\epsilon}}(1+\log K).
    \]
\end{theorem}

To prove this theorem, we bound the regret using $1-p_t(i^\ast)$.
This analysis follows the general strategy of \citet{NEURIPS2021_ito_robustness_adversarial_decision}, while the technical details are adapted to the $(1+\epsilon)$-moment setting.
\begin{lemma}
    \label{lem:known-parameter-local-regret-formal}
    Suppose that \cref{asp:unknown-epsilon_wise} holds.
    Fix an expert $i\in[K]$.
    Then, \cref{alg:known-parameter-shannon-ftrl} satisfies
    \begin{equation*}
        \mathcal{R}_T(i)\leq\beta_1\log K+B_\epsilon\theta^{\frac{1}{1+\epsilon}}(1+\log K)^{\frac{\epsilon}{1+\epsilon}}\prn*{\sum_{t=1}^T\E\brk*{1-p_t(i)}}^{\frac{1}{1+\epsilon}}.
    \end{equation*}
\end{lemma}

\begin{proof}
    Combining \cref{lem:known-parameter-refined-ftrl-formal} with \cref{lem:known-parameter-stability-clipping} gives
    \begin{align}
        \label{eq:known-parameter-local-regret-pathwise}
        \sum_{t=1}^T{\inpr{p_t-e_i,\ell_t}}
        &=\sum_{t=1}^T\inpr{p_t-e_i,\tildeell_t}+\sum_{t=1}^T\inpr{p_t-e_i,-r_t-\tildeell_t}\nonumber\\
        &\leq\beta_1\log K+\sum_{t=1}^{T-1}\prn{\beta_{t+1}-\beta_t}H(p_{t+1})+\sum_{t=1}^T\beta_t^{-\epsilon}W_t(i),
    \end{align}
    where $W_t(i)$ is defined in \eqref{eq:known-parameter-gamma-def}.
    For every $j,k\in[K]$, the inequality $\abs{x-y}^{1+\epsilon}\leq2^\epsilon\prn*{\abs{x}^{1+\epsilon}+\abs{y}^{1+\epsilon}}$ and \cref{asp:unknown-epsilon_wise} imply
    \[
        \E_t\brk{\abs*{\ell_t(j)-\ell_t(k)}^{1+\epsilon}}\leq2^{1+\epsilon}\theta.
    \]
    Since $r_t(i)=\sum_{k=1}^Kp_t(k)\prn*{\ell_t(k)-\ell_t(i)}$, Jensen's inequality and the vanishing of the $k=i$ term give
    \[
        \E_t\brk*{w_t(i)}\leq\sum_{k\neq i}p_t(k)\E_t\brk*{\abs*{\ell_t(i)-\ell_t(k)}^{1+\epsilon}}\leq2^{1+\epsilon}\theta\prn*{1-p_t(i)}.
    \]
    Similarly, Jensen's inequality gives $\bar{w}_t\leq\sum_{j=1}^K\sum_{k=1}^Kp_t(j)p_t(k)\abs*{\ell_t(j)-\ell_t(k)}^{1+\epsilon}$.
    The $(i,i)$ term is zero, and the remaining weights sum to $1-p_t(i)^2\leq2(1-p_t(i))$, so we have
    \[
        \E_t\brk*{\bar{w}_t}\leq2^{1+\epsilon}\theta\prn*{1-p_t(i)^2}\leq2^{2+\epsilon}\theta\prn*{1-p_t(i)}.
    \]
    Hence, by the definition of $W_t(i)$, we have
    \begin{equation*}
        \E_t\brk*{W_t(i)}
        \leq\prn*{\frac{e}{2}+1}\E_t\brk*{\bar{w}_t+w_t(i)}
        \leq A_\epsilon\theta\prn*{1-p_t(i)}.
    \end{equation*}
    Taking expectations of \eqref{eq:known-parameter-local-regret-pathwise} and using the tower property give
    \begin{equation}
        \label{eq:known-parameter-local-regret-expected}
        \mathcal{R}_T(i)\leq\beta_1\log K+\sum_{t=1}^{T-1}\prn{\beta_{t+1}-\beta_t}\E\brk*{H(p_{t+1})}+A_\epsilon\theta\sum_{t=1}^T\beta_t^{-\epsilon}\E\brk*{1-p_t(i)}.
    \end{equation}

    First, we bound the second term.
    Set $x_t=\E[1-p_t(i)]$.
    By the binary entropy decomposition, we have $H(p_t)\leq h\prn{1-p_t(i)}+\prn{1-p_t(i)}\log(K-1),$ where $h(u)=-u\log u-(1-u)\log(1-u)$ for $u\in[0,1]$, under the convention $0\log0=0$.
    Since $h(u)\leq u\log(e/u)$ for every $u\in(0,1]$ and $K-1\leq K$, we obtain
    \[
        H(p_t)\leq\prn*{1-p_t(i)}\log\prn*{\frac{eK}{1-p_t(i)}}.
    \]
    When $p_t(i)=1$, the right-hand side is defined as zero by continuity.
    Since $x\mapsto x\log(eK/x)$ is concave on $[0,1]$, Jensen's inequality gives
    \[
        \E\brk*{H(p_{t+1})}\leq x_{t+1}\log\prn*{\frac{eK}{x_{t+1}}}.
    \]
    Concavity of $t\mapsto t^{\frac{1}{1+\epsilon}}$ implies that
    \[
        \beta_{t+1}- \beta_t =\theta^{\frac{1}{1+\epsilon}}(1+\log K)^{-\frac{1}{1+\epsilon}}\prn*{(t+1)^{\frac{1}{1+\epsilon}}-t^{\frac{1}{1+\epsilon}}}\leq\frac{\theta^{\frac{1}{1+\epsilon}}}{(1+\epsilon)(1+\log K)^{\frac{1}{1+\epsilon}}}t^{-\frac{\epsilon}{1+\epsilon}}.
    \]
    If $T=1$, the entropy-drift sum is zero.
    If $T\geq2$, using \eqref{eq:known-parameter-weighted-entropy-formal} in \cref{lem:known-parameter-weighted-entropy-formal} with $z_t=\E[1-p_{t+1}(i)]$ yields
    \begin{equation*}
        \sum_{t=1}^{T-1}\prn{\beta_{t+1}- \beta_t}\E\brk*{H(p_{t+1})}
        \leq\frac{M_\epsilon}{1+\epsilon}\theta^{\frac{1}{1+\epsilon}}(1+\log K)^{\frac{\epsilon}{1+\epsilon}}\prn*{\sum_{t=1}^T\E\brk*{1-p_t(i)}}^{\frac{1}{1+\epsilon}}.
    \end{equation*}

    Next, we bound the third term.
    We have $\theta\beta_t^{-\epsilon}=\theta^{\frac{1}{1+\epsilon}}(1+\log K)^{\frac{\epsilon}{1+\epsilon}}t^{-\frac{\epsilon}{1+\epsilon}}$.
    Hence, using \eqref{eq:known-parameter-weighted-mass-formal} in \cref{lem:known-parameter-weighted-entropy-formal} gives
    \begin{equation*}
        A_\epsilon\theta\sum_{t=1}^T\beta_t^{-\epsilon}\E\brk{1-p_t(i)}\leq(1+\epsilon)A_\epsilon\theta^{\frac{1}{1+\epsilon}}(1+\log K)^{\frac{\epsilon}{1+\epsilon}}\prn*{\sum_{t=1}^T\E\brk*{1-p_t(i)}}^{\frac{1}{1+\epsilon}}.
    \end{equation*}
    Substituting these two inequalities into \eqref{eq:known-parameter-local-regret-expected} proves the claim.
\end{proof}

The following lemma provides the key analytical bounds used in the proof of \cref{lem:known-parameter-local-regret-formal}.
\begin{lemma}
    \label{lem:known-parameter-weighted-entropy-formal}
    Let $N\geq1$, let $z_1,\ldots,z_N\in[0,1]$, and set $Z=\sum_{t=1}^Nz_t$.
    Then, it holds that
    \begin{align}
        \sum_{t=1}^Nt^{-\frac{\epsilon}{1+\epsilon}}z_t &\leq(1+\epsilon)Z^{\frac{1}{1+\epsilon}},
        \label{eq:known-parameter-weighted-mass-formal}\\
        \sum_{t=1}^Nt^{-\frac{\epsilon}{1+\epsilon}}z_t\log\prn*{\frac{eK}{z_t}} &\leq M_\epsilon(1+\log K)Z^{\frac{1}{1+\epsilon}},
        \label{eq:known-parameter-weighted-entropy-formal}
    \end{align}
    where the summand in \eqref{eq:known-parameter-weighted-entropy-formal} is interpreted as zero when $z_t=0$.
\end{lemma}

\begin{proof}
    Let $\omega_t=t^{-\frac{\epsilon}{1+\epsilon}}$ and, for $u\in[0,1]$, define $n(u)=\abs*{\set*{t\in[N]:z_t\geq u}}$.
    Since $(\omega_t)_{t\geq1}$ is nonincreasing, every set $I\subseteq[N]$ with $\abs{I}=n$ satisfies
    \begin{equation*}
        \sum_{t\in I}\omega_t\leq\sum_{t=1}^nt^{-\frac{\epsilon}{1+\epsilon}}\leq(1+\epsilon)n^{\frac{1}{1+\epsilon}},
    \end{equation*}
    where the last inequality follows from \eqref{eq:known-parameter-sum-bound}.
    Jensen's inequality gives
    \begin{equation*}
        \sum_{t=1}^N\omega_tz_t
        =\int_0^1\sum_{t:z_t\geq u}\omega_t \mathrm{d}u
        \leq(1+\epsilon)\int_0^1n(u)^{\frac{1}{1+\epsilon}}\mathrm{d}u
        \leq(1+\epsilon)\prn*{\int_0^1n(u) \mathrm{d}u}^{\frac{1}{1+\epsilon}}
        =(1+\epsilon)Z^{\frac{1}{1+\epsilon}}.
    \end{equation*}
    This proves \eqref{eq:known-parameter-weighted-mass-formal}.

    For every integer $m\geq0$, define $I_m=\set{t\in[N]:2^{-(m+1)}<z_t\leq2^{-m}}$.
    Since every $t\in I_m$ satisfies $2^{m+1}z_t>1$, summing over $t\in I_m$ gives $\abs{I_m}\leq2^{m+1}Z$.
    Moreover, it holds that $z_t\log(1/z_t)\leq2^{-m}(m+1)\log2$ on $I_m$.
    Using the above inequalities, we obtain
    \begin{align*}
        \sum_{t=1}^N\omega_tz_t\log\frac{1}{z_t}
        &\leq \sum_{m=0}^\infty2^{-m}(m+1)\log2\sum_{t\in I_m}\omega_t
        \leq (1+\epsilon)\log2\sum_{m=0}^\infty 2^{-m}(m+1)\abs{I_m}^{\frac{1}{1+\epsilon}}\\
        &\leq(1+\epsilon)2^{\frac{1}{1+\epsilon}}Z^{\frac{1}{1+\epsilon}}\log2 \sum_{m=0}^\infty(m+1)2^{-\frac{m\epsilon}{1+\epsilon}}
        =\frac{(1+\epsilon)2^{\frac{1}{1+\epsilon}}\log2}{\prn{1-2^{-\frac{\epsilon}{1+\epsilon}}}^2}Z^{\frac{1}{1+\epsilon}}.
    \end{align*}
    Since $\log(eK/z_t)=1+\log K+\log(1/z_t)$ and $1+\log K\geq1$, combining this bound with \eqref{eq:known-parameter-weighted-mass-formal} proves \eqref{eq:known-parameter-weighted-entropy-formal}.
\end{proof}

Now, we prove \cref{thm:known-parameter-stochastic-formal}.
\begin{proof}[Proof of \cref{thm:known-parameter-stochastic-formal}]
    Set $S_T=\sum_{t=1}^T\E\brk*{1-p_t(i^\ast)}$.
    Since $p_1$ is uniform and $K\geq2$, we have $S_T\geq1-1/K\geq1/2$.
    Thus, we have
    \begin{equation*}
        \beta_1\log K
        =\theta^{\frac{1}{1+\epsilon}}(1+\log K)^{-\frac{1}{1+\epsilon}}\log K
        \leq2^{\frac{1}{1+\epsilon}}\theta^{\frac{1}{1+\epsilon}}(1+\log K)^{\frac{\epsilon}{1+\epsilon}}S_T^{\frac{1}{1+\epsilon}}.
    \end{equation*}
    Combining this inequality with \cref{lem:known-parameter-local-regret-formal} gives
    \[
        \mathcal{R}_T\leq\prn{B_\epsilon+2^{\frac{1}{1+\epsilon}}}(1+\log K)^{\frac{\epsilon}{1+\epsilon}}\theta^{\frac{1}{1+\epsilon}}S_T^{\frac{1}{1+\epsilon}},
    \]
    where $\mathcal{R}_T=\mathcal{R}_T(i^\ast)$ by \cref{asp:self_bounded}.
    If $\mathcal{R}_T\leq0$, the claim holds immediately, so we suppose that $\mathcal{R}_T>0$.
    By \cref{asp:self_bounded}, it holds that
    \[
        S_T\leq\frac{\mathcal{R}_T+C}{\Delta_{\min}}.
    \]
    Set $a=\prn{B_\epsilon+2^{\frac{1}{1+\epsilon}}}(1+\log K)^{\frac{\epsilon}{1+\epsilon}}\theta^{\frac{1}{1+\epsilon}}\Delta_{\min}^{-\frac{1}{1+\epsilon}}$.
    From the subadditivity of $x\mapsto x^{\frac{1}{1+\epsilon}}$, we obtain
    \begin{equation*}
        \mathcal{R}_T
        \leq a\prn*{\mathcal{R}_T+C}^{\frac{1}{1+\epsilon}}
        \leq a\mathcal{R}_T^{\frac{1}{1+\epsilon}}+aC^{\frac{1}{1+\epsilon}}.
    \end{equation*}
    For every $x\geq0$, Young's inequality gives $ax^{\frac{1}{1+\epsilon}}\leq\kappa_\epsilon a^{\frac{1+\epsilon}{\epsilon}}+\frac{x}{2}$, where we recall that $\kappa_\epsilon=\frac{\epsilon}{1+\epsilon}\prn*{\frac{2}{1+\epsilon}}^{\frac{1}{\epsilon}}$.
    Applying Young's inequality with $x=\mathcal{R}_T$ yields
    \begin{align*}
        \mathcal{R}_T
        &\leq2\kappa_\epsilon a^{\frac{1+\epsilon}{\epsilon}}+2aC^{\frac{1}{1+\epsilon}}\\
        &=2\kappa_\epsilon\prn{B_\epsilon+2^{\frac{1}{1+\epsilon}}}^{\frac{1+\epsilon}{\epsilon}}(1+\log K)\prn*{\frac{\theta}{\Delta_{\min}}}^{\frac{1}{\epsilon}}+2\prn{B_\epsilon+2^{\frac{1}{1+\epsilon}}}(1+\log K)^{\frac{\epsilon}{1+\epsilon}}\prn*{\frac{\theta C}{\Delta_{\min}}}^{\frac{1}{1+\epsilon}}.
    \end{align*}
    This is the claimed bound.
\end{proof}

\section{Parameter Adaptation under the i.i.d.\ Environment}
\label{app:iid-adaptation}

This section gives an additional result under the i.i.d.\ environment.
In contrast to the algorithms analyzed in \cref{sec:algorithms}, the algorithm below is designed only for i.i.d.\ loss vectors.
Adversarial losses and general self-bounded environments are outside its scope.
The result nevertheless shows that adaptation to the moment parameters does not require a logarithmic dependence on $T$ in the gap-dependent regret.

\subsection{Setting and Algorithm}
\label{app:subsec:iid-adaptation-setting}

\begin{algorithm}[t]
    \caption{Parameter-Adaptive Median-of-Means Leader}
    \label{alg:parameter-adaptive-mom-leader}
    \begin{algorithmic}[1]
        \Require {Number of experts $K\geq2$.}
        \State {Fix an expert $i_0\in[K]$ and a deterministic tie-breaking rule.}
        \For{$t=1,2,\ldots$}
            \If{$t=1$}
                \State {Initialize $I_t\gets i_0$ and $p_t\gets e_{I_t}$.}
            \Else
                \State {$n\gets t-1$, $q_n\gets\ceil*{\log(2K)+\log(1+\log(n+1))}+1$, and $b_n\gets2q_n-1$.}
                \If{$n<b_n$}
                    \State {$I_t\gets i_0$ and $p_t\gets e_{I_t}$.}
                \Else
                    \State {$m_n\gets\floor{n/b_n}$ and $\mathcal{B}_{n,s}\gets\set*{(s-1)m_n+1,\ldots,sm_n}$ for every $s\in[b_n]$.}
                    \State {$\bar\mu_{n,s}(i)\gets m_n^{-1}\sum_{u\in\mathcal{B}_{n,s}}\ell_u(i)$ for every $s\in[b_n]$ and $i\in[K]$.}
                    \State {$\hat\mu_{n,b_n}(i)\gets\operatorname{median}\set*{\bar\mu_{n,s}(i):s\in[b_n]}$ for every $i\in[K]$.}
                    \State {$I_t\gets\argmin_{i\in[K]}\hat\mu_{n,b_n}(i)$ using the fixed tie-breaking rule, and $p_t\gets e_{I_t}$.}
                \EndIf
            \EndIf
            \State {Predict $p_t$, incur loss $\inpr{p_t,\ell_t}$, and observe $\ell_t$.}
        \EndFor
    \end{algorithmic}
\end{algorithm}

Let the loss vectors $\ell_1,\ell_2,\dots,\ell_T\in\R^K$ be drawn i.i.d.\ independently of the learner's internal randomization.
The coordinates of each loss vector may be dependent.
We consider \cref{asp:unknown-epsilon_wise}.
Let $\mu_i=\E[\ell_1(i)]$, and write $\mu_\ast=\min_{i\in[K]}\mu_i$ and $\Delta_i=\mu_i-\mu_\ast$.
If there is at least one suboptimal expert, define the minimum positive gap by $\Delta_{\min}=\min\set*{\Delta_i:\Delta_i>0}$.
The optimal expert need not be unique for the results in this section.
Note that Jensen's inequality gives $\abs{\mu_i}\leq\theta^{\frac{1}{1+\epsilon}}$, and therefore we have $0\leq\Delta_i\leq2\theta^{\frac{1}{1+\epsilon}}$ for every $i\in[K]$.

\cref{alg:parameter-adaptive-mom-leader} uses a coordinatewise median-of-means estimate.
Given $n$ observations and an odd integer $b=2q-1\leq n$, partition the first $bm$ observations into $b$ consecutive blocks of equal size $m=\floor{n/b}$ and discard the remaining $(n-bm)$ observations.
For each $i\in[K]$, let $\hat\mu_{n,b}(i)$ be the median of the $b$ empirical block means of coordinate $i$.
This construction limits the influence of outliers, since an extreme observation affects only its own block mean.

Because $I_t$ depends only on past observations, it is independent of $\ell_t$ in the i.i.d.\ setting.
For any $i^\ast\in\argmin_i\mu_i$, the regret of \cref{alg:parameter-adaptive-mom-leader} therefore satisfies $\mathcal{R}_T=\sum_{t=1}^T\E\brk{\Delta_{I_t}}$.
If every expert is optimal, the regret is zero, so below we consider the case $\Delta_{\min}>0$ when stating gap-dependent bounds.

Median-of-means estimators are a classical tool for robust mean estimation under weak moment assumptions (see \citet{lugosi2019meanestimation_ht} for an overview). 
Median-of-means estimators have also been used in sequential learning with heavy-tailed observations, for example in heavy-tailed bandits \citep{Bubeck2013BanditsWithHeavyTail}.

\paragraph{Computational complexity.}
A direct implementation of \cref{alg:parameter-adaptive-mom-leader} takes $O(KT^2)$ time, but an equivalent implementation requires only $O(KT)$ total time.
Store the prefix sums $S_n(i)=\sum_{u=1}^n\ell_u(i)$, with $S_0(i)=0$, so that each block mean is computed as $\bar\mu_{n,s}(i)=(S_{sm_n}(i)-S_{(s-1)m_n}(i))/m_n$ in constant time.
The block means remain unchanged whenever $(b_n,m_n)$ is unchanged, so the estimates and the selected expert can be reused on these rounds.
Each recomputation takes $O(Kb_n)$ time using a linear-time selection algorithm for the medians.
When $b_n=b$ is fixed, $m_n$ changes once every $b$ rounds.
Moreover, the values of $n$ at which $b_n$ changes grow at least geometrically, and changes occurring when $n<b_n$ do not require recomputation.
Consequently, including the $O(K)$ cost per round of updating the prefix sums, the total time is $O(KT)$ with $O(KT)$ memory.

\subsection{Median-of-Means Bounds}
\label{app:subsec:iid-adaptation-mom-bounds}
We first show the following technical lemma.

\begin{lemma}
    \label{lem:iid-adaptation-mom-tail}
    Let $q$ and $n$ be positive integers satisfying $n\geq2q-1$, set $b=2q-1$, and let $I_{n+1}$ minimize the coordinatewise median-of-means estimates $\hat\mu_{n,b}(i)$.
    Then, for every $x>0$, it holds that
    \begin{equation}
        \label{eq:iid-adaptation-selected-gap-tail}
        \P\prn*{\Delta_{I_{n+1}}>x}
        \leq\min\set*{1,\prn*{\frac{128\theta K^{\frac{1}{q}}(b/n)^\epsilon}{x^{1+\epsilon}}}^q}.
    \end{equation}
\end{lemma}

\begin{proof}
    Fix a coordinate $i$ and a block of size $m$, and write $X_1,\dots,X_m$ for the losses of expert $i$ in that block. Then $X_1,\dots,X_m$ are i.i.d.\ and satisfy $\E|X_j|^{1+\varepsilon}\le\theta$.
    Let $X'_1,\dots,X'_m$ be independent copies of $X_1,\dots,X_m$, and let $\sigma_1,\dots,\sigma_m$ be independent Rademacher random variables that are also independent of these samples.
    Conditional Jensen's inequality gives
    \[
        \E\brk*{\abs[\Bigg]{\sum_{j=1}^m\prn*{X_j-\E[X_j]}}^{1+\epsilon}}\leq\E\brk*{\abs[\Bigg]{\sum_{j=1}^m\prn*{X_j-X'_j}}^{1+\epsilon}}.
    \]
    Since the differences $X_j-X'_j$ are independent and symmetric, the random vectors $(X_j-X'_j)_{j=1}^m$ and $(\sigma_j(X_j-X'_j))_{j=1}^m$ have the same distribution, and conditional Jensen's inequality applied to the concave function $z\mapsto z^{\frac{1+\epsilon}{2}}$, using $1+\epsilon\leq2$, yields
    \begin{align*}
        \E\brk*{\abs[\Bigg]{\sum_{j=1}^m\prn*{X_j-X'_j}}^{1+\epsilon}}
        &=\E\brk*{\E_\sigma\brk*{\abs[\Bigg]{\sum_{j=1}^m\sigma_j\prn*{X_j-X'_j}}^{1+\epsilon}}}\\
        &\leq\E\brk*{\prn[\Bigg]{\sum_{j=1}^m\abs{X_j-X'_j}^2}^{\frac{1+\epsilon}{2}}}
        \leq\sum_{j=1}^m\E\brk*{\abs{X_j-X'_j}^{1+\epsilon}},
    \end{align*}
    where $\E_\sigma$ denotes expectation with respect to the $\set{\sigma_j}_{j=1}^m$.
    The inequality $\abs{x-y}^{1+\epsilon}\leq2^\epsilon(\abs{x}^{1+\epsilon}+\abs{y}^{1+\epsilon})$ implies $\E[\abs{X_j-X'_j}^{1+\epsilon}]\leq2^{1+\epsilon}\theta$.
    Dividing the resulting bound by $m^{1+\epsilon}$ gives 
    \begin{equation}
        \label{eq:iid-adaptation-mean-moment}
        \E\brk*{\abs[\Bigg]{\frac{1}{m}\sum_{j=1}^mX_j-\E[X_1]}^{1+\epsilon}}\leq\frac{2^{1+\epsilon}\theta}{m^\epsilon}.
    \end{equation}

    We next prove \eqref{eq:iid-adaptation-selected-gap-tail}.
    Fix a coordinate $i$ and let $\bar X_s(i)$ be its empirical mean in block $s\in[b]$.
    Markov's inequality and \eqref{eq:iid-adaptation-mean-moment} give, for every $u>0$,
    \[
        \P\prn*{\abs{\bar X_s(i)-\mu_i}>u}\leq \frac{2^{1+\epsilon}\theta}{m^\epsilon u^{1+\epsilon}}.
    \]
    We set $a=\frac{2^{1+\epsilon}\theta}{m^\epsilon u^{1+\epsilon}}$.
    If the median $\hat\mu_{n,b}(i)$ differs from $\mu_i$ by more than $u$, at least $q$ of the $b$ block means differ from $\mu_i$ by more than $u$.
    The blocks are independent, so a union bound over all subsets of $q$ blocks gives
    \[
        \P\prn*{\abs{\hat\mu_{n,b}(i)-\mu_i}>u}\leq\binom{b}{q}a^q\leq2^ba^q\leq(4a)^q.
    \]
    Since $n/b\geq1$, we have $m=\floor{n/b}\geq n/(2b)$.
    Using $1+\epsilon\leq2$ therefore yields
    \[
        \P\prn*{\abs{\hat\mu_{n,b}(i)-\mu_i}>u}\leq\prn*{\frac{32\theta(b/n)^\epsilon}{u^{1+\epsilon}}}^q.
    \]

    Let $i^\ast\in\argmin_i\mu_i$.
    Since $I_{n+1}$ minimizes the estimated means, we have
    \begin{equation*}
        \Delta_{I_{n+1}}
        =\mu_{I_{n+1}}-\mu_{i^\ast}\leq\abs{\mu_{I_{n+1}}-\hat\mu_{n,b}(I_{n+1})}+\abs{\hat\mu_{n,b}(i^\ast)-\mu_{i^\ast}}
        \leq2\max_{i\in[K]}\abs{\hat\mu_{n,b}(i)-\mu_i}.
    \end{equation*}
    Applying a union bound to the preceding inequality with $u=x/2$, and using $2^{1+\epsilon}\leq4$, gives
    \[
        \P\prn*{\Delta_{I_{n+1}}>x}\leq K\prn*{\frac{128\theta(b/n)^\epsilon}{x^{1+\epsilon}}}^q=\prn*{\frac{128\theta K^{\frac{1}{q}}(b/n)^\epsilon}{x^{1+\epsilon}}}^q.
    \]
    Combining this bound with the trivial bound $\P(\Delta_{I_{n+1}}>x)\leq1$ proves the claim.
\end{proof}

The following lemma integrates the tail bound in \cref{lem:iid-adaptation-mom-tail} to bound the expected gap $\E[\Delta_{I_{n+1}}]$.

\begin{lemma}
    \label{lem:iid-adaptation-integrated-gap}
    In the setting of \cref{lem:iid-adaptation-mom-tail}, define $a_n=\prn{128\theta K^{\frac{1}{q}}}^{\frac{1}{1+\epsilon}}\prn{{b}/{n}}^{\frac{\epsilon}{1+\epsilon}}$.
    If $n\geq b$, then we have $\E[\Delta_{I_{n+1}}]\leq\frac{(1+\epsilon)q}{(1+\epsilon)q-1}a_n$.
    If $\Delta_{\min}>0$, then we also have $\E[\Delta_{I_{n+1}}]\leq\frac{(1+\epsilon)q}{(1+\epsilon)q-1}\min\set{a_n,a_n^{(1+\epsilon)q}\Delta_{\min}^{1-(1+\epsilon)q}}$.
\end{lemma}

\begin{proof}
    By \cref{lem:iid-adaptation-mom-tail}, $\Delta_{I_{n+1}}$ satisfies $\P(\Delta_{I_{n+1}}>x)\leq\min\set*{1,\prn{a_n/x}^{(1+\epsilon)q}}$ for every $x>0$.
    Then, we have
    \[
        \E[\Delta_{I_{n+1}}]=\int_0^\infty\P(\Delta_{I_{n+1}}>x)\,\mathrm{d}x\leq a_n+\int_{a_n}^\infty\prn*{\frac{a_n}{x}}^{(1+\epsilon)q}\,\mathrm{d}x=\frac{(1+\epsilon)q}{(1+\epsilon)q-1}a_n,
    \]
    which proves the first claim.

    Suppose now that $\Delta_{\min}>0$.
    The selected gap takes values in $\set{0}\cup[\Delta_{\min},\infty)$, so
    \[
        \E[\Delta_{I_{n+1}}]=\Delta_{\min}\P(\Delta_{I_{n+1}}>0)+\int_{\Delta_{\min}}^\infty\P(\Delta_{I_{n+1}}>x)\,\mathrm{d}x.
    \]
    If $a_n\leq\Delta_{\min}$, taking the limit of the tail bound as $x$ increases to $\Delta_{\min}$ gives $\P(\Delta_{I_{n+1}}>0)\leq(a_n/\Delta_{\min})^{(1+\epsilon)q}$.
    Consequently, we have
    \begin{equation*}
        \E[\Delta_{I_{n+1}}]
        \leq a_n^{(1+\epsilon)q}\Delta_{\min}^{1-(1+\epsilon)q}+\int_{\Delta_{\min}}^\infty\prn*{\frac{a_n}{x}}^{(1+\epsilon)q}\,\mathrm{d}x
        =\frac{(1+\epsilon)q}{(1+\epsilon)q-1}a_n^{(1+\epsilon)q}\Delta_{\min}^{1-(1+\epsilon)q}.
    \end{equation*}
    If $a_n>\Delta_{\min}$, then $a_n<a_n^{(1+\epsilon)q}\Delta_{\min}^{1-(1+\epsilon)q}$.
    This proves the second claim.
\end{proof}

\subsection{Regret Guarantees}
\label{app:subsec:iid-adaptation-regret}

We now show that increasing the number of blocks slowly with the sample size yields simultaneous adaptation to every $\epsilon\in(0,1]$, while $\theta$, $\Delta_{\min}$, and the horizon are unknown to the algorithm.

\begin{theorem}
    \label{thm:iid-full-moment-adaptation}
    For every fixed $\epsilon\in(0,1]$, when $\Delta_{\min}>0$, \cref{alg:parameter-adaptive-mom-leader} satisfies
    \[
        \mathcal{R}_T=O\prn*{\prn*{\frac{\theta}{\Delta_{\min}}}^{\frac{1}{\epsilon}}\prn*{\log K+\log\log\prn*{e+\frac{\theta}{\Delta_{\min}^{1+\epsilon}}}}}.
    \]
\end{theorem}

Although a doubly logarithmic factor remains, this bound does not contain the $\log T$ factor appearing in the gap-dependent bounds of \cref{thm:omd-regret-self-bounded}.
However, the present proof requires a hidden constant that can be doubly exponential in $1/\epsilon$.
We do not establish whether this dependence is necessary.

\begin{proof}
    Define $h=\prn*{\frac{128e\theta}{\Delta_{\min}^{1+\epsilon}}}^{\frac{1}{\epsilon}}$ and $\Lambda=\log(2K)+\log\prn*{1+\log(e+h)}$.
    The inequality $\Delta_{\min}\leq2\theta^{\frac{1}{1+\epsilon}}$ and the definition of $h$ imply $h>1$.
    Choose a constant $C_\epsilon\geq1$ such that $\log(1+\log C_\epsilon)\geq4/\epsilon$ and $C_\epsilon\geq6\prn*{\log(1+\log(3C_\epsilon))+4}$.
    Such a constant exists because $\log\log x=o(x)$ as $x\to\infty$.
    Set $N=\ceil*{C_\epsilon h\Lambda}$.
    The first condition on $C_\epsilon$ and the inequality $N\geq C_\epsilon$ imply $q_N\geq4/\epsilon$.
    Since $C_\epsilon\geq1$, $h>1$, and $\Lambda\geq1$, we have $N+1\leq3C_\epsilon h\Lambda$ and $1+\log(N+1)\leq(1+\log(3C_\epsilon))(1+\log h)(1+\log\Lambda)$.
    Using the definition of $\Lambda$ and the inequality $1+\log\Lambda\leq \Lambda$, the definition of $b_N$ therefore gives $b_N\leq2\prn*{\log(1+\log(3C_\epsilon))+\Lambda+\log\Lambda}+3$.
    This inequality, $\log\Lambda\leq \Lambda$, and $\Lambda\geq1$ imply $b_N\leq2\prn*{\log(1+\log(3C_\epsilon))+4}\Lambda$.
    The second condition on $C_\epsilon$ therefore implies $N\geq C_\epsilon h\Lambda\geq3hb_N$.
    The definition of $q_n$ also gives $q_n\geq\log(2K)$ for every $n$, and hence $K^{\frac{1}{q_n}}\leq e$ and $\frac{(1+\epsilon)q_n}{(1+\epsilon)q_n-1}\leq2$.

    First, we consider $n\geq N$.
    The inequality $1+\log(n+1)\leq\frac{n}{N}(1+\log(N+1))$ implies $b_n\leq b_N+2\log(n/N)+2$.
    Since $\log(n/N)\leq2(\sqrt{n/N}-1)$ and $b_N\geq5$, it follows that $b_n\leq3b_N\sqrt{n/N}$.
    Combining this estimate with $N\geq3hb_N$ gives
    \begin{equation}
        \label{eq:iid-adaptation-large-n-scale}
        \frac{hb_n}{n}\leq\sqrt{\frac{N}{n}}\leq1.
    \end{equation}
    Since $h>1$, \eqref{eq:iid-adaptation-large-n-scale} also implies $n\geq b_n$, so the algorithm computes the median for every $n\geq N$.
    Moreover, the monotonicity of $q_n$ and $q_N\geq4/\epsilon$ give $\epsilon q_n\geq \epsilon q_N\geq4$.

    Let $a_n$ be the quantity in \cref{lem:iid-adaptation-integrated-gap} with $q=q_n$ and $b=b_n$.
    Since $K^{\frac{1}{q_n}}\leq e$ and $h^\epsilon=128e\theta/\Delta_{\min}^{1+\epsilon}$, we have
    $
        \frac{a_n}{\Delta_{\min}}\leq\prn*{\frac{hb_n}{n}}^{\frac{\epsilon}{1+\epsilon}}.
    $
    Applying the second bound in \cref{lem:iid-adaptation-integrated-gap}, using $\frac{(1+\epsilon)q_n}{(1+\epsilon)q_n-1}\leq2$, and then using \eqref{eq:iid-adaptation-large-n-scale} gives
    \begin{equation*}
        \E[\Delta_{I_{n+1}}]
        \leq 2\Delta_{\min}\prn*{\frac{a_n}{\Delta_{\min}}}^{(1+\epsilon)q_n}
        \leq2\Delta_{\min}\prn*{\frac{hb_n}{n}}^{\epsilon q_n}
        \leq2\Delta_{\min}\prn*{\frac{N}{n}}^2,
    \end{equation*}
    where the last inequality uses $\epsilon q_n\geq4$.
    Therefore, we have
    \[
        \sum_{n=N}^{T-1} \E[\Delta_{I_{n+1}}]\leq2\Delta_{\min}N^2\sum_{n=N}^\infty n^{-2}\leq4\Delta_{\min}N.
    \]

    It remains to control the regret over the rounds with $n<N$.
    Since $(b_n)_{n\geq1}$ is nondecreasing, at most $b_N$ such rounds do not compute the median, and their total regret is at most $2\theta^{\frac{1}{1+\epsilon}}b_N$.
    Using the first bound in \cref{lem:iid-adaptation-integrated-gap}, $K^{\frac{1}{q_n}}\leq e$, and $b_n\leq b_N$ gives
    \begin{equation}
        \label{eq:iid-adaptation-small-n-regret}
        \sum_{n=0}^{N-1}\E[\Delta_{I_{n+1}}]\leq2\theta^{\frac{1}{1+\epsilon}}b_N+2(1+\epsilon)(128e\theta)^{\frac{1}{1+\epsilon}}b_N^{\frac{\epsilon}{1+\epsilon}}N^{\frac{1}{1+\epsilon}}.
    \end{equation}
    By the inequality $\Delta_{\min}\leq2\theta^{\frac{1}{1+\epsilon}}$ and the definition of $h$, we have
    $
        2\theta^{\frac{1}{1+\epsilon}}
        \leq (64e)^{\frac{1}{\epsilon}}\theta^{\frac{1}{1+\epsilon}}
        \leq \prn*{\frac{128e\theta}{\Delta_{\min}}}^{\frac{1}{\epsilon}}
        =
        \Delta_{\min}h.
    $
    Together with $hb_N\leq N$, this bounds the first term of \eqref{eq:iid-adaptation-small-n-regret} by $\Delta_{\min}N$.
    Moreover, using $(128e\theta)^{\frac{1}{1+\epsilon}}=\Delta_{\min}h^{\frac{\epsilon}{1+\epsilon}}$, $hb_N\leq N$, and $1+\epsilon\leq2$, we obtain
    \begin{equation*}
        2(1+\epsilon)(128e\theta)^{\frac{1}{1+\epsilon}}b_N^{\frac{\epsilon}{1+\epsilon}}N^{\frac{1}{1+\epsilon}}
        =2(1+\epsilon)\Delta_{\min}N\prn*{\frac{hb_N}{N}}^{\frac{\epsilon}{1+\epsilon}}
        \leq4\Delta_{\min}N.
    \end{equation*}
    Thus, the regret accumulated over the rounds with $n<N$ is at most $5\Delta_{\min}N$.
    Combining this estimate with the bound for $n\geq N$ yields $\mathcal{R}_T\leq9\Delta_{\min}N=O(\Delta_{\min}h\Lambda)$ for every $T\geq1$.
    Finally, substituting $\Delta_{\min}h=(128e)^{\frac{1}{\epsilon}}(\theta/\Delta_{\min})^{\frac{1}{\epsilon}}$ and the definition of $\Lambda$ proves the claim.
\end{proof}

\section{Experiments}
\label{app:experiments}

\begin{figure}[t]
    \centering
    \includegraphics[width=\textwidth]{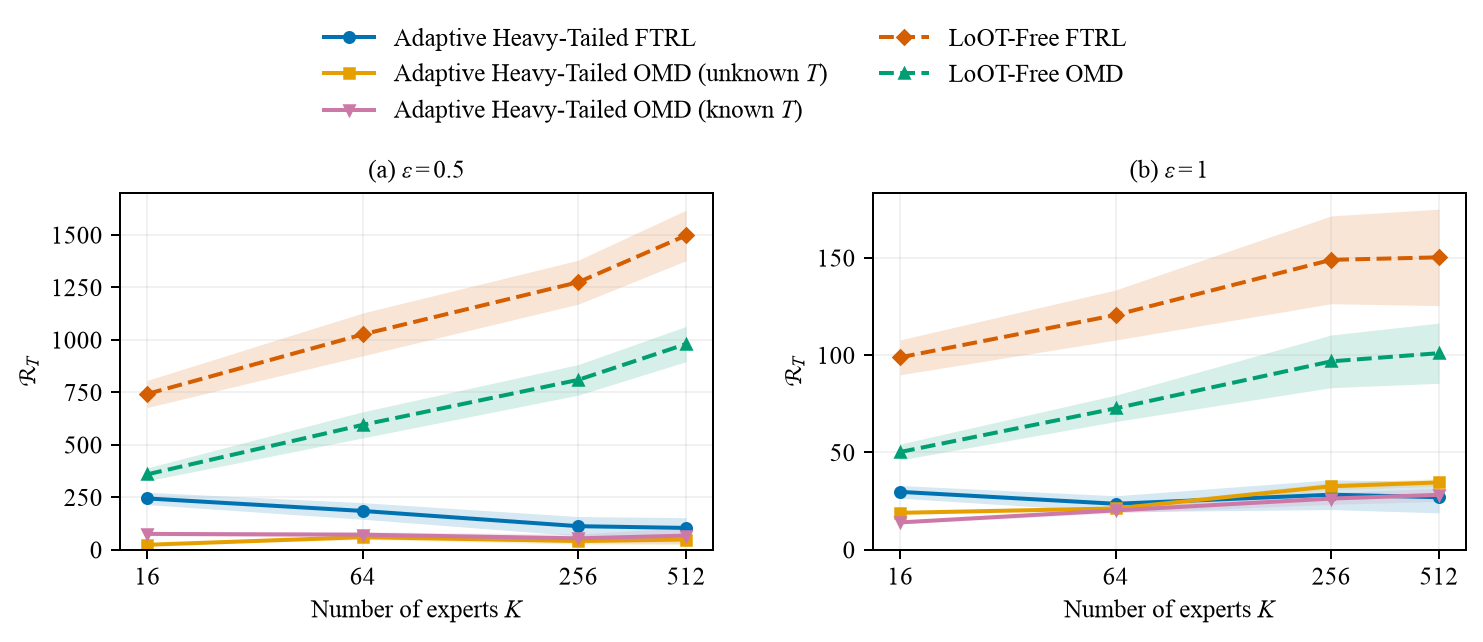}
    \caption{Regret on the balanced sparse-spike instance with $T=10000$, $\theta=1$, and $\Delta=1$.
    Both panels vary $K$.
    Panel (a) uses $\epsilon=0.5$, whereas panel (b) uses $\epsilon=1$.
    Curves show means over $100$ paired repetitions, and shaded regions indicate plus or minus one standard error.}
    \label{fig:balanced-sparse-spike-main}
\end{figure}

We conduct experiments to evaluate the performance of our proposed algorithms.
We compare our adaptive algorithms with the LoOT-Free algorithms of \citet{NEURIPS2025_Moulin_Expert_HTEXP} on a balanced sparse-spike instance from \cref{prop:moulin-balanced-sparse-spike} and a stochastic instance.
The first experiment examines whether the polynomial dependence on the number of experts established in \cref{app:subsec:alg_specific_lower_bound} appears at finite horizons and whether our methods avoid this dependence.
The second experiment evaluates performance when a constant mean loss gap and rare spikes coexist throughout the horizon.
These experiments are intended as a finite-horizon illustration of the mechanism identified by the theory, rather than as an estimate of asymptotic rates.
We conducted the experiments on a MacBook Air with 16 GB of RAM and an Apple M3 chip.

\paragraph{Setup.}
We conduct experiments on two instances.
For the first instance, we use the balanced sparse-spike instance with $\theta=1$ and $T=10000$.
Let $m=\lfloor T/2\rfloor$ and $B=(\theta Km/2)^{\frac{1}{1+\epsilon}}$.
In each round $t\leq m$, independently across rounds, with probability $1/m$, we draw an ordered pair $(J_t^+,J_t^-)$ uniformly from the set of ordered pairs of distinct coordinates and set $\ell_t(J_t^+)=B$, $\ell_t(J_t^-)=-B$, and $\ell_t(i)=0$ for every $i\notin\set{J_t^+,J_t^-}$.
With the remaining probability $1-1/m$, we set $\ell_t(i)=0$ for every $i\in[K]$.
In each round $t>m$, we set $\ell_t(i^\ast)=0$ for the optimal expert $i^\ast=1$ and $\ell_t(i)=\Delta$ for every $i\neq i^\ast$.

The second instance is inspired by the experiments of \citet[Appendix~F]{NEURIPS2025_Moulin_Expert_HTEXP}, with a modified spike amplitude.
Let $T=10000$ and $\Delta=1$.
We set $\ell_t(i)=\Delta\ind\brk{i\neq i^\ast}+X_t(i)$ for all $t\in[T]$ and $i\in[K]$, where $i^\ast=1$.
Independently across rounds and experts, $X_t(i)$ takes the values $\pm T^{\frac{1}{1+\epsilon}}$ with probability $1/(2T)$ each, and is zero otherwise.
Then, the losses satisfy $\E\brk{\abs{\ell_t(i)}^{1+\epsilon}}\leq2$.
Unlike the first instance, this instance maintains a constant mean loss gap and permits spikes throughout the horizon.

We consider $K\in\set{16,64,256,512}$, $\epsilon\in\set{0.5,1}$, and $\Delta=1$.
For each configuration, we perform $100$ paired repetitions, in which all algorithms receive the same loss sequence.
We compare Adaptive Heavy-Tailed FTRL (\cref{alg:adaptive-multiscale-ftrl-componentwise_simple}), Adaptive Heavy-Tailed OMD (\cref{alg:parameter-free-heavy-tailed-omd}), and the LoOT-Free FTRL and OMD algorithms.
To examine the effect of the doubling trick and match the horizon information available to LoOT-Free OMD, we also include a version of Adaptive Heavy-Tailed OMD that sets $H_0=T$ and $I_0=[T]$ in \cref{alg:parameter-free-heavy-tailed-omd} and runs without the doubling trick.
None of the algorithms receives the data-generating values of $\epsilon$ or $\theta$.
The LoOT-Free algorithms use their second-moment versions, corresponding to $\epsilon=1$, in both $\epsilon$ settings.
In our experiments, we compute the FTRL and OMD updates using bisection.

Our metric is regret defined as in \eqref{eq:def_regret}.
Unless stated otherwise, curves show means over the paired repetitions, and shaded regions indicate plus or minus one standard error.

\paragraph{Results.}
\begin{figure}[t]
    \centering
    \includegraphics[width=\textwidth]{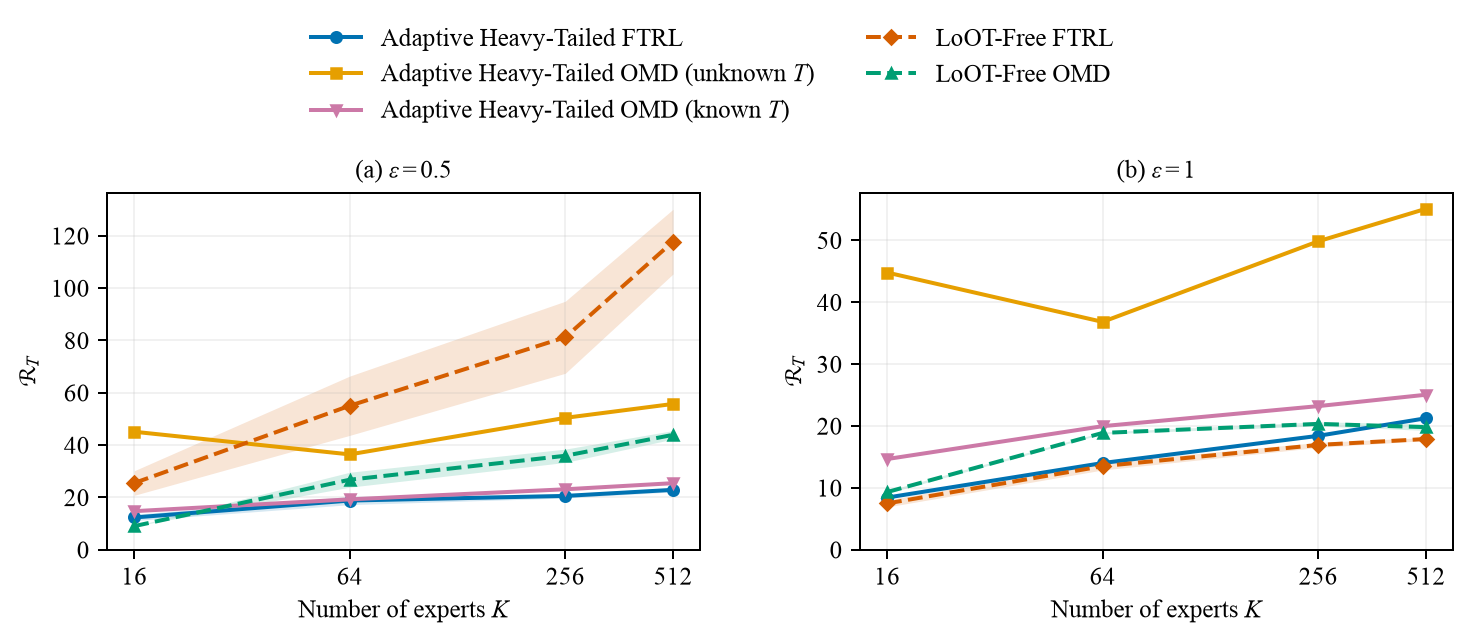}
    \caption{Regret on the stochastic instance with $T=10000$ and $\Delta=1$.
    Both panels vary $K$.
    Panel (a) uses $\epsilon=0.5$, whereas panel (b) uses $\epsilon=1$.
    Curves show means over $100$ paired repetitions, and shaded regions indicate plus or minus one standard error.}
    \label{fig:stochastic-heuristic-main}
\end{figure}

\Cref{fig:balanced-sparse-spike-main} compares regret on the first instance as the number of experts varies, with the horizontal axis on a logarithmic scale.
For $\epsilon=0.5$, the regrets of both LoOT-Free algorithms increase with $K$, whereas both of our algorithms achieve lower regret and do not exhibit a comparable increase.
For $\epsilon=1$, both of our algorithms remain comparatively flat, while the LoOT-Free algorithms show a more modest increase than in the $\epsilon=0.5$ setting.
In addition, on this instance, the versions of Adaptive Heavy-Tailed OMD with and without known $T$ show a relatively small difference in regret.
This experimental result is consistent with our theoretical results, which establish a polynomial dependence on $K$ for the LoOT-Free algorithms under $\epsilon<1$.
The polynomial factor in the lower bound disappears in the second-moment case $\epsilon=1$.

\Cref{fig:stochastic-heuristic-main} compares regret on the second instance as the number of experts varies, with the horizontal axis on a logarithmic scale.
For $\epsilon=0.5$, Adaptive Heavy-Tailed FTRL and Adaptive Heavy-Tailed OMD with known $T$ achieve lower regret than the LoOT-Free algorithms except at $K=16$, and their regret increases more slowly with $K$.
For $\epsilon=1$, their regret shows a similar dependence on $K$ to that of the LoOT-Free algorithms.
Adaptive Heavy-Tailed OMD with unknown $T$ incurs higher regret than its known-horizon counterpart in both $\epsilon$ settings.
As in the first experiment, the LoOT-Free algorithms exhibit a stronger increase in regret with $K$ for $\epsilon=0.5$ than for $\epsilon=1$.

\end{document}